\def\year{2023}\relax
\documentclass[letterpaper]{article} 
\usepackage{aaai23}  
\usepackage{times}  
\usepackage{helvet}  
\usepackage{courier}  
\usepackage[hyphens]{url}  
\usepackage{graphicx} 
\usepackage{natbib}  
\usepackage{caption} 
\DeclareCaptionStyle{ruled}{labelfont=normalfont,labelsep=colon,strut=off} 
\usepackage{algorithm}
\usepackage{algorithmic}

\usepackage{subfigure}
\usepackage[alpha]{mdpn}
\usepackage{cs839}

\usepackage{todonotes}
\newcounter{todocounter}

\usepackage{newfloat}
\usepackage{listings}
\floatstyle{ruled}
\newfloat{listing}{tb}{lst}{}
\floatname{listing}{Listing}
\title{Scaling Marginalized Importance Sampling to High-Dimensional State-Spaces via State Abstraction}
\author{
Brahma S. Pavse {\normalfont and} Josiah P. Hanna
}
\affiliations{
    Department of Computer Sciences\\
    University of Wisconsin -- Madison\\
    pavse@wisc.edu, jphanna@cs.wisc.edu

}

\usepackage{bibentry}

\begin{document}

\maketitle

\begin{abstract}
We consider the problem of off-policy evaluation (OPE) in reinforcement learning (RL), where the goal is to estimate the performance of an evaluation policy, $\pi_e$, using a fixed dataset, $\mathcal{D}$, collected by one or more policies that may be different from $\pi_e$. Current OPE algorithms may produce poor OPE estimates under policy distribution shift i.e., when the probability of a particular state-action pair occurring under $\pi_e$ is very different from the probability of that same pair occurring in $\mathcal{D}$ \cite{voloshin2021opebench, fu2021opebench}. In this work, we propose to improve the accuracy of OPE estimators by projecting the high-dimensional state-space into a low-dimensional state-space using concepts from the state abstraction literature. Specifically, we consider marginalized importance sampling (MIS) OPE algorithms which compute state-action distribution correction ratios to produce their OPE estimate. In the original ground state-space, these ratios may have high variance which may lead to high variance OPE. However, we prove that in the lower-dimensional abstract state-space the ratios can have lower variance resulting in lower variance OPE. We then highlight the challenges that arise when estimating the abstract ratios from data, identify sufficient conditions to overcome these issues, and present a minimax optimization problem whose solution yields these abstract ratios. Finally, our empirical evaluation on difficult, high-dimensional state-space OPE tasks shows that the abstract ratios can make MIS OPE estimators achieve lower mean-squared error and more robust to hyperparameter tuning than the ground ratios.
\end{abstract}


\section{Introduction}
    One of the key challenges when applying reinforcement learning (RL) \cite{sutton1998rlbook}
    to real-world tasks is the problem of
    off-policy evaluation (OPE) \cite{fu2021opebench, voloshin2021opebench}. The goal of OPE is to evaluate a policy of interest
    by leveraging offline data generated by possibly different policies. Solving the OPE
    problem would enable us to estimate the performance of a potentially risky policy without having to actually deploy it. This capability is especially important for sensitive real-world tasks such as healthcare and autonomous
    driving.
    \begin{figure}[H]
        \centering
            \includegraphics[scale=0.55]{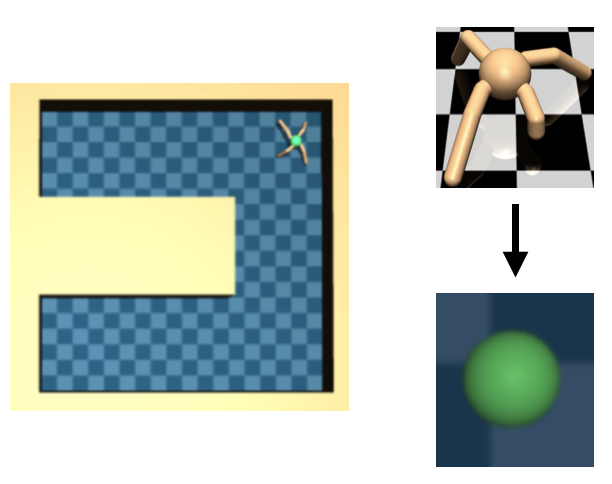}
        \caption{\footnotesize Left side: AntUMaze domain. Right side: Projecting high-dimensional ant into lower-dimensional point-mass.}
        \label{fig:abs_motivation}
    \end{figure}
    The core OPE problem is to produce 
    accurate policy value estimates under policy distribution shift. This problem is particularly difficult
    on tasks with high-dimensional state-spaces \cite{voloshin2021opebench, fu2021opebench}.
    For example, consider the AntUMaze problem illustrated 
    on the left side of Figure \ref{fig:abs_motivation}. In this
    task, an ant-like robot with a high-dimensional state representation moves in a U-shaped maze and receives a reward
    only for reaching a specific $2$D coordinate goal location. The state-space
    of this task includes information such as $2$D location, ant limb angles,
    torso orientation etc., resulting in a $29$-dimensional
    state-space. The OPE task is 
    to evaluate the performance of a particular policy's ability to take
    the ant to the $2$D goal location using data that may be collected by different policies. Policy distribution shift is common in this type of high-dimensional
    task since
    the chances of different policies inducing similar limb angles, torso orientations, paths traversed etc. are incredibly slim. Notice, however, that while different policies may induce different body configurations, they may traverse similar $2$D paths since all (successful) policies must move the ant through roughly the same path to reach the goal. Moreover, the only critical information needed from the state-space to determine the ant's per-step reward are its $2$D coordinates. Motivated by this
    observation, we propose to improve
    the accuracy of OPE algorithms on high-dimensional state-space tasks by projecting
    the high-dimensional state-space into a lower-dimensional space. This 
    idea is illustrated on the right side of
    Figure \ref{fig:abs_motivation} where the ant is reduced to a $2$D point-mass.
    
    With this general motivation in mind, in this paper, we leverage concepts from the state abstraction 
    literature \cite{li2006stateabs} to improve the accuracy of marginalized importance
    sampling (MIS) OPE algorithms which estimate state-action density correction ratios to compute a policy value estimate \cite{qiang2018horizon, xie2019MIS, yin2020misvariance}. Due to the low chances of similarity between states 
    of policies in high-dimensional state-spaces, current MIS algorithms can produce high variance state-action density ratios,
    resulting in high variance OPE estimates. 
    However, if we are given a suitable state abstraction function, we can 
    %
    project the high-dimensional \textit{ground}  state-space into a lower-dimensional \textit{abstract} state-space. The projection step
    increases the chances of similarity between these
    lower-dimensional states, resulting in
    low variance density ratios and
    OPE estimates. To the best of our knowledge, this work is the first
    to leverage concepts from state abstraction to improve OPE. We make the following contributions:

    \begin{enumerate}
        \item Theoretical analyses showing: (a) the variance of abstract state-action ratios
        is at most that of ground state-action ratios; and (b) that our abstract MIS OPE estimator is unbiased, strongly consistent, and can have lower variance 
        than its ground equivalent.
        \item An algorithm, based on a popular class of
        MIS algorithms, to estimate the abstract state-action ratios. 
        \item An empirical analysis of our estimator on a variety of
        high-dimensional state-space
        tasks.
    \end{enumerate} 
    
\section{Preliminaries}
In this section, we discuss relevant background information.
\subsection{Notation and Problem Setup}
    We consider an infinite-horizon
    Markov decision process (MDP), $\mathcal{M} = \langle \sset, \aset, \mathcal{R}, P, \gamma , d_0\rangle$,
    where $\sset$ is the state-space, $\aset$ is the action-space,
    $\mathcal{R}:\sset \times\aset \to \Delta([0,\infty))$ is the reward function,
    $P:\sset\times\aset\to\Delta(\sset)$ is the transition dynamics function, $\gamma\in[0,1)$ is the discount factor, and $d_0\in \Delta(\sset)$ is the initial state distribution. The
    agent acting, according to policy $\pi$, in the MDP generates a trajectory: $s_0, a_0, r_0, s_1, a_1, r_1, ...$, where $s_0\sim d_0$, $a_t\sim\pi(\cdot|s_t)$, 
    $r_t\sim \mathcal{R}(s_t, a_t)$, and $s_{t+1}\sim P(\cdot|s_t, a_t)$ for
    $t\geq0$. We define $r(s,a) := \E_{r\sim\mathcal{R}(s,a)}[r]$ and the agent's discounted state-action occupancy measure under policy $\pi$ as $d_{\pi}$:
    \begin{align*}
        d_{\pi}(s,a) := \lim_{T\to\infty}
            \frac{\sum_{t = 0}^{T - 1} \gamma^t d_{\pi}(s_t, a_t)}{\sum_{t = 0}^{T-1}\gamma^t}
    \end{align*}
    where $d_{\pi}(s_t, a_t)$ is the probability the agent will be in state $s$ and take action $a$ at time-step $t$ under policy $\pi$.

    In the infinite-horizon setting, we define the performance of policy $\pi$ to be its average reward, $\rho(\pi):= \E_{(s, a)\sim d_{\pi}, r\sim \mathcal{R}(s,a)}[r]$. Note that $\E_{(s, a)\sim d_{\pi}, r\sim 
    \mathcal{R}(s,a)}[r] = (1 - \gamma) \E_{s_0\sim d_0, a_0\sim 
    \pi}[q^\pi(s_0, a_0)]$ where $q^\pi(s,a) := \E[\sum_{k = 0}^\infty \gamma^tr_{k+t} | s_t = s, a_t = a]$
    is the action-value function which satisfies
    $q^\pi(s,a) = r(s,a) + \gamma\E_{s'\sim P(s,a), a'\sim\pi(s')}[q^\pi(s', a')]$.

\subsection{Off-Policy Evaluation (OPE)}
    In behavior-agnostic off-policy evaluation (OPE), the goal 
    is to 
    estimate the performance of an evaluation policy $\pi_e$ given only a fixed offline data set of transition tuples, $\mathcal{D} := \{(s_i, a_i, s_i', r_i)\}_{i=1}^{mT}$,
    where $(s_i, a_i)\sim d_{\mathcal{D}}$, $m$ is the batch size (number of trajectories), and $T$ is the fixed length of each trajectory,
    generated by \textit{unknown} and possibly \textit{multiple} 
    behavior
    policies. The difficulty in OPE is to estimate $\rho(\pi_e)$
    under $d_{\pi_e}$ given samples only from $d_\mathcal{D}$.
    
    We define the average-reward in dataset $\mathcal{D}$ to be $\bar{r}_{\mathcal{D}}:= \E_{(s, a)\sim d_{\mathcal{D}}, r\sim \mathcal{R}(s,a)}[r]$. 
    As in prior OPE work, we assume that if $d_{\pi_e}(s,a) > 0$ then $d_{\mathcal{D}}(s,a)> 0$. Empirically, we measure the accuracy of an estimate
    $\hat{\rho}(\pi_e)$ by generating $M$ datasets
    and then computing the \textit{relative}
    mean-squared error (MSE):
    $\text{MSE}(\hat{\rho}(\pi_e)) := \frac{1}{M} \sum_{i=1}^M \frac{\left(\rho(\pi_e) - \hat{\rho_i}(\pi_e)\right)^2}{\left(\rho(\pi_e) - \bar{r}_{\mathcal{D}_i}\right)^2}$, where  $\hat{\rho_i}(\pi_e)$ is computed using dataset $\mathcal{D}_i$ and $\bar{r}_{\mathcal{D}_i}$
    is the average reward in $\mathcal{D}_i$.
 
    \subsubsection{Marginalized Importance Sampling (MIS)}
    In this work, we focus on MIS methods, which evaluate $\pi_e$ by using
    the ratio between $d_{\pi_e}$ and $d_{\mathcal{D}}$. That is,
    MIS methods evaluate $\pi_e$ by estimating 
    $\rho(\pi_e):=  \E_{(s, a)\sim d_{\mathcal{D}}, r\sim \mathcal{R}(s,a)}[\zeta(s,a) r]$, where
    \begin{equation*}
        \zeta(s,a) := \frac{d_{\pi_e}(s,a)}{d_{\mathcal{D}}(s,a)}
    \end{equation*}
    is the state-action density ratio for state-action pair $(s,a)$
    and $d_\pi(s,a) = d_\pi(s)\pi(a|s)$.
    When the true $\zeta$ is known, the empirical estimate of $\rho(\pi_e)$ is:
    \begin{align}
        \hat{\rho}(\pi_e) &:= \frac{1}{N}\sum_{i=1}^N \zeta(s_i, a_i) r(s_i,a_i)
        \label{eq:qiang_estimator}
    \end{align}
    where $N$ is the number of samples. In practice, however,
    $\zeta$ is unknown and must be estimated.
    
    One set of $\zeta$-estimation algorithms, which have also
    shown potential for good OPE performance \cite{voloshin2021opebench}, is the DICE family \cite{nachum2020bestdice}.  While
    there are many variations, the general DICE optimization problem is:
    \begin{equation}
        \begin{split}
        &\max_{\zeta:\sset\times A\to\mathbb{R}}\min_{\nu:\sset\times A\to\mathbb{R}}J(\zeta, \nu) :=\\
            &\E_{(s, a, s')\sim d_{\mathcal{D}}, a'\sim\pi_e}[\zeta(s,a) (\nu(s,a) - \gamma\nu(s', a'))]\\
            &- (1 - \gamma)\E_{s_0\sim d_0, a_0\sim\pi_e}[\nu(s_0, a_0)] 
        \end{split}
        \label{eq:final_gr_obj}
    \end{equation}
    where the solution to the optimization, 
    $\zeta^*(s,a)$, are the true ratios. The estimator we present in Section
    \ref{sec:abs_mis} builds upon the DICE framework.

\subsection{State Abstractions}
    \label{sec:bg_state_abs}
    We define a state abstraction function as a mapping $\phi:\sset\to\sset^\phi$, where  $\sset$ is called the \textit{ground} state-space and $\sset^\phi$ is called the \textit{abstract} state-space.
    We consider state abstraction functions that
    partition the ground state-space into disjoint sets.

    We can use $\phi$ to project the original MDP into a new abstract MDP
    with the same action-space $\aset$ and reward and transition dynamics functions defined as:
    \begin{align*}
        \mathcal{R}^\phi(s^\phi, a) &= \sum_{s\in\phi^{-1}(s^\phi)}w(s)\mathcal{R}(s,a)\\ 
        P^\phi(s'^{\phi} | s^\phi, a) &= \sum_{s\in\phi^{-1}(s^\phi), s'\in\phi^{-1}(s'^{\phi})}w(s)P(s'|s,a)
    \end{align*}
    where $w:\sset\to[0,1]$ is a ground state weighting function where for each
    $s^\phi\in\mathcal{S}^\phi$, $\sum_{s\in\phi^{-1}(s^\phi)}w(s) = 1$ \cite{li2006stateabs}.
    Similarly a policy can be transformed
    into its abstract equivalent as:
    \begin{align*}
        \pi^\phi(a|s^\phi) &= \sum_{s\in\phi^{-1}(s^\phi)}w(s)\pi(a|s). 
    \end{align*}
    We define the following 
    state-weighting function for an arbitrary policy $\pi$: $w_\pi(s) = \frac{d_\pi(s)}{\sum_{s'\in \phi^{-1}(s^\phi)} d_\pi(s')}$ and only consider abstractions that satisfy:
    \begin{assumption} [Reward distribution equality] $\forall s_1,s_2\in \sset$ such that $\phi(s_1) = \phi(s_2)$, $\forall a, \mathcal{R}(s_1,a) = \mathcal{R}(s_2, a)$.
    \label{assumption:reward_equality}
    \end{assumption}
    Assumption \ref{assumption:reward_equality} implies that, regardless
    of the choice of the state-weighting function, for
    given action $a$, $\forall s \in s^\phi$, $\mathcal{R}^\phi(s^\phi,a) = \mathcal{R}(s,a)$ 
    i.e. the reward distribution of an abstract state 
    equals that of the ground states within that abstract state.

\section{Abstract Marginalized Importance Sampling}
\label{sec:abs_mis}
Marginalized IS methods may suffer from high variance in high-dimensional state-spaces. To potentially reduce this high variance, we propose
to first use $\phi$ to project $\mathcal{D}$ into
the abstract state-space to obtain: $\mathcal{D}^\phi \coloneqq \{(s^\phi, a, r^\phi, s'{^\phi})\}$ where $s^\phi = \phi(s)$ and $r^\phi(s^\phi,a) = r(s,a)\forall s\in s^\phi$, and then use the following estimator 
on $\mathcal{D}^\phi$ to estimate $\rho(\pi_e^\phi)$:
\begin{definition}[Abstract estimator] We define our estimator of
$\rho(\pi^\phi_e)$ as follows:
    \begin{align}
        \hat{\rho}(\pi^\phi_e) &:= \frac{1}{N}\sum_{i=1}^N \frac{d_{\pi^\phi_e}(s_i^\phi, a_i)}{d_{\mathcal{D}^\phi}(s_i^\phi, a_i)}r^\phi(s_i^\phi,a_i)
        \label{eq:our_estimator}
    \end{align}
where $N$ is the number of samples, $d_{\pi^\phi}(s^\phi, a) = d_{\pi^\phi}(s^\phi)\pi^\phi(a|s^\phi) $ with $d_{\pi^\phi}(s^\phi) = \sum_{s\in\phi^{-1}(s^\phi)}d_{\pi}(s)$ and $\pi^\phi$ constructed
using $w_\pi$.
\end{definition}

In the remainder of this section, we first 
give an example to build intuition for why the abstract
ratios $\frac{d_{\pi^\phi_e}(s^\phi, a)}{d_{\mathcal{D}^\phi}(s^\phi, a)}$ can have lower variance than the ground ratios and then show theoretically that the OPE estimator given in Equation (\ref{eq:our_estimator}) is strongly consistent and 
can produce lower variance OPE estimates of $\rho(\pi_e)$ than the ground equivalent (Equation (\ref{eq:qiang_estimator})) when the true ratios are known. Finally, when the true abstract ratios are unknown, we 
identify sufficient conditions under which unbiased estimation of the ratios is possible and adapt an existing DICE algorithm to estimate them.

\subsection{A Hard Example for Ground MIS Ratios}
    \begin{figure}[H]
        \centering
            \includegraphics[scale=0.375]{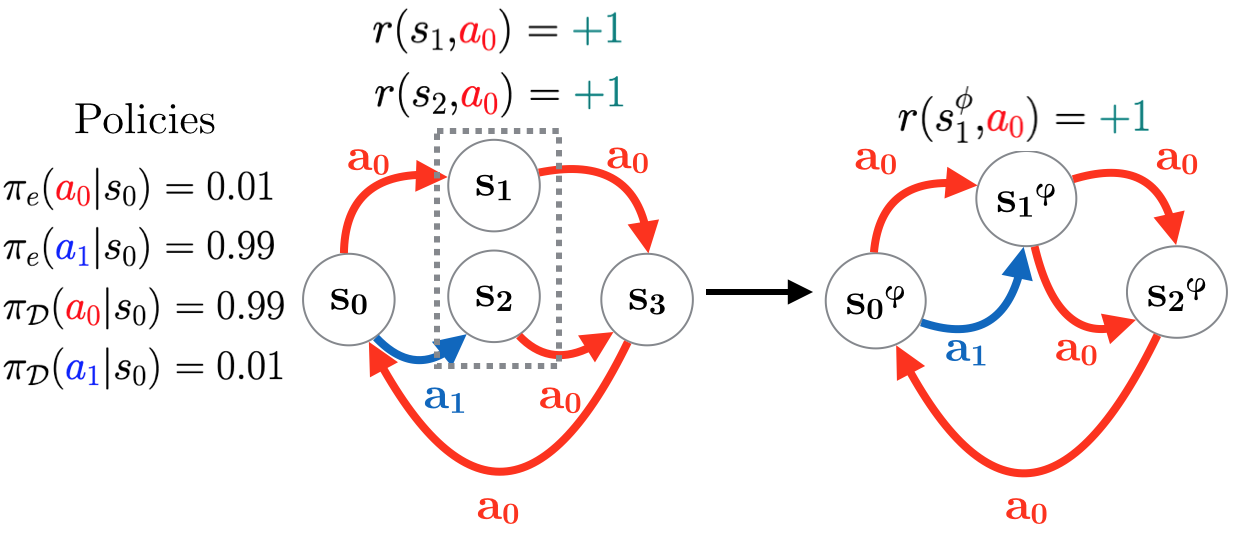}
        \caption{\footnotesize TwoPath MDP where ground density ratios for
        $(s_1, a_0)$ and $(s_2, a_0)$ are
        high variance. However, upon aggregation of equivalent
        states (grey dotted lines), the abstract 
        density ratio of $(s^\phi_1, a_0)$ is 
        low variance.}
        \label{fig:hard_eg}
    \end{figure}
    We present a hard example for ground MIS ratios in Figure \ref{fig:hard_eg} that
    provides intuition for why the abstract
    MIS ratios can have lower variance ratios than the ground ratios. 
    Consider two symmetric policies, $\pi_e$ and $\pi_\mathcal{D}$, executed
    in the ground MDP (left side). In this example, the high variance of
    the true state-action density ratios  
    $\frac{d_{\pi_e}(s_1, a_0)}{d_{\pi_\mathcal{D}}(s_1, a_0)}\approx 0$
    and $\frac{d_{\pi_e}(s_2, a_0)}{d_{\pi_\mathcal{D}}(s_2, a_0)}\approx 100$ can
    lead to high variance estimates of $\rho(\pi_e)$. Notice, however, that
    states $s_1$ and $s_2$ are essentially 
    equivalent i.e. $r(s_1, a) = r(s_2, a)\forall a\in\aset$ and can be aggregated
    together into a single state, $s^\phi_1$ (Assumption \ref{assumption:reward_equality}). We find that the
    state-action density ratio in this abstract MDP (right side)  $\frac{d_{\pi^\phi_e}(s_1^\phi, a_0)}{d_{\pi^\phi_\mathcal{D}}(s_1^\phi, a_0)} = 1$ is of low variance, which can lead to
    low variance estimates of $\rho(\pi_e)$. 
    
    In general, we prove that
    the abstract ratios are guaranteed to have variance at most that of
    the
    ground ratios (proof in Appendix
    \ref{app:proofs}):
    \begin{restatable}{theorem}{theoremOursLowerVarianceRatios}
    $\text{Var}\left(\frac{d_{\pi^\phi_e}(s^\phi, a)}{d_{\mathcal{D}^\phi}(s^\phi, a)}\right)\leq \text{Var}\left(\frac{d_{\pi_e}(s, a)}{d_{\mathcal{D}}(s, a)}\right)$
    \label{theorem:our_lower_var_ratios}
    \end{restatable}
    where equality holds if either or both of the following are true: 1) $\phi$ is the identity
    function i.e. $\phi(s) = s, \forall s\in\sset$ and 2) if $\forall s_1,s_2\in \sset$ such that $\phi(s_1) = \phi(s_2) = s^\phi$
    and for a given action $a$,
    $\frac{d_{\pi_e}(s_1, a)}{d_{\mathcal{D}}(s_1, a)} = \frac{d_{\pi_e}(s_2, a)}{d_{\mathcal{D}}(s_2, a)}, \forall s^\phi \in \sset^\phi, a\in\aset$. Thus, Theorem \ref{theorem:our_lower_var_ratios} implies that
    projecting $\sset \to \sset^\phi$ can lower the variance of density ratios.
    
\subsection{MIS OPE with True Abstract Ratios}
    We now present the statistical properties of our
    estimator assuming it has access to the true abstract state-action ratios.
    Due to space constraints, we defer proofs to Appendix
    \ref{app:proofs}.
    
    We prove our abstract estimator is unbiased (Theorem \ref{theorem:our_unbiased} in Appendix \ref{app:proofs}) and strongly consistent (Theorem \ref{theorem:our_mse_consistency} and Corollary \ref{theorem:our_consistency}):
    \begin{restatable}{theorem}{theoremOursMSEConsistent} Our estimator, $\hat{\rho}(\pi^\phi_e)$,
    given in Equation \ref{eq:our_estimator} is an asymptotically 
    consistent estimator of $\rho(\pi_e)$ in terms of MSE: $\lim_{N\to\infty}\E[(\hat{\rho}(\pi^\phi_e) - \rho(\pi_e))^2] = 0$.
    \label{theorem:our_mse_consistency}
    \end{restatable}
    We also compare the variances between our
    abstract estimator and the ground equivalent. If we assume that
     $\mathcal{D}$ is i.i.d as done in previous work \cite{sutton2008dyna, masatoshi2019minimax, nachum2019dualdice, zhang2020gendice}, then
    $\text{Var}(\hat{\rho}(\pi^\phi_e))\leq \text{Var}(\hat{\rho}(\pi_e))$,
    where equality holds only under the 
    same conditions as described below Theorem
    \ref{theorem:our_lower_var_ratios}. In general, however, variance
    reduction depends on the covariances between the weighted per-step rewards \cite{liu2019cursevariance}:
    \begin{restatable}{theorem}{theoremOursVariance} If Assumption \ref{assumption:reward_equality} and if for any fixed $1\leq t < k \leq T$, $\text{Cov}\left(\frac{d_{\pi^\phi_e}(s_t^\phi,a_t)}{d_{\mathcal{D}^\phi}(s_t^\phi,a_t)}r^\phi(s_t^\phi, a_t), \frac{d_{\pi^\phi_e}(s_k^\phi,a_k)}{d_{\mathcal{D}^\phi}(s_k^\phi,a_k)}r^\phi(s_k^\phi, a_k)\right) \leq
    \text{Cov}\left(\frac{d_{\pi_e}(s_t,a_t)}{d_{\mathcal{D}}(s_t,a_t)}r(s_t, a_t), \frac{d_{\pi_e}(s_k,a_k)}{d_{\mathcal{D}}(s_k,a_k)}r(s_k, a_k)\right)$ hold then $\text{Var}(\hat{\rho}(\pi^\phi_e))\leq \text{Var}(\hat{\rho}(\pi_e))$.
    \label{theorem:our_variance}
    \end{restatable}

\subsection{Estimating the Abstract Ratios}
    Thus far, we have assumed access to the true abstract ratios. However, in practice, these ratios are unknown and must be estimated from $\mathcal{D}^\phi$.
    In this section, we highlight the challenges in
     estimating the abstract ratios and identify sufficient conditions on $\phi$ that allow for accurate estimation. Once
    $\phi$ satisfies these conditions, any off-the-shelf state-action density
    estimation method can be used to estimate the abstract
    ratios. In this work, we focus on
    showing that DICE estimates the true abstract ratios. We note that the following new conditions on $\phi$ are only needed for unbiased estimation of the abstract ratios; accurate OPE on the abstracted MDP can still be realized under only Assumption \ref{assumption:reward_equality}.
    
    We first observe that evaluating $\pi_e$ using $\mathcal{D}$ is equivalent to evaluating $\pi^\phi_e$ using $\mathcal{D}^\phi$ if $\mathcal{D}^\phi$ is generated from an abstract MDP with transition dynamics constructed according to $w_{\pi_e}$. This equivalence is given by the following proposition:
    \begin{restatable}{proposition} {lemRpiToAbsRpi} If Assumption 
    \ref{assumption:reward_equality} holds, the average reward 
    of ground policy $\pi$ executed
    in ground MDP $\mathcal{M}$, $\rho(\pi)$, is equal to the average reward of
    abstract policy $\pi^\phi$ executed in abstract MDP $\mathcal{M}^\phi$
    constructed with $w_\pi$,
    $\rho(\pi^\phi)$. That is, $\rho(\pi) = \rho(\pi^\phi)$.
    \label{lemma:g_val_abs_val_equality}
    \end{restatable}
    Proposition \ref{lemma:g_val_abs_val_equality} suggests that we can estimate $\rho(\pi_e)$ by applying any OPE algorithm to evaluate $\pi_e^\phi$ using $\mathcal{D}^\phi$ if the abstract transition dynamics of $\mathcal{D}^\phi$ are distributed according to $P_{w_{\pi_e}}^\phi$. Unfortunately, since $w_{\pi_e}$ is unknown, two challenges arise.
    Fortunately, there are special cases where unbiased estimation of the abstract ratios is still possible using only existing MIS algorithms.
    
    \textbf{Challenge 1: Transition dynamics distribution shift}. The off-policy data $\mathcal{D}^\phi:= \{(s^\phi, a, r^\phi, s'^\phi)\}$ is distributed in the following way: $(s^\phi, a)\sim d_{\mathcal{D}^\phi},
    r^\phi\sim\mathcal{R}^\phi, s'^{\phi} \sim
    P^{\phi}_{w_\mathcal{D}}(s^\phi, a)$ where $P_{w_\mathcal{D}}^\phi$ are the transition dynamics
    of the abstract MDP constructed with $w_\mathcal{D}$
    as the state-weighting function. Thus, in addition
    to the original policy distribution shift problem, we also encounter a
    \textit{transition dynamics distribution
    shift} problem due to the projection where we want to evaluate $\pi^\phi_e$
    in an MDP with $P_{w_{\pi_e}}^\phi$, but we only have
    samples of data generated in an MDP with $P_{w_\mathcal{D}}^\phi$. Moreover, since $w_{\pi_e}$ is unknown, we cannot compute $P_{w_{\pi_e}}^\phi$ and  
    correct the distribution shift, say through importance sampling \cite{precup2000ISOPE}, as we would correct for policy distribution shift. However, one condition on $\phi$ that will avoid the transition distribution shift is:
    \begin{assumption} [Transition dynamics similarity] $\forall s_1,s_2\in \sset$ such that $\phi(s_1) = \phi(s_2)$, $\forall a\in\mathcal{A}, x\in\sset^\phi$, we have $\sum_{s'\in \phi^{-1}(x)}P(s'|s_1, a) = \sum_{s'\in \phi^{-1}(x)}P(s'|s_2, a)$.
    \label{assumption:trans_bisim}
    \end{assumption}
    Together with Assumption \ref{assumption:reward_equality},
    $\phi$ is now the so-called \textit{bisimulation} abstraction \cite{ferns2011bisim, castro2020bisim}. This
    property of $\phi$ eliminates the transition dynamics 
    distribution shift since now $P_{w_{\pi_e}}^\phi(s'^\phi|s^\phi, a) = P_{w_\mathcal{D}}^\phi(s'^\phi|s^\phi, a) = P^\phi(s'^\phi | s^\phi, a) = \sum_{s'\in s'^\phi}P(s'|s, a) \forall s\in s^\phi, \forall a\in \aset$ (applying Assumption \ref{assumption:trans_bisim} and definition of $P^\phi$ from Section \ref{sec:bg_state_abs}). Thus, we have $\mathcal{D}^\phi$ distributed with $P_{w_{\pi_e}}^\phi = P^\phi$, which then allows us to
    apply any MIS algorithm to compute the abstract state-action
    ratios using $\mathcal{D}^\phi$. 
    
    \textbf{Challenge 2: Inaccessible $\pi^\phi_e$}. To the best of our knowledge, all existing MIS algorithms require access to the evaluation policy to estimate the density ratios. However, in the abstract MDP, $\pi^\phi_e$ is inaccessible since $w_{\pi_e}$ is unknown. To overcome this issue, we identify the following condition on $\phi$:
    \begin{assumption} [$\pi_e$ action-distribution equality]
    $\forall s_1,s_2\in \sset$ such that $\phi(s_1) = \phi(s_2)$, $\pi_e(s_1) = \pi_e(s_2)$
    \label{assumption:pi_act_equality}
    \end{assumption}
    Assumption \ref{assumption:pi_act_equality} then gives us
    $\pi^\phi_e(s^\phi) = \pi_e(s), \forall s\in s^\phi$ (applying Assumption \ref{assumption:pi_act_equality} and definition of $\pi^\phi$ from Section \ref{sec:bg_state_abs}), which
    allows us to simulate sampling from $\pi^\phi_e(\cdot | \phi(s))$ by just
    sampling from $\pi_e(\cdot|s)$.

    Given a $\phi$ with these properties, we can compute the abstract ratios needed to estimate $\rho(\pi_e)$ by applying a suitable MIS algorithm to $\mathcal{D}^\phi$.
    We use BestDICE \cite{nachum2020bestdice}, and call our algorithm AbstractBestDICE, which solves the following optimization problem:
    \begin{equation}
        \begin{split}
        &\min_{\nu, \lambda} \max_{\zeta} J(\nu, \zeta, \lambda) := -\E_{\mathcal{D}^\phi}\left[\frac{1}{2} \zeta(s^\phi,a)^2\right]\\ &+ \E_{\mathcal{D}^\phi}\left[\zeta(s^\phi,a) \left(\gamma\E_{a'\sim\pi^\phi_e}[\nu(s'^\phi, a')] 
        - \nu(s^\phi,a) - \lambda\right)\right] \\
        &+ (1 - \gamma) \E_{s^\phi_0 \sim d_{0^\phi}, a_0\sim \pi^\phi_e}[\nu(s^\phi_0, a_0)] + \lambda
        \end{split}
        \label{eq:main_opt}
    \end{equation}
    where, $\nu:\sset^\phi\times\aset\to \mathbb{R}$, $\lambda\in\mathbb{R}$, and $\zeta:\sset^\phi\times\aset\to \mathbb{R}_{\geq0}$. The solution to the optimization, $\zeta^*(s^\phi, a) = d_{\pi^\phi_e}(s^\phi, a)/d_{\mathcal{D}^\phi}(s^\phi, a)$ is the true abstract
    ratios. Since the derivation of AbstractBestDICE follows the same steps
    as BestDICE, we defer it to Appendix \ref{app:abs_dice_der}.
    
    We note that it may be difficult to validate  Assumptions \ref{assumption:trans_bisim}
    and \ref{assumption:pi_act_equality} in practice, which may result in
    loss of the consistency
    guarantee of Theorem \ref{theorem:our_mse_consistency}. Nevertheless, in Section \ref{sec:empirical_study} we show that AbstractBestDICE leads to accurate OPE in high-dimensional state-spaces even when assumptions may not hold.

\section{Empirical Study}
\label{sec:empirical_study}
We will now show how projecting $\sset\to\sset^\phi$ can produce more accurate OPE estimates in practice.
We answer the following questions:
\begin{enumerate}
    \item Do the true abstract ratios produce lower variance OPE
    estimates than the true ground ratios?
    \item Does AbstractBestDICE: (a) compute the true ratios when Assumptions
    \ref{assumption:reward_equality}, 
    \ref{assumption:trans_bisim}, and \ref{assumption:pi_act_equality} are satisfied and (b) produce data-efficient and stable estimates of $\rho(\pi_e)$ even when Assumptions \ref{assumption:trans_bisim} and \ref{assumption:pi_act_equality}  fail to hold?

\end{enumerate}

\subsection{Empirical Setup}
    In this section, we describe the algorithms and domains of our empirical
    study. Due to space constraints, we defer specific details to the appendix (\ref{app:tabular_exp} and \ref{app:func_approx_exp}).

\subsubsection{Algorithms}  We compare AbstractBestDICE
to ground BestDICE \cite{nachum2020bestdice}. As also reported by \citet{nachum2020bestdice, fu2021opebench}, we found in preliminary experiments (see Appendix \ref{app:func_approx_exp}) that BestDICE performed much better than other MIS methods such as DualDICE \cite{nachum2019dualdice}, Minimax-Weight Learning \cite{masatoshi2019minimax}, etc.

\subsubsection{Domains} We focus on high-dimensional state-space tasks, which have been known to be particularly challenging for
DICE methods \cite{fu2021opebench}. For each environment 
below we specify a fixed $\phi$.
\begin{itemize}
    \item \textbf{TwoPath MDP}. This MDP is pictured in Figure \ref{fig:hard_eg}. In this domain, Assumptions
    \ref{assumption:reward_equality}, \ref{assumption:trans_bisim}, and
    \ref{assumption:pi_act_equality} are satisfied. We also
    run the same experiments for when these assumptions are violated (see
    Appendix \ref{app:tabular_exp}).
    
    \item \textbf{Reacher} \cite{brockman2016gym}. A robotic arm tries
    to move to a goal location. Here, $s\in\mathbb{R}^{11}$ and $a\in\mathbb{R}^{2}$. 
    Since the reward function is the Euclidean distance
    between the arm and goal, $\phi$ extracts only 
    the arm-to-goal vector, and angular velocities from the ground state, resulting in $s^\phi\in\mathbb{R}^4$. 
    \item \textbf{Walker2D} \cite{brockman2016gym}. A bi-pedal robot tries
    to move as fast as possible. Here, $s\in\mathbb{R}^{18}$ and $a\in\mathbb{R}^{6}$. 
    We use the Euclidean distance from the start location as the reward function and use a $\phi$ that extracts $x$ and $z$ coordinates and top angle of the walker's body, resulting in $s^\phi\in\mathbb{R}^3$.
    \item \textbf{Pusher} \cite{brockman2016gym}. A robotic arm
    tries to push an object to a goal location. Here, $s\in\mathbb{R}^{23}$ and $a\in\mathbb{R}^{7}$. Since the reward function is the Euclidean distance
    between object and arm and object and goal, $\phi$ extracts only object-to-arm and object-to-goal vectors, resulting in 
    $s^\phi\in\mathbb{R}^6$.
    \item \textbf{AntUMaze} \cite{fu2020d4rl}. This sparse-reward task requires
    an ant to move from one end of the U-shaped maze to the other end.
    Here, $s\in\mathbb{R}^{29}$ and $a\in\mathbb{R}^{8}$. We use the ``play"
    version where the goal location is fixed. Since the reward function is $+1$ only if the $2$D location of the ant is at 
    a certain Euclidean distance from the $2$D goal location, $\phi$ extracts only
    the $2$D coordinates of the ant, resulting in
    $s^\phi\in\mathbb{R}^2$.
\end{itemize}

For Reacher, Walker2D, Pusher, and AntUMaze, $\phi$ satisfies only Assumption
\ref{assumption:reward_equality}.

\subsection{Empirical Results}
    In this section, we describe our main empirical results; additional experiments can be found in appendix \ref{app:func_approx_exp}.
    
\subsubsection{True Ratios for OPE}
We conduct an experiment on the TwoPath MDP to estimate
$\rho(\pi_e)$ where
we apply the ground estimator given in Equation (\ref{eq:qiang_estimator}) and our abstract estimator given in Equation (\ref{eq:our_estimator}),
assuming \textit{both have access to their respective true ratios}. The
results of this experiment are illustrated in Figure \ref{fig:true_ratio_quality}. We can observe that the abstract estimator
 with the true abstract ratios produces substantially more
data-efficient and lower
variance OPE estimates for different batch sizes compared
to the ground equivalent.

\begin{figure}[!h]
    \centering
        \subfigure[MSE vs. Batch size (\# of trajectories). The vertical axis axis is log-scaled. Errors are computed over $15$ trials with $95\%$ confidence intervals. Lower is better. Since $\rho(\pi_e) = \rho(\pi_\mathcal{D})$ in this domain, we use regular MSE instead of relative.]{\label{fig:true_ratio_quality}\includegraphics[scale=0.17]{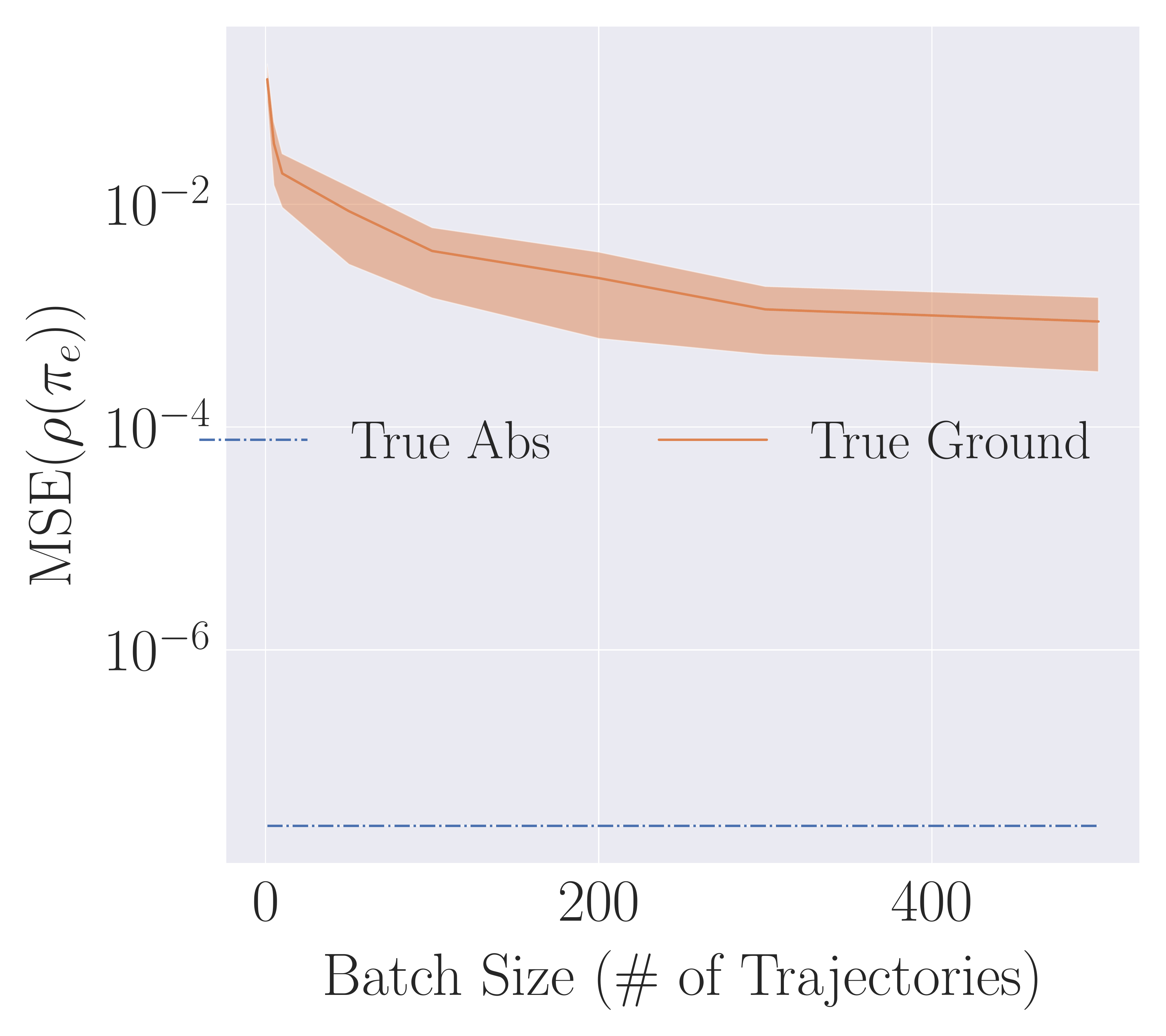}}
        \subfigure[Estimated $\hat{d}_{\pi^\phi_e}$ (vertical axis) vs. true $d_{\pi^\phi_e}$ (horizontal axis). Values are averaged over $15$ trials. We expect the dots to be as close
        to the 
        black line as possible. Each dot is for each $(s^\phi, a)$. Dots are only at extreme ends due to choice of $\pi_e$, $\pi_\mathcal{D}$, and ToyMDP design. Only $4$ out of $6$ dots are visible due to overlap between the dots for $(s_1^\phi, a_1)$ and 
        $(s_2^\phi, a_1)$, and $(s_1^\phi, a_0)$  and $(s_2^\phi, a_0)$.]{\label{fig:true_ratio_estimation}\includegraphics[scale=0.19]{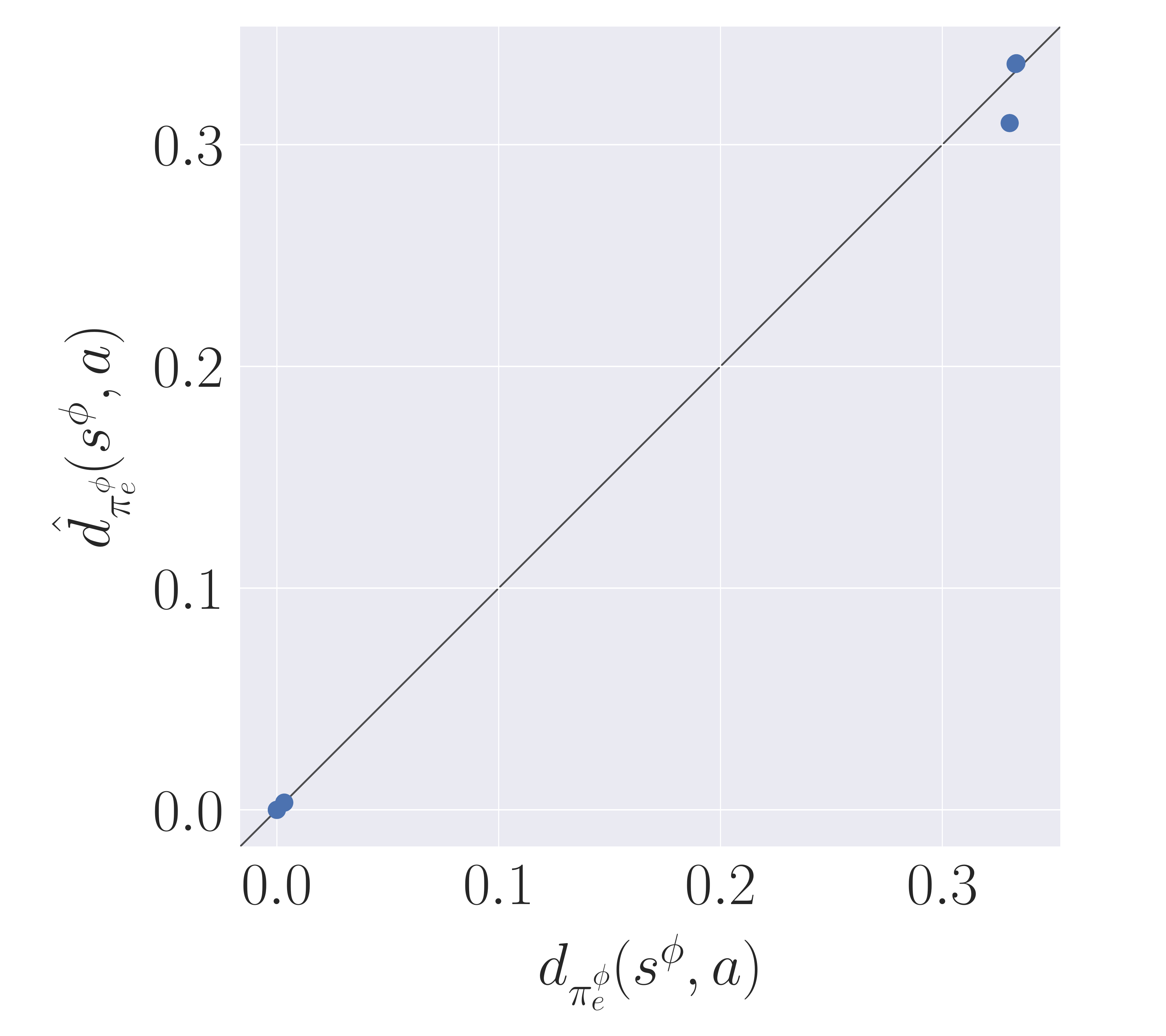}}
    \caption{\footnotesize True ratio experiments.}
    \label{fig:true_ratio_exp}
\end{figure}

\subsubsection{True Ratio Estimation} To verify if AbstractBestDICE accurately estimates the true ratios
we conduct the following experiment on the TwoPath MDP. We
give AbstractBestDICE data of batch size $300$ to estimate the abstract
ratios $\hat{\zeta}^\phi$ and then use $\hat{\zeta}^\phi$ to estimate the abstract
state-action densities of $\pi^\phi_e$, $\hat{d}_{\pi^\phi_e}(s^\phi,a) = \hat{\zeta}^\phi(s^\phi, a) d_{\pi^\phi_\mathcal{D}}(s^\phi, a)$, where
we have access to the true $d_{\pi^\phi_\mathcal{D}}(s^\phi, a)$. We 
then compare $\hat{d}_{\pi^\phi_e}$ to the true $d_{\pi^\phi_e}$, which we compute using a batch size of $300$ trajectories collected from $\pi^\phi_e$ roll-outs, 
using a correlation plot shown in Figure \ref{fig:true_ratio_estimation}. From the figure we can see
AbstractBestDICE accurately estimates the abstract state-action density
ratios. When assumptions are violated, however, ratio
estimation accuracy can reduce (see Appendix \ref{app:tabular_exp}).
\begin{figure*}[hbtp]
    \centering
        \subfigure[Reacher]{\label{fig:reacher_vs_batch}\includegraphics[scale=0.118]{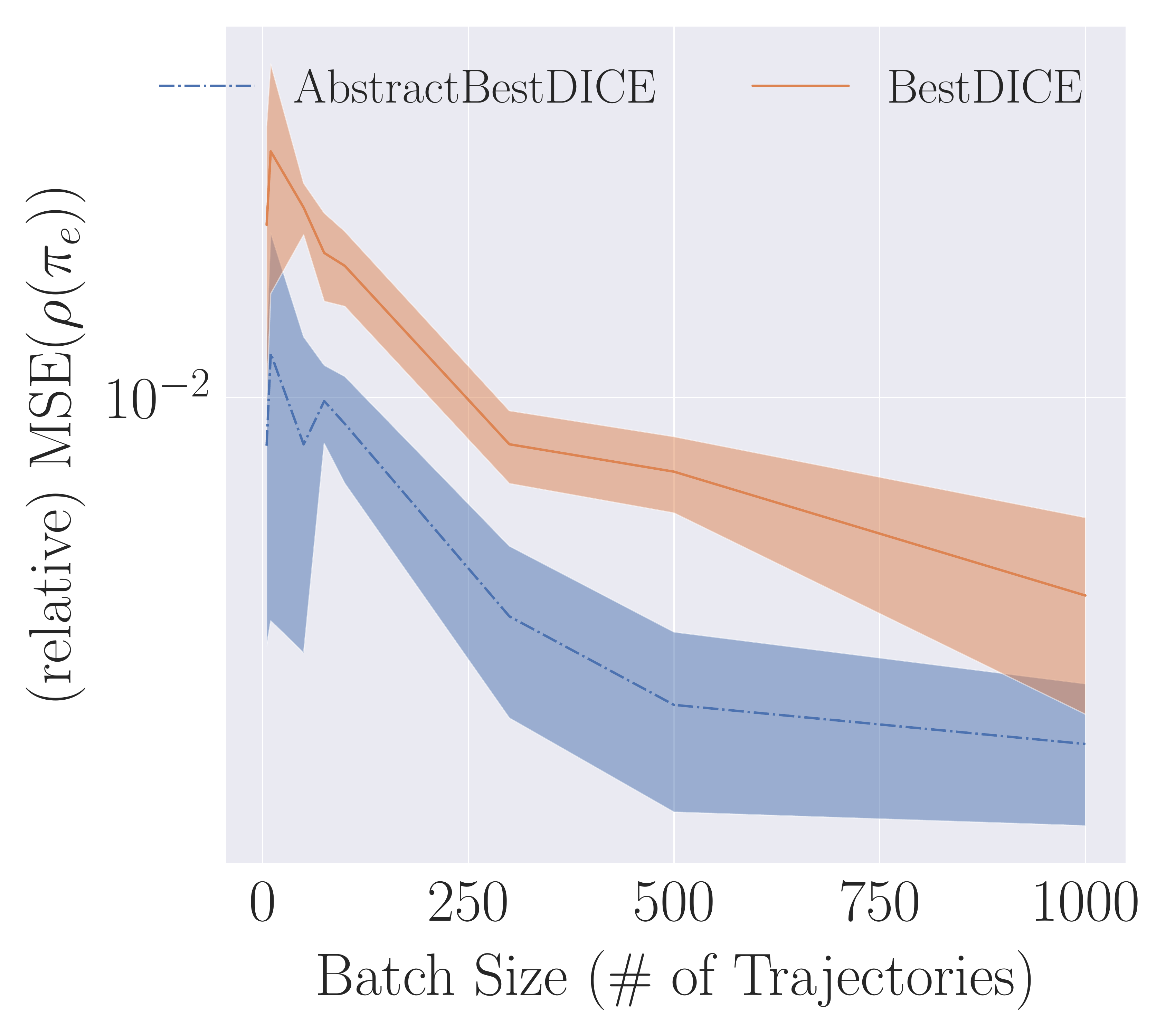}}
        \subfigure[Walker2D]{\label{fig:walker_vs_batch}\includegraphics[scale=0.118]{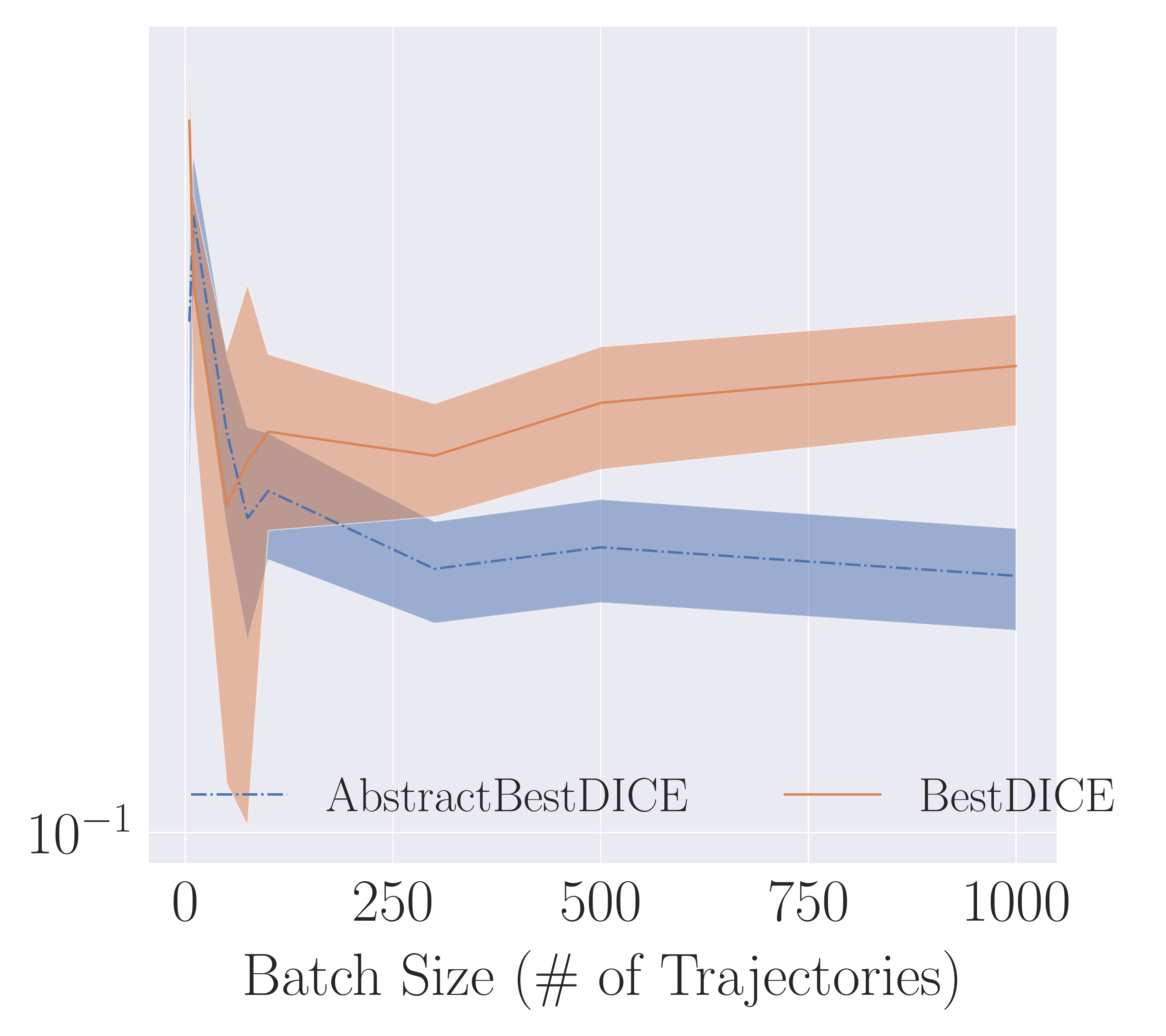}}
        \subfigure[Pusher]{\label{fig:pusher_vs_batch}\includegraphics[scale=0.118]{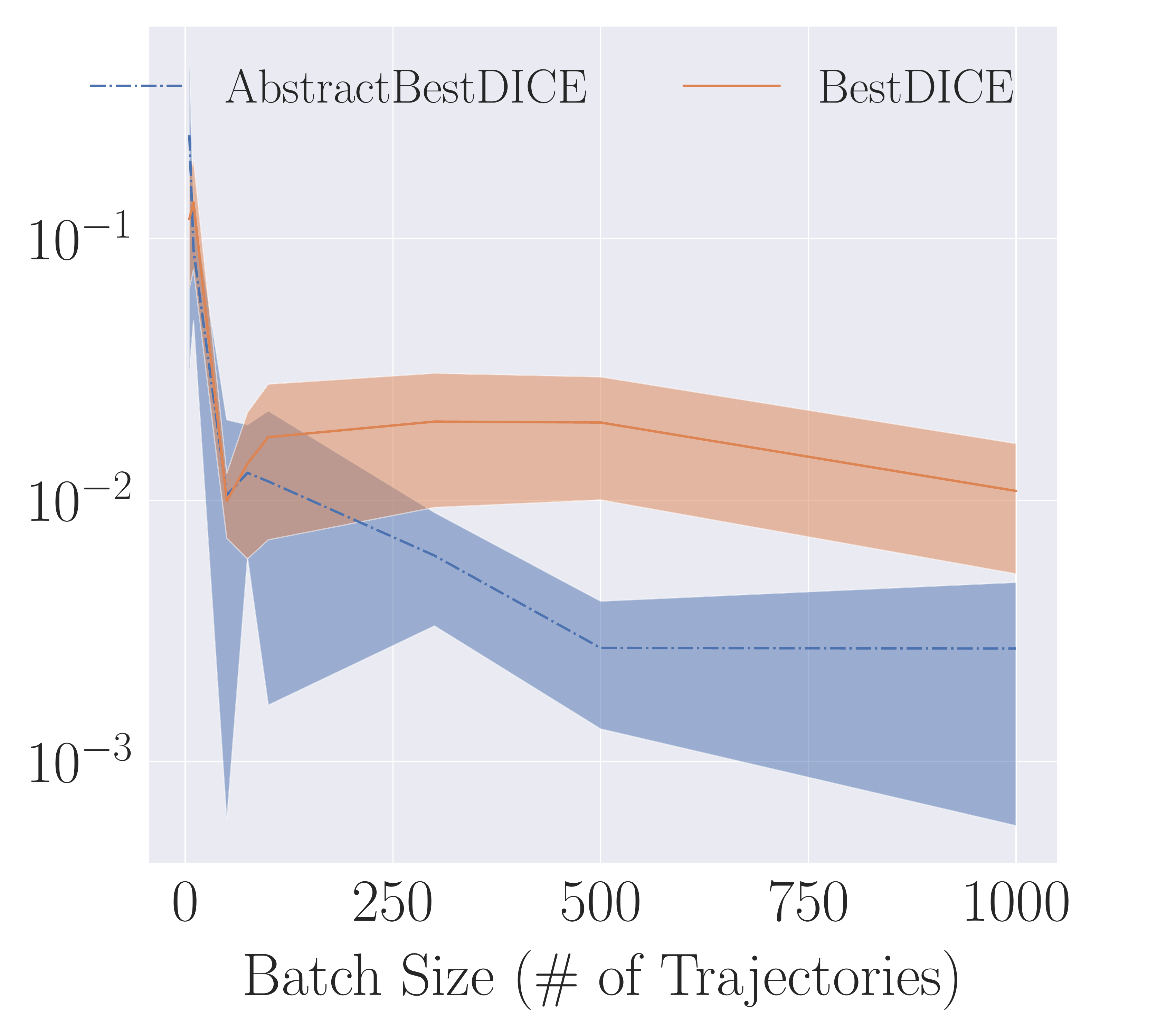}}
        \subfigure[AntUMaze]{\label{fig:antumaze_vs_batch}\includegraphics[scale=0.118]{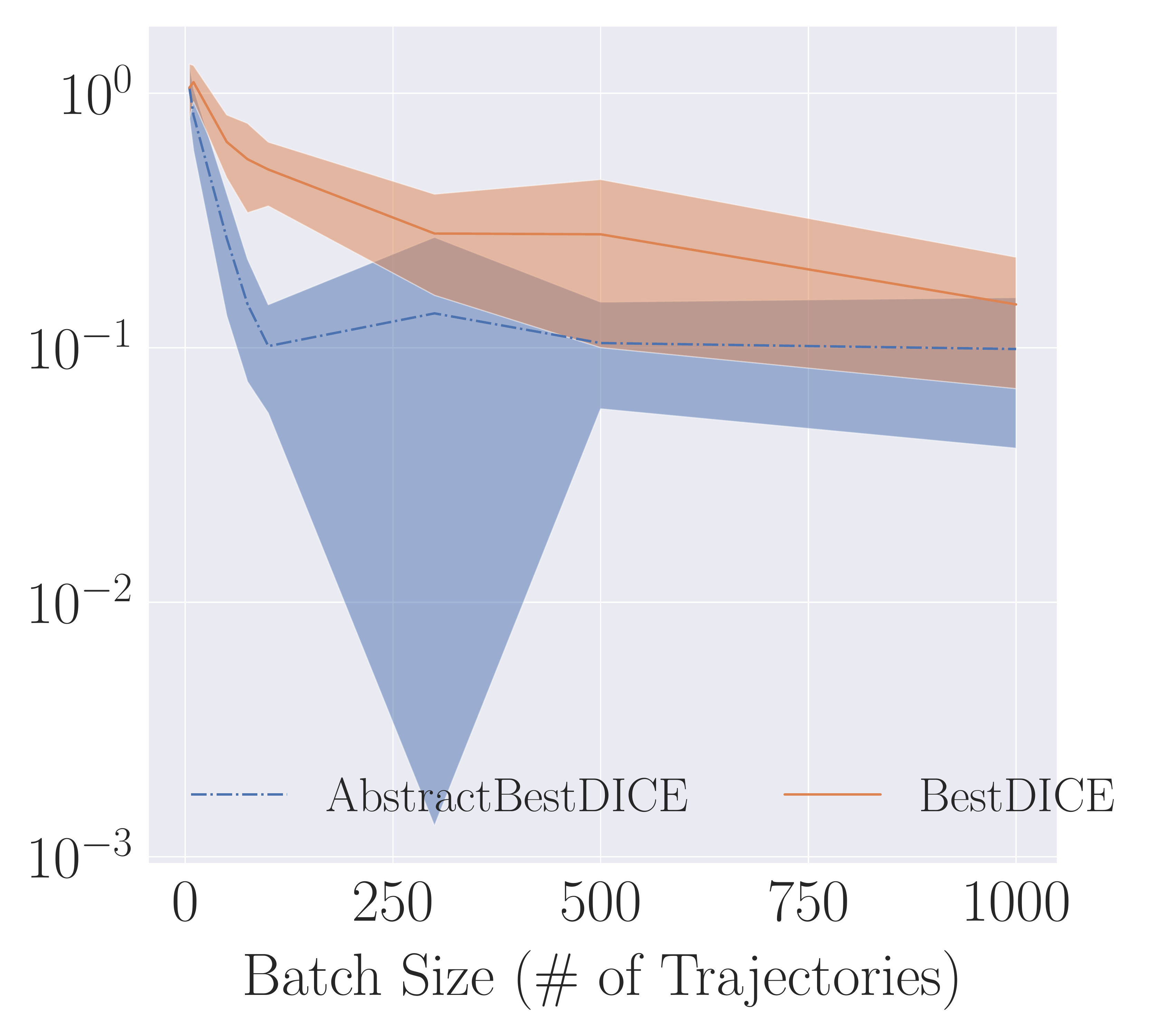}}
    \caption{\footnotesize Relative MSE vs. Batch Size (\# of trajectories). Vertical axis is log-scaled. Errors are computed over $15$ trials with $95\%$ confidence intervals. Lower is better.}
    \label{fig:vs_batch}
\end{figure*}
\subsubsection{Data-Efficiency} Figure \ref{fig:vs_batch}
shows the results of our (relative) MSE vs. batch size experiment
for the function approximation case. For a given batch size, we 
train each algorithm for $100$k epochs with different hyperparameters sets,
record the (relative) MSE on the last epoch by each hyperparameter set, and plot
the lowest MSE achieved by these hyperparameter sets. We find that
AbstractBestDICE is able to achieve lower MSE than BestDICE for a given batch size. We note that while
hyperparameter tuning is difficult in OPE, in this experiment, we aim to evaluate
the performance of each algorithm assuming both had favorable hyperparameters.


\subsubsection{Hyperparameter Robustness} Finally, we study
the robustness of these algorithms to hyperparameters tuning. In practical OPE, hyperparameter tuning with respect to MSE is
    impractical since the true $\rho(\pi_e)$ is unknown \cite{fu2021opebench, paine2020hparam}. Thus,
    we want OPE algorithms to be as robust as possible to hyperparameter
    tuning. The main hyperparameters for DICE
    are the learning
    rates of $\zeta$ and $\nu$, $\alpha_\zeta$ and $\alpha_\nu$. For these
    experiments, we focus on very small batch sizes, where we would 
    expect high
    sensitivity. The results
of this study are in Figure \ref{fig:hp_sens}.
    We find that our algorithm has a less volatile
    MSE than BestDICE (also see appendix \ref{app:func_approx_exp} for more similar results). 
    In a related experiment, we also find
    AbstractBestDICE can be more stable than BestDICE during training (see
    appendix \ref{app:func_approx_exp}).

\begin{figure}[]
    \centering
    \subfigure{\label{fig:antumaze_hp_sens}\includegraphics[scale=0.17]{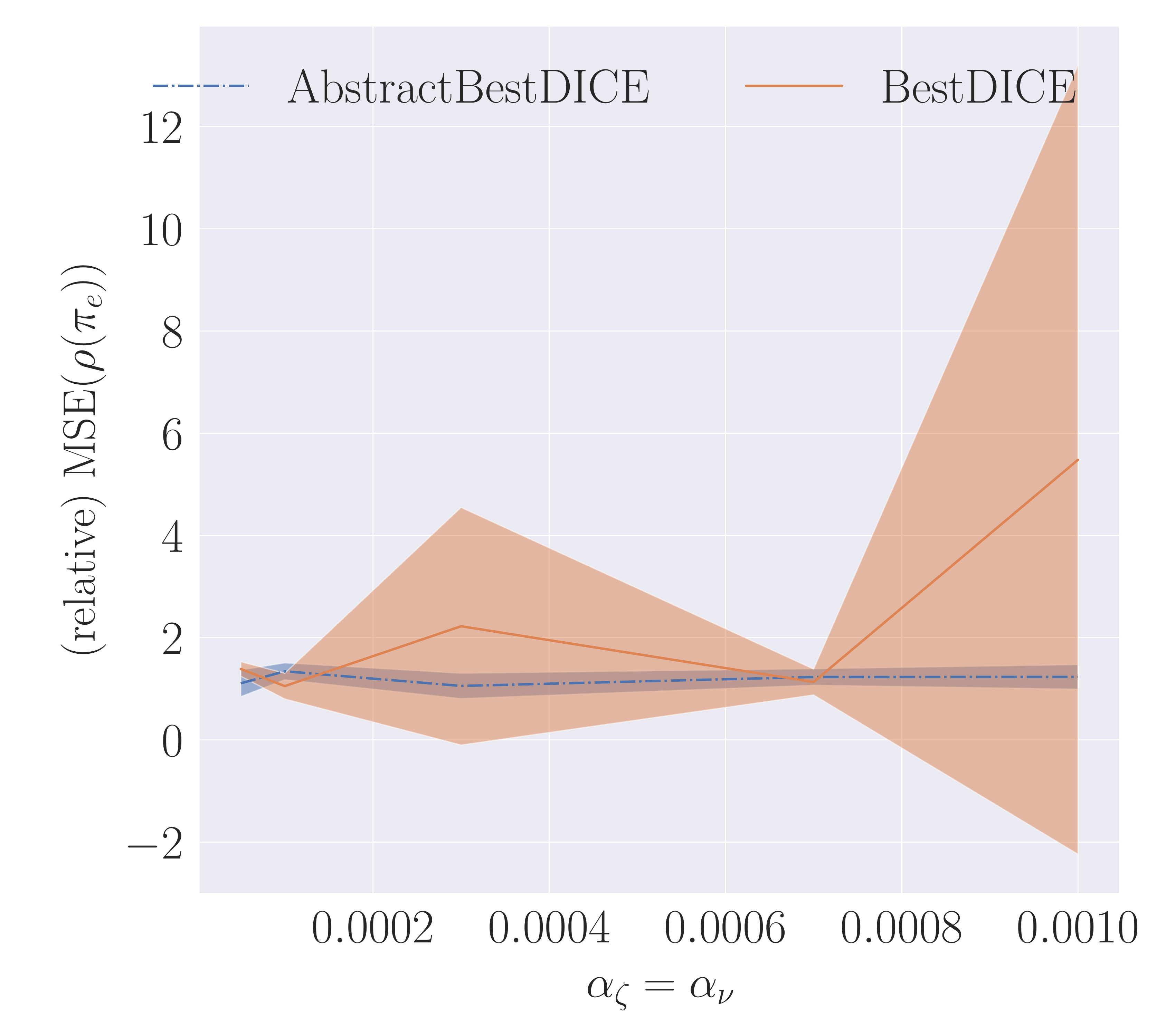}}
    \caption{\footnotesize Robustness of BestDICE and AbstractBestDICE to hyperparameters on
    the AntUMaze domain for batch size (\# of trajectories) of $5$. Errors are computed over $15$ trials with $95\%$ confidence intervals. Lower is better.}
    \label{fig:hp_sens}
\end{figure}

\section{Related Work}
    \textbf{MIS and Off-Policy Evaluation}. There have been
    broadly three families
    of MIS algorithms in the OPE literature to estimate
    state-action density ratios. One is the DICE family, which
    includes:
    minimax-weight learning \cite{masatoshi2019minimax}, DualDice
    \cite{nachum2019dualdice}, GenDICE 
    \cite{zhang2020gendice}, GradientDICE \
    \cite{zhang2020gradientdice}, and BestDICE 
    \cite{nachum2020bestdice}. In our work,
    we adapt BestDICE to estimate the abstract ratios.
    The second family of MIS algorithms is the COP-TD algorithm \cite{hallak2017coptd, gelada2019coptd}, which
    learns the state
    density ratios with an online TD-styled update. The
    third family is the variational power method 
    \cite{wen2020vpm} algorithm which generalizes
    the power iteration method to estimate density ratios. While
    our focus has been on MIS algorithms, there are many other OPE algorithms such as model-based methods \cite{zhang2021modelbasedOPE, hanna2017bootstrapping, liu2018mbope},
    fitted-Q evaluation \cite{le2019FQE}, doubly-robust
    methods \cite{jiang2015dr, thomas2016dataefficient},
    and IS \cite{precup2000ISOPE, thomas2015thesis, hanna2019importance, thomas2015hcope}.
    
    \textbf{State Abstraction and Representation Learning}. The literature on state abstraction
    is extensive \cite{sing1994abs, dietterich1999hrl, ferns2011bisim, li2006stateabs, abel2020thesis}. However, much of this work has been exclusively focused on building
    a theory of abstraction and on learning
    optimal policies. 
    A related topic
    to state abstraction is representation learning. Recently, there
    has been much work showing the importance of good representations
    for offline RL \cite{wang2021instability, yin2022offlinelinear, geng2022effective, zhan2022realizeconc, chen2022offline}. To the best of our knowledge, no work
    has leveraged state abstraction techniques to improve the accuracy 
    of OPE algorithms.
    
    
    
\section{Summary and Future Work}
    In this work, we showed
    that we can improve the accuracy of OPE estimates by
    projecting the original ground state-space into a lower-dimensional abstract state-space using state abstraction and performing OPE in the resulting abstract Markov decision process. Our 
    theoretical results proved that: 1)
    abstract state-action ratios have variance at most that
    of the ground ratios; and 2) the abstract MIS OPE estimator
    is unbiased, strongly consistent, and can have lower variance than
    the ground equivalent. We then highlighted the challenges that arise when estimating the abstract ratios from data, identified sufficient conditions to overcome these issues, and adapted BestDICE into AbstractBestDICE to estimate the abstract ratios. In our empirical results, we obtained more accurate OPE estimates with added hyperparameter robustness on difficult, high-dimensional state-space tasks.
    
    There are several directions for future work. First,
    Assumptions \ref{assumption:trans_bisim}  and \ref{assumption:pi_act_equality} are strict. Further investigation is needed to see if these
    assumptions can be relaxed. Second, we assumed
    the abstraction function was given. It would be interesting to leverage existing ideas \cite{gelada2019bisim, zhang2021bisim} to learn $\phi$. Finally, we want to emphasize that this work 
    instantiates the general abstraction + OPE direction. While this work 
    focused exclusively
    on MIS algorithms, a promising
    direction will be to apply
    abstraction techniques to model-based, trajectory IS,
    and value-function based OPE.

\section*{Acknowledgements}
Support for this research was provided by the Office of the Vice Chancellor for Research and Graduate Education at
the University of Wisconsin — Madison with funding from the Wisconsin Alumni Research Foundation. The authors thank the anonymous reviewers, Nicholas Corrado, Ishan Durugkar, and Subhojyoti Mukherjee for their helpful comments in improving this work.

\bibliography{aaai23}

\onecolumn
\appendix

\section{Appendix}

\subsection{Preliminaries}
This section provides the supporting lemmas and definitions 
that we leverage to prove our lemmas and theorems.
\begin{definition}[Almost Sure Convergence] A sequence of
random variables, $(X_n)_{n=1}^\infty$, almost surely converges to the random variable,
$X$ if
\begin{align*}
    \PR\left(\lim_{n\to\infty} X_n = X\right) = 1
\end{align*}
We write $X_n \overset{a.s.}\to X $ to denote that the sequence $(X_n)_{n=1}^\infty$
converges almost surely to $X$.
\end{definition}

\begin{definition}[(Strongly) Consistent Estimator] Let $\theta$
be a real number and $(\hat{\theta}_n)_{n=1}^\infty$ be an infinite
sequence of random variables. We call $\hat{\theta}_n$ a
(strongly) consistent estimator of $\theta$ if and only if
$\hat{\theta}_n \overset{a.s.}\to \theta$. 
\end{definition}

\begin{lemma}
If $(X_i)_{i=1}^\infty$ is a sequence of uniformly
bounded real-valued random variables, 
then 
$X_n \overset{a.s.}\to X$ if and only if $\lim_{N\to\infty}\E[(X_n - X)^2] = 0$.
\label{lemma:thomas_mse_consistency}
\end{lemma}
\begin{proof}
	See Lemma 3 in \citet{thomas2016dataefficient}.
\end{proof}

\subsection{Assumptions and Definitions}
    In the main paper, we provided the major assumptions required for our theoretical and empirical work relevant to abstraction and OPE.
    Here we provide supporting assumptions typically used in the OPE literature used for the theoretical analysis.
    \begin{assumption} [Coverage] For all 
    $(s,a) \in \sset \times \aset$, if $\pi_e(a|s) > 0$ then 
    $\pi_b(a|s) > 0$.
    \label{assumption:coverage}
    \end{assumption}

    
    \begin{assumption} [Non-negative reward] We assume that the reward function 
    is bounded between $[0, \infty)$.
    \label{assumption:bounded_reward}
    \end{assumption}
    
    \begin{definition} [Ground state normalized weightings] For a given
        policy $\pi$, each ground state $s\in s^\phi$,  has
        a state aggregation weight, $w_\pi(s) = \frac{d_\pi(s)}{\sum_{s'\in \phi^{-1}(s^\phi)} d_\pi(s')}$,
        where $d_\pi(s)$ is the discounted state-occupancy measure of $\pi$.
        \label{def:norm_weights_agg_states}
    \end{definition}

\subsection{Proofs}
\label{app:proofs}

In the proofs below, we denote the collection of behavior policies
that generated $\mathcal{D}$ with $\pi_{\mathcal{D}}$. That is, $\pi_\mathcal{D}$ is the conditional probability of an action occurring in a given state in the data. Similarly, we also have $\pi^\phi_\mathcal{D}$. These minor changes give us 
$d_{\mathcal{D}} = d_{\pi_\mathcal{D}}$ and $d_{\mathcal{D}^\phi} = d_{\pi^\phi_\mathcal{D}}$.

\begin{restatable}{lemma}{lemExpAbsGroundF}
For an arbitrary function, $f$, $\E_{s^\phi \sim d_{\pi^\phi}, a \sim \pi^\phi}\left[f(s^\phi, a)\right] = \E_{s \sim d_\pi, a \sim \pi}\left[f(\phi(s), a)\right]$.
\label{lemma:exp_f_g_abs_equal}
\end{restatable}
\begin{proof}
\begin{align*}
    \E_{s^\phi \sim d_{\pi^\phi}, a \sim \pi^\phi}\left[f(s^\phi, a)\right] &= \sum_{s^\phi, a} d_{\pi^\phi}(s^\phi) \pi^\phi(a|s^\phi) f(s^\phi, a)\\
    &\labelrel={lemma:exp_f_g_abs_equal_1}\sum_{s^\phi, a}   d_{\pi^\phi}(s^\phi)\sum_{s \in \phi^{-1}(s^\phi)} \frac{\pi(a|s)d_\pi(s)}{d_{\pi^\phi}(s^\phi)} f(s^\phi, a)\\
    &=\sum_{s^\phi, a}   \sum_{s \in \phi^{-1}(s^\phi)} \pi(a|s)d_\pi(s) f(s^\phi, a)\\
    &= \sum_{s, a}  \pi(a|s)d_\pi(s) f(\phi(s), a)\\
    \E_{s^\phi \sim d_\pi^\phi, a \sim \pi^\phi}\left[f(s^\phi, a)\right] &= \E_{s \sim d_\pi, a \sim \pi}\left[f(\phi(s), a)\right]
\end{align*}
where \eqref{lemma:exp_f_g_abs_equal_1} is due to Definition \ref{def:norm_weights_agg_states} and we can replace $s^\phi$
with $\phi(s)$ when we know $s\in s^\phi$.
\end{proof}

\theoremOursLowerVarianceRatios*
\begin{proof}

Before comparing the variances, we note that due to Assumption
\ref{assumption:coverage} and Lemma
\ref{lemma:exp_f_g_abs_equal}:
\begin{align*}
    \E_{s\sim d_{\pi_\mathcal{D}}, a\sim \pi_\mathcal{D}}\left[\frac{d_{\pi_e}(s)\pi_e(a|s)}{d_{\pi_\mathcal{D}}(s)\pi_\mathcal{D}(a|s)}\right] = \E_{s\sim d_{\pi^\phi_\mathcal{D}}, a\sim \pi^\phi_\mathcal{D}}\left[\frac{d_{\pi^\phi_e}(s^\phi)\pi^\phi_e(a|s^\phi)}{d_{\pi^\phi_\mathcal{D}}(s^\phi)\pi^\phi_\mathcal{D}(a|s^\phi)}\right] = 1
\end{align*}

Denote, $V^g := \text{Var}\left( \frac{d_{\pi_e}(s)\pi_e(a|s)}{d_{\pi_\mathcal{D}}(s)\pi_\mathcal{D}(a|s)}\right)$ and $V^\phi := \text{Var}\left(\frac{d_{\pi^\phi_e}(s^\phi)\pi^\phi_e(a|s^\phi)}{d_{\pi^\phi_\mathcal{D}}(s^\phi)\pi^\phi_\mathcal{D}(a|s^\phi)}\right)$. Now consider the difference between the two variances. 
\begin{align*}
    D &= V^g - V^\phi\\
    &= \text{Var}\left( \frac{d_{\pi_e}(s)\pi_e(a|s)}{d_{\pi_\mathcal{D}}(s)\pi_\mathcal{D}(a|s)}\right) - \text{Var}\left(\frac{d_{\pi^\phi_e}(s^\phi)\pi^\phi_e(a|s^\phi)}{d_{\pi^\phi_\mathcal{D}}(s^\phi)\pi^\phi_\mathcal{D}(a|s^\phi)}\right)\\
    &= \E_{s\sim d_{\pi_\mathcal{D}}, a\sim \pi_\mathcal{D}}\left[\left(\frac{d_{\pi_e}(s)\pi_e(a|s)}{d_{\pi_\mathcal{D}}(s)\pi_\mathcal{D}(a|s)}\right)^2\right]
    - \E_{s^\phi \sim d_{\pi^\phi_\mathcal{D}}, a\sim \pi_\mathcal{D}^\phi}\left[\left(\frac{d_{\pi^\phi_e}(s^\phi)\pi^\phi_e(a|s^\phi)}{d_{\pi^\phi_\mathcal{D}}(s^\phi)\pi^\phi_\mathcal{D}(a|s^\phi)}\right)^2\right]\\
    &= \sum_{s, a}d_{\pi_\mathcal{D}}(s)\pi_\mathcal{D}(a|s)\left(\frac{d_{\pi_e}(s)\pi_e(a|s)}{d_{\pi_\mathcal{D}}(s)\pi_\mathcal{D}(a|s)}\right)^2
    - \sum_{s^\phi, a}d_{\pi^\phi_\mathcal{D}}(s^\phi)\pi^\phi_\mathcal{D}(a|s^\phi) \left(\frac{d_{\pi^\phi_e}(s^\phi)\pi^\phi_e(a|s^\phi)}{d_{\pi^\phi_\mathcal{D}}(s^\phi)\pi^\phi_\mathcal{D}(a|s^\phi)}\right)^2\\
    &= \sum_{s^\phi, a}\left(\sum_{s\in s^\phi}d_{\pi_\mathcal{D}}(s)\pi_\mathcal{D}(a|s)\left(\frac{d_{\pi_e}(s)\pi_e(a|s)}{d_{\pi_\mathcal{D}}(s)\pi_\mathcal{D}(a|s)}\right)^2
    - d_{\pi^\phi_\mathcal{D}}(s^\phi)\pi^\phi_\mathcal{D}(a|s^\phi) \left(\frac{d_{\pi^\phi_e}(s^\phi)\pi^\phi_e(a|s^\phi)}{d_{\pi^\phi_\mathcal{D}}(s^\phi)\pi^\phi_\mathcal{D}(a|s^\phi)}\right)^2\right)
\end{align*}

We can analyze this difference by 
looking at one abstract state and one action and all the states that belong to it. That
is, for a fixed abstract state, $s^\phi$, and fixed action, $a$, we have:

\begin{align*}
    D' &= \sum_{s\in s^\phi}d_{\pi_\mathcal{D}}(s)\pi_\mathcal{D}(a|s)\left(\frac{d_{\pi_e}(s)\pi_e(a|s)}{d_{\pi_\mathcal{D}}(s)\pi_\mathcal{D}(a|s)}\right)^2
    - \left(d_{\pi^\phi_\mathcal{D}}(s^\phi)\pi^\phi_\mathcal{D}(a|s^\phi) \left(\frac{d_{\pi^\phi_e}(s^\phi)\pi^\phi_e(a|s^\phi)}{d_{\pi^\phi_\mathcal{D}}(s^\phi)\pi^\phi_\mathcal{D}(a|s^\phi)}\right)^2\right)\\
    &= \left(\sum_{s\in s^\phi}\frac{(d_{\pi_e}(s)\pi_e(a|s))^2}{d_{\pi_\mathcal{D}}(s)\pi_\mathcal{D}(a|s)}\right)
    - \left( \frac{(d_{\pi^\phi_e}(s^\phi)\pi^\phi_e(a|s^\phi))^2}{d_{\pi^\phi_\mathcal{D}}(s^\phi)\pi^\phi_\mathcal{D}(a|s^\phi)}\right)\\
    &\labelrel={theorem:our_variance_ratio_1}  \left(\left(\sum_{s\in s^\phi}\frac{(d_{\pi_e}(s)\pi_e(a|s))^2}{d_{\pi_\mathcal{D}}(s)\pi_\mathcal{D}(a|s)}\right)
    - \left( \frac{(d_{\pi^\phi_e}(s^\phi)\pi^\phi_e(a|s^\phi)^2}{d_{\pi^\phi_\mathcal{D}}(s^\phi)\pi^\phi_\mathcal{D}(a|s^\phi)}\right)\right)\\
\end{align*}
where \eqref{theorem:our_variance_ratio_1} is due to Definition \ref{def:norm_weights_agg_states}.

If we can show that $D' \geq 0$ for all possible sizes of $|s^\phi|$, we will
the have the original difference, $D$, is a sum of only non-negative terms, thus proving Theorem \ref{theorem:our_lower_var_ratios}. We will prove $D' \geq 0$ by inductive proof on the size of $|s^\phi|$
from $1$ to some $n \leq |\sset|$.

Let our statement to prove, $P(n)$ be that $D' \geq 0$ where $n = |s^\phi|$. This is
trivially true for $P(1)$ where the ground state equals the abstract state. Now consider
the inductive hypothesis, $P(n)$ is true for $n\geq1$. Now with the inductive step,
we must show that $P(n+1)$ is true given $P(n)$ is true. Starting with the 
inductive hypothesis:

\begin{align*}
    D'' &= \underbrace{\left(\sum_{s\in s^\phi}\frac{(d_{\pi_e}(s)\pi_e(a|s))^2}{d_{\pi_\mathcal{D}}(s)\pi_\mathcal{D}(a|s)}\right)}_{S}
    - \left( \frac{(\overbrace{(d_{\pi^\phi_e}(s^\phi)\pi^\phi_e(a|s^\phi)}^{C})^2}{\underbrace{d_{\pi^\phi_\mathcal{D}}(s^\phi)\pi^\phi_\mathcal{D}(a|s^\phi)}_{C'}}\right) \geq 0\\
\end{align*}
We define $S := \left(\sum_{s\in s^\phi}\frac{(d_{\pi_e}(s)\pi_e(a|s))^2}{d_{\pi_\mathcal{D}}(s)\pi_\mathcal{D}(a|s)}\right)$, $C := (d_{\pi^\phi_e}(s^\phi)\pi^\phi_e(a|s^\phi)$, and $C' :=d_{\pi^\phi_\mathcal{D}}(s^\phi)\pi^\phi_\mathcal{D}(a|s^\phi)$. After making the substitutions, we have:
\begin{align}
    C^2 \leq SC'
\end{align}
We have the above result holding true for when the $|s^\phi| = n$. Now consider 
the inductive step in relation to the inductive hypothesis where a new state, $s_{n+1}$ is added to the abstract state. We have the
following difference:

\begin{align*}
    D'' &= S + \frac{\left(d_{\pi_e}(s_{n+1})\pi_e(a|s_{n+1})\right)^2}{d_{\pi_\mathcal{D}}(s_{n+1})\pi_\mathcal{D}(a|s_{n+1})} - \frac{(C + d_{\pi_e}(s_{n+1})\pi_e(a|s_{n+1}))^2}{C' + d_{\pi_\mathcal{D}}(s_{n+1})\pi_\mathcal{D}(a|s_{n+1})}\\
\end{align*}
For ease in notation, let $x = d_{\pi_e}(s_{n+1})\pi_e(a|s_{n+1})$ and 
$y = d_{\pi_\mathcal{D}}(s_{n+1})\pi_\mathcal{D}(a|s_{n+1})$. The above difference is then:
\begin{align*}
    D'' &= S + \frac{x^2}{y} - \frac{(C + x)^2}{C' + y}\\
        &= \frac{Sy + x^2}{y} - \frac{(C + x)^2}{C' + y}\\
        &= \frac{1}{y(C' + y)}((Sy + x^2)(C' + y) - (C + x)^2y)\\
        &= \frac{1}{y(C' + y)} (SyC' + Sy^2 + x^2C' + x^2y - C^2y - x^2y
    -2Cxy)\\
    &= \frac{1}{y(C' + y)}(SyC' + Sy^2 + x^2C' - C^2y -2Cxy)\\
\end{align*}
The above difference, $D''$, is minimized most when $C$ is as large as possible. From the
inductive hypothesis, we have $C\leq \sqrt{SC'}$. The minimum difference can
be written as:
\begin{align*}
 D''&= \frac{1}{y(C' + y)}(Sy^2 + x^2C' -2\sqrt{SC'}xy)\\
    &= \frac{1}{y(C' + y)} (y\sqrt{S} - x\sqrt{C'})^2\\
    &\geq 0
\end{align*}

So we have $D'' \geq 0$ for $|s^\phi| = n + 1$, which means $D' \geq 0$ . We have showed that $P(n)$ is true for all $n$. We now have the original
difference, $D$, to be a sum of non-negative terms after performing this same
grouping for all abstract states and actions, which results in:
\begin{align*}
    \text{Var}\left( \frac{d_{\pi^\phi_e}(s^\phi)\pi^\phi_e(a|s^\phi)}{d_{\pi^\phi_\mathcal{D}}(s^\phi)\pi^\phi_\mathcal{D}(a|s^\phi)}\right) &\leq
    \text{Var}\left( \frac{d_{\pi_e}(s)\pi_e(a|s)}{d_{\pi_\mathcal{D}}(s)\pi_\mathcal{D}(a|s)}\right)
\end{align*}

Thus, we have:
\begin{align*}
    \text{Var}\left(\frac{d_{\pi^\phi_e}(s^\phi, a)}{d_{\pi^\phi_\mathcal{D}}(s^\phi, a)}\right)\leq \text{Var}\left(\frac{d_{\pi_e}(s, a)}{d_{\pi_\mathcal{D}}(s, a)}\right)
\end{align*}
\end{proof}

\lemRpiToAbsRpi*
\begin{proof}
Consider the definition of $R_\pi^\phi$:
\begin{align*}
    \rho(\pi^\phi) &= \sum_{s^\phi, a}d_{\pi^\phi}(s^\phi)\pi^\phi(a|s^\phi)r(s^\phi, a)\\
    &=\sum_{s^\phi, a} \left(\left(\sum_{s\in\phi^{-1}(s^\phi)} d_\pi(s)\right)
    \left(\sum_{s\in\phi^{-1}(s^\phi)} \pi(a|s)w_\pi(s)\right)
    r(s^\phi,a)\right)\\
    &\labelrel={lemma:g_val_abs_val_equality_1}\sum_{\phi(s), a} \left(\left(\sum_{s\in\phi^{-1}(s^\phi)} d_\pi(s)\right)
    \left(\sum_{s\in\phi^{-1}(s^\phi)} \frac{\pi(a|s)d_\pi(s)}{\sum_{s'\in\phi^{-1}(s^\phi)} d_\pi(s')}\right) r(s^\phi,a)\right)  \\
    &\labelrel={lemma:g_val_abs_val_equality_2} \sum_{\phi(s), a}
    \left(\sum_{s\in\phi^{-1}(s^\phi)} \pi(a|s)d_\pi(s)\right)r(s^\phi,a)\\
    &\labelrel={lemma:g_val_abs_val_equality_3}\sum_{\phi(s), a}
    \left(\sum_{s\in\phi^{-1}(s^\phi)} \pi(a|s)d_\pi(s)r(s,a)\right)\\
    &= \sum_{s, a}\pi(a|s)d_\pi(s)r(s,a)\\
    \rho(\pi^\phi) &= \rho(\pi)
\end{align*}
where \eqref{lemma:g_val_abs_val_equality_1} is due to Definition \ref{def:norm_weights_agg_states}, 
\eqref{lemma:g_val_abs_val_equality_2} is due to Definition \ref{def:norm_weights_agg_states} and Assumption \ref{assumption:reward_equality}
\end{proof}

\begin{restatable}{theorem}{theoremOursUnbiased} If 
Assumption \ref{assumption:reward_equality} holds, our estimator, $\hat{\rho}(\pi^\phi_e)$
as defined in Equation \ref{eq:our_estimator}, is an unbiased estimator of $\rho(\pi_e)$.
\label{theorem:our_unbiased}
\end{restatable}
\begin{proof}

We first consider the expectation of a single sample, $X = \frac{d_{\pi^\phi_e}(s^\phi)\pi^\phi_e(a|s^\phi)}{d_{\pi^\phi_\mathcal{D}}(s^\phi)\pi^\phi_\mathcal{D}(a|s^\phi)}r^\phi(s^\phi,a)$:

\begin{align*}
    \E_{s \sim d_{\pi_\mathcal{D}}, a \sim \pi_\mathcal{D}}\left[X\right] &= \sum_{s,a}d_{\pi_\mathcal{D}}(s)\pi_\mathcal{D}(a|s)\frac{d_{\pi^\phi_e}(s^\phi)\pi^\phi_e(a|s^\phi)}{d_{\pi^\phi_\mathcal{D}}(s^\phi)\pi^\phi_\mathcal{D}(a|s^\phi)}r^\phi(s^\phi,a)\\
    &\labelrel={theorem:our_unbiased_1} \sum_{s^\phi, a}\sum_{s \in \phi^{-1}(s^\phi)}d_{\pi_\mathcal{D}}(s)\pi_\mathcal{D}(a|s)\frac{d_{\pi^\phi_e}(s^\phi)\pi^\phi_e(a|s^\phi)}{d_{\pi^\phi_\mathcal{D}}(s^\phi)\pi^\phi_\mathcal{D}(a|s^\phi)}r^\phi(s^\phi,a) \\
    &\labelrel={theorem:our_unbiased_2}  \sum_{s^\phi, a}\frac{d_{\pi^\phi_e}(s^\phi)\pi^\phi_e(a|s^\phi)}{d_{\pi^\phi_\mathcal{D}}(s^\phi)\pi^\phi_\mathcal{D}(a|s^\phi)}\sum_{s \in \phi^{-1}(s^\phi)}d_{\pi_\mathcal{D}}(s)\pi_\mathcal{D}(a|s)r^\phi(s^\phi,a)\\
    &\labelrel={theorem:our_unbiased_3} \sum_{s^\phi, a}\frac{d_{\pi^\phi_e}(s^\phi)\pi^\phi_e(a|s^\phi)}{d_{\pi^\phi_\mathcal{D}}(s^\phi)\pi^\phi_\mathcal{D}(a|s^\phi)}r^\phi(s^\phi,a)\sum_{s \in \phi^{-1}(s^\phi)}d_{\pi_\mathcal{D}}(s)\pi_\mathcal{D}(a|s)\\
    &= \sum_{s^\phi, a}\frac{d_{\pi^\phi_e}(s^\phi)\pi^\phi_e(a|s^\phi)}{d_{\pi^\phi_\mathcal{D}}(s^\phi)\pi^\phi_\mathcal{D}(a|s^\phi)}r^\phi(s^\phi,a)\sum_{s'\in \phi^{-1}(s^\phi)} d_{\pi_\mathcal{D}}(s')\sum_{s \in \phi^{-1}(s^\phi)}\frac{d_{\pi_\mathcal{D}}(s)\pi_\mathcal{D}(a|s)}{\sum_{s'\in \phi^{-1}(s^\phi)} d_{\pi_\mathcal{D}}(s')}\\
    &= \sum_{s^\phi, a}\frac{d_{\pi^\phi_e}(s^\phi)\pi^\phi_e(a|s^\phi)}{d_{\pi^\phi_\mathcal{D}}(s^\phi)\pi^\phi_\mathcal{D}(a|s^\phi)}r^\phi(s^\phi,a)(d_{\pi^\phi_\mathcal{D}}(s^\phi)\pi^\phi_\mathcal{D}(a|
    s^\phi))\\
    &= \sum_{s^\phi, a}d_{\pi^\phi_e}(s^\phi)\pi^\phi_e(a|s^\phi)r^\phi(s^\phi,a)\\
    &\labelrel={theorem:our_unbiased_4}  \rho(\pi^\phi_e)\\
    \E_{s \sim d_{\pi_\mathcal{D}}, a \sim \pi_\mathcal{D}}\left[X\right] &\labelrel={theorem:our_unbiased_5} \rho(\pi_e)
\end{align*}
where \eqref{theorem:our_unbiased_3} is due to Definition \ref{def:norm_weights_agg_states} and Assumption \ref{assumption:reward_equality},
\eqref{theorem:our_unbiased_4} is due to Assumption \ref{assumption:coverage}, and 
\eqref{theorem:our_unbiased_5} is due to Proposition \ref{lemma:g_val_abs_val_equality}.

We have the bias defined as:
\begin{align*}
    \text{Bias}[\hat{\rho}(\pi^\phi_e)] &= \E_{s \sim d_{\pi_\mathcal{D}}, a \sim \pi_\mathcal{D}}[\hat{\rho}(\pi^\phi_e)] - R_{\pi_e}\\
    &= \E_{s \sim d_{\pi_\mathcal{D}}, a \sim \pi_\mathcal{D}}\left[\frac{1}{mT}\sum_{i=1}^{mT} \frac{d_{\pi^\phi_e}(s_i^\phi)\pi^\phi_e(a_i|s_i^\phi)}{d_{\pi^\phi_\mathcal{D}}(s_i^\phi)\pi^\phi_\mathcal{D}(a_i|s_i^\phi)}r^\phi(s_i^\phi,a_i)\right] - R_{\pi_e}\\
    &\labelrel={theorem:our_unbiased_6}  \frac{1}{mT}\sum_{i=1}^{mT}\E_{s_i \sim d_{\pi_\mathcal{D}}, a_i \sim \pi_\mathcal{D}}\left[ \frac{d_{\pi^\phi_e}(s_i^\phi)\pi^\phi_e(a_i|s_i^\phi)}{d_{\pi^\phi_\mathcal{D}}(s_i^\phi)\pi^\phi_\mathcal{D}(a_i|s_i^\phi)}r^\phi(s_i^\phi,a_i)\right] - R_{\pi_e}\\
    &\labelrel={theorem:our_unbiased_7} \left(\frac{1}{mT}\sum_{i=1}^{mT}R_{\pi_e}\right) - R_{\pi_e}\\
    &= \rho(\pi_e) - \rho(\pi_e)\\
    \text{Bias}[\hat{\rho}(\pi^\phi_e)]  &= 0
\end{align*}
where \eqref{theorem:our_unbiased_6} is due to linearity of expectation and \eqref{theorem:our_unbiased_7} is due to expectation of a single sample.
\end{proof}

\theoremOursMSEConsistent*
\begin{proof}
    We have the MSE of $\hat{\rho}(\pi^\phi_e)$ w.r.t $\rho(\pi_e)$
    defined in terms
    of the bias and variance as follows:

\begin{align*}
    \text{MSE}(\hat{\rho}(\pi^\phi_e)) = \E[(\hat{\rho}(\pi^\phi_e) - \rho(\pi_e))^2] &= \text{Var}[\hat{\rho}(\pi^\phi_e)] + (\text{Bias}[\hat{\rho}(\pi^\phi_e)])^2\\
    &\labelrel={theorem:our_mse_consistency_1}  \text{Var}[\hat{\rho}(\pi^\phi_e)]\\
    &= \frac{1}{(mT)^2} \text{Var}\left(\sum_{i=1}^{mT} \frac{d_{\pi^\phi_e}(s_i^\phi)\pi^\phi_e(a_i|s_i^\phi)}{d_{\pi^\phi_\mathcal{D}}(s_i^\phi)\pi^\phi_\mathcal{D}(a_i|s_i^\phi)}r^\phi(s_i^\phi,a_i)\right)
\end{align*}
where \eqref{theorem:our_mse_consistency_1} is because $\hat{\rho}$ is an unbiased estimator as shown in Theorem \ref{theorem:our_unbiased}.

Due to Assumptions \ref{assumption:coverage} and \ref{assumption:bounded_reward}, $\left(\sum_{i=1}^{mT} \frac{d_{\pi^\phi_e}(s_i^\phi)\pi^\phi_e(a_i|s_i^\phi)}{d_{\pi^\phi_\mathcal{D}}(s_i^\phi)\pi^\phi_\mathcal{D}(a_i|s_i^\phi)}r^\phi(s_i^\phi,a_i)\right)$ is a bounded value. Thus, as $mT\to\infty$, $\text{Var}[\hat{\rho}(\pi^\phi_e)]\to0$. We then have $\lim_{mT\to\infty}\E[(\hat{\rho}(\pi^\phi_e) - \rho(\pi_e))^2] = 0$. Thus,
the estimator $\hat{\rho}(\pi^\phi_e)$ is consistent in MSE.
\end{proof}

\begin{restatable}{corollary}{theoremOursStronglyConsistent}If Assumption
    \ref{assumption:reward_equality} holds, then
    our estimator, $\hat{\rho}(\pi^\phi_e)$
    as defined in Equation \ref{eq:our_estimator} is an asymptotically strongly consistent estimator 
    of $\rho(\pi_e)$.
    \label{theorem:our_consistency}
\end{restatable}
\begin{proof}
    Theorem \ref{theorem:our_mse_consistency} showed that $\hat{\rho}(\pi^\phi_e)$ is consistent in terms
    of MSE. Then by applying Lemma \ref{lemma:thomas_mse_consistency}, we have
    $\hat{\rho}(\pi^\phi_e)$ to be
    an asymptotically strongly consistent estimator
    of $\rho(\pi_e)$. That is, $\hat{\rho}(\pi^\phi_e) \overset{a.s.}{\to}\rho(\pi_e)$.
\end{proof}

Note that this proof is very similar to the one given in Theorem \ref{theorem:our_lower_var_ratios}, where the main the difference is
that we are analyzing the variance of an OPE estimator on given
batch of data.
\theoremOursVariance*
\begin{proof}

We first consider the general form of the variance of
the baseline estimator, $\hat{\rho}(\pi_e)$ (and the 
similar form applies to 
$\hat{\rho}(\pi^\phi_e)$ ). For simplicity in notation, we take
the batch size, $m = 1$, but the analysis holds for
general $m$.

\begin{align*}
    \text{Var}[\hat{\rho}(\pi_e)] &= \text{Var}\left(\frac{1}{T}\sum_{t=1}^T \frac{d_{\pi_e}(s_t)\pi_e(a_t|s_t)}{d_{\pi_\mathcal{D}}(s_t)\pi_\mathcal{D}(a_t|s_t)}r(s_t,a_t)\right)\\
    &= \frac{1}{T^2}\left(\sum_{t=1}^T\underbrace{\text{Var}\left( \frac{d_{\pi_e}(s_t)\pi_e(a_t|s_t)}{d_{\pi_\mathcal{D}}(s_t)\pi_\mathcal{D}(a_t|s_t)}r(s_t,a_t)\right)}_{V^g_t} + 2\sum_{t<k}\text{Cov}\left(\frac{d_{\pi_e}(s_t,a_t)}{d_{\pi_\mathcal{D}}(s_t,a_t)}r(s_t, a_t), \frac{d_{\pi_e}(s_k,a_k)}{d_{\pi_\mathcal{D}}(s_k,a_k)}r(s_k, a_k)\right)\right)
\end{align*}

Here we have $V_t^g$ to be the variance of a single sample at time
$t$. Similarly, for our estimator, we have $V_t^\phi = \text{Var}\left( \frac{d_{\pi^\phi_e}(s^\phi)\pi^\phi_e(a|s^\phi)}{d_{\pi^\phi_\mathcal{D}}(s^\phi)\pi^\phi_\mathcal{D}(a|s^\phi)}r^\phi(s^\phi,a)\right)$. We will show that $V_t^\phi\leq V_t^g$. We drop the subscript $t$
for convenience since the analysis applies for all $t$.

Before comparing the variances, we note that due to Assumption
\ref{assumption:reward_equality} and \ref{assumption:coverage}, Lemma
\ref{lemma:exp_f_g_abs_equal}, and Proposition \ref{lemma:g_val_abs_val_equality}:
\begin{align*}
    \E_{s\sim d_{\pi_\mathcal{D}}, a\sim \pi_\mathcal{D}}\left[\frac{d_{\pi_e}(s)\pi_e(a|s)}{d_{\pi_\mathcal{D}}(s)\pi_\mathcal{D}(a|s)}r(s,a)\right] = \E_{s\sim d_{\pi_\mathcal{D}}, a\sim \pi_\mathcal{D}}\left[\frac{d_{\pi^\phi_e}(\phi(s))\pi^\phi_e(a|\phi(s))}{d_{\pi^\phi_\mathcal{D}}(\phi(s))\pi^\phi_\mathcal{D}(a|\phi(s))}r^\phi(\phi(s),a)\right] = \rho(\pi_e)
\end{align*}

Consider the difference between the two variances. 
\begin{align*}
    D &= V^g - V^\phi\\
    &= \text{Var}\left( \frac{d_{\pi_e}(s)\pi_e(a|s)}{d_{\pi_\mathcal{D}}(s)\pi_\mathcal{D}(a|s)}r(s,a)\right) - \text{Var}\left( \frac{d_{\pi^\phi_e}(\phi(s))\pi^\phi_e(a|\phi(s))}{d_{\pi^\phi_\mathcal{D}}(\phi(s))\pi^\phi_\mathcal{D}(a|\phi(s))}r^\phi(\phi(s),a)\right)\\
    &= \E_{s\sim d_{\pi_\mathcal{D}}, a\sim \pi_\mathcal{D}}\left[\left(\frac{d_{\pi_e}(s)\pi_e(a|s)}{d_{\pi_\mathcal{D}}(s)\pi_\mathcal{D}(a|s)}r(s,a)\right)^2\right]
    - \E_{s \sim d_{\pi_\mathcal{D}}, a\sim \pi_\mathcal{D}}\left[\left(\frac{d_{\pi^\phi_e}(\phi(s))\pi^\phi_e(a|\phi(s))}{d_{\pi^\phi_\mathcal{D}}(\phi(s))\pi^\phi_\mathcal{D}(a|\phi(s))}r^\phi(\phi(s),a)\right)^2\right]\\
    &\labelrel={theorem:our_variance_1} \E_{s\sim d_{\pi_\mathcal{D}}, a\sim \pi_\mathcal{D}}\left[\left(\frac{d_{\pi_e}(s)\pi_e(a|s)}{d_{\pi_\mathcal{D}}(s)\pi_\mathcal{D}(a|s)}r(s,a)\right)^2\right]
    - \E_{s^\phi \sim d_{\pi^\phi_\mathcal{D}}, a\sim \pi_\mathcal{D}^\phi}\left[\left(\frac{d_{\pi^\phi_e}(s^\phi)\pi^\phi_e(a|s^\phi)}{d_{\pi^\phi_\mathcal{D}}(s^\phi)\pi^\phi_\mathcal{D}(a|s^\phi)}r^\phi(s^\phi,a)\right)^2\right]\\
    &= \sum_{s, a}d_{\pi_\mathcal{D}}(s)\pi_\mathcal{D}(a|s)\left(\frac{d_{\pi_e}(s)\pi_e(a|s)}{d_{\pi_\mathcal{D}}(s)\pi_\mathcal{D}(a|s)}r(s,a)\right)^2
    - \sum_{s^\phi, a}d_{\pi^\phi_\mathcal{D}}(s^\phi)\pi^\phi_\mathcal{D}(a|s^\phi) \left(\frac{d_{\pi^\phi_e}(s^\phi)\pi^\phi_e(a|s^\phi)}{d_{\pi^\phi_\mathcal{D}}(s^\phi)\pi^\phi_\mathcal{D}(a|s^\phi)}r^\phi(s^\phi,a)\right)^2
\end{align*}
where \eqref{theorem:our_variance_1} is due to  Lemma \ref{lemma:exp_f_g_abs_equal}.

We can analyze this difference by 
looking at one abstract state and one action and all the states that belong to it. That
is, for a fixed abstract state, $s^\phi$, and fixed action, $a$, we have:

\begin{align*}
    D' &= \sum_{s\in s^\phi}d_{\pi_\mathcal{D}}(s)\pi_\mathcal{D}(a|s)\left(\frac{d_{\pi_e}(s)\pi_e(a|s)}{d_{\pi_\mathcal{D}}(s)\pi_\mathcal{D}(a|s)}r(s,a)\right)^2
    - \left(d_{\pi^\phi_\mathcal{D}}(s^\phi)\pi^\phi_\mathcal{D}(a|s^\phi) \left(\frac{d_{\pi^\phi_e}(s^\phi)\pi^\phi_e(a|s^\phi)}{d_{\pi^\phi_\mathcal{D}}(s^\phi)\pi^\phi_\mathcal{D}(a|s^\phi)}r^\phi(s^\phi,a)\right)^2\right)\\
    &= \left(\sum_{s\in s^\phi}\frac{(d_{\pi_e}(s)\pi_e(a|s)r(s,a))^2}{d_{\pi_\mathcal{D}}(s)\pi_\mathcal{D}(a|s)}\right)
    - \left( \frac{(d_{\pi^\phi_e}(s^\phi)\pi^\phi_e(a|s^\phi)r^\phi(s^\phi,a))^2}{d_{\pi^\phi_\mathcal{D}}(s^\phi)\pi^\phi_\mathcal{D}(a|s^\phi)}\right)\\
    &\labelrel={theorem:our_variance_2}  r(s,a)^2\left(\left(\sum_{s\in s^\phi}\frac{(d_{\pi_e}(s)\pi_e(a|s))^2}{d_{\pi_\mathcal{D}}(s)\pi_\mathcal{D}(a|s)}\right)
    - \left( \frac{(d_{\pi^\phi_e}(s^\phi)\pi^\phi_e(a|s^\phi)^2}{d_{\pi^\phi_\mathcal{D}}(s^\phi)\pi^\phi_\mathcal{D}(a|s^\phi)}\right)\right)\\
\end{align*}
where \eqref{theorem:our_variance_2} is due to Definition \ref{def:norm_weights_agg_states} and Assumption \ref{assumption:reward_equality}.

If we can show that $D' \geq 0$ for all possible sizes of $|s^\phi|$, we will
the have the original difference, $D$, a sum of only non-negative terms, thus proving Theorem \ref{theorem:our_variance}. We will prove $D' \geq 0$ by proof by induction on the size of $|s^\phi|$
from $1$ to some $n \leq |\sset|$. First note that $r(s,a)^2 > 0$, so we
can ignore this term.

Let our statement to prove, $P(n)$ be that $D' \geq 0$ where $n = |s^\phi|$. This is
trivially true for $P(1)$ where the ground state equals the abstract state. Now consider
the inductive hypothesis, $P(n)$ is true for $n\geq1$. Now with the inductive step,
we must show that $P(n+1)$ is true given $P(n)$ is true. Starting with the 
inductive hypothesis:

\begin{align*}
    D'' &= \underbrace{\left(\sum_{s\in s^\phi}\frac{(d_{\pi_e}(s)\pi_e(a|s))^2}{d_{\pi_\mathcal{D}}(s)\pi_\mathcal{D}(a|s)}\right)}_{S}
    - \left( \frac{(\overbrace{(d_{\pi^\phi_e}(s^\phi)\pi^\phi_e(a|s^\phi)}^{C})^2}{\underbrace{d_{\pi^\phi_\mathcal{D}}(s^\phi)\pi^\phi_\mathcal{D}(a|s^\phi)}_{C'}}\right) \geq 0\\
\end{align*}

After making the substitutions, we have:
\begin{align}
    C^2 \leq SC'
\end{align}
We have the above result holding true for when the $|s^\phi| = n$. Now consider 
the inductive step in relation to the inductive hypothesis where a new state, $s_{n+1}$ is added to the abstract state. We have the
following difference:

\begin{align*}
    D'' &= S + \frac{\left(d_{\pi_e}(s_{n+1})\pi_e(a|s_{n+1})\right)^2}{d_{\pi_\mathcal{D}}(s_{n+1})\pi_\mathcal{D}(a|s_{n+1})} - \frac{(C + d_{\pi_e}(s_{n+1})\pi_e(a|s_{n+1}))^2}{C' + d_{\pi_\mathcal{D}}(s_{n+1})\pi_\mathcal{D}(a|s_{n+1})}\\
\end{align*}
For ease in notation, let $x = d_{\pi_e}(s_{n+1})\pi_e(a|s_{n+1})$ and 
$y = d_{\pi_\mathcal{D}}(s_{n+1})\pi_\mathcal{D}(a|s_{n+1})$. The above difference is then:
\begin{align*}
    D'' &= S + \frac{x^2}{y} - \frac{(C + x)^2}{C' + y}\\
        &= \frac{Sy + x^2}{y} - \frac{(C + x)^2}{C' + y}\\
        &= \frac{1}{y(C' + y)}((Sy + x^2)(C' + y) - (C + x)^2y)\\
        &= \frac{1}{y(C' + y)} (SyC' + Sy^2 + x^2C' + x^2y - C^2y - x^2y
    -2Cxy)\\
    &= \frac{1}{y(C' + y)}(SyC' + Sy^2 + x^2C' - C^2y -2Cxy)\\
\end{align*}
The above difference, $D''$, is minimized most when $C$ is as large as possible. From the
inductive hypothesis, we have $C\leq \sqrt{SC'}$. The minimum difference can
be written as:
\begin{align*}
 D''&= \frac{1}{y(C' + y)}(Sy^2 + x^2C' -2\sqrt{SC'}xy)\\
    &= \frac{1}{y(C' + y)} (y\sqrt{S} - x\sqrt{C'})^2\\
    &\geq 0
\end{align*}

So we have $D'' \geq 0$ for $|s^\phi| = n + 1$, which means $D' \geq 0$ . We have showed that $P(n)$ is true for all $n$. We now have the original
difference, $D$, to be a sum of non-negative terms after performing this same
grouping for all abstract states and actions, which results in:
\begin{align*}
    V^\phi &\leq V^g\\
    \text{Var}\left( \frac{d_{\pi^\phi_e}(s^\phi)\pi^\phi_e(a|s^\phi)}{d_{\pi^\phi_\mathcal{D}}(s^\phi)\pi^\phi_\mathcal{D}(a|s^\phi)}r^\phi(s^\phi,a)\right) &\leq
    \text{Var}\left( \frac{d_{\pi_e}(s)\pi_e(a|s)}{d_{\pi_\mathcal{D}}(s)\pi_\mathcal{D}(a|s)}r(s,a)\right)
\end{align*}

Thus, when the covariance terms interact favorably, that is, if for any fixed $1\leq t < k \leq T$,\\ $\text{Cov}\left(\frac{d_{\pi^\phi_e}(s_t^\phi,a_t)}{d_{\pi^\phi_\mathcal{D}}(s_t^\phi,a_t)}r^\phi(s_t^\phi, a_t), \frac{d_{\pi^\phi_e}(s_k^\phi,a_k)}{d_{\pi^\phi_\mathcal{D}}(s_k^\phi,a_k)}r^\phi(s_k^\phi, a_k)\right) \leq
    \text{Cov}\left(\frac{d_{\pi_e}(s_t,a_t)}{d_{\pi_\mathcal{D}}(s_t,a_t)}r(s_t, a_t), \frac{d_{\pi_e}(s_k,a_k)}{d_{\pi_\mathcal{D}}(s_k,a_k)}r(s_k, a_k)\right)$, we have:
\begin{align*}
    \text{Var}(\hat{\rho}(\pi^\phi_e))\leq \text{Var}(\hat{\rho}(\pi_e))
\end{align*}
\end{proof}

\subsection{Additional AbstractBestDICE Derivation Details}
\label{app:abs_dice_der}
    
    We proceed assuming $\phi$ satisfies Assumptions \ref{assumption:reward_equality}, \ref{assumption:trans_bisim},
    and \ref{assumption:pi_act_equality} from the main text. We base
    the following derivation on that of DualDICE \cite{nachum2019dualdice}. The only difference between DualDICE and BestDICE (which we use in our experiments) is that the latter enforces a positivity and unit mean constraint on the ratios.
    
    \textbf{Technical observation} Consider the same technical observation made in DualDICE: the solution to the scalar convex optimization problem
    $\min_x J(x)  \coloneqq \frac{1}{2}mx^2 - nx$,
    where $m\in\mathbb{R}_{>0}$ and $n\in\mathbb{R}_{\geq 0}$, is $x^* = \frac{n}{m}$. This observation can then be connected 
    to a similar convex
    problem but in terms of functions where the ratio $\frac{n}{m}$ 
    corresponds to the true abstract ratios, $d_{\pi^\phi_e}(s^\phi, a)/d_{\mathcal{D}^\phi}(s^\phi, a)$. Consider the following
    convex problem where $x \in \sset^\phi\times\aset\to[0, C] \subset \mathbb{R}$:
    \begin{equation*}
        \begin{split}
        \min_{x:\sset^\phi\times\aset\to\mathcal{C}} J(x) &:= \frac{1}{2}\E_{(s^\phi, a) \sim d_{\mathcal{D}^\phi}}[x(s^\phi, a)^2]
        - \E_{(s^\phi, a) \sim d_{\pi^\phi_e}}[x(s^\phi, a)]
        \end{split}
    \end{equation*}
    The solution to this optimization problem is the true abstract state-action ratios:
    $x^*(s^\phi, a) = d_{\pi^\phi_e}(s^\phi, a)/d_{\mathcal{D}^\phi}(s^\phi, a)$.
    
     \textbf{Change-of-variables} As also done by \citet{nachum2019dualdice}, we next consider a change-of-variables trick to obtain an objective that can be approximated with samples from $\mathcal{D}^\phi$. 
     However, as noted in Section \ref{sec:abs_mis}, there are two challenges we must overcome before applying this procedure on $\mathcal{D}^\phi$. To overcome these challenges, we identified Assumptions \ref{assumption:trans_bisim} and \ref{assumption:pi_act_equality} as conditions that $\phi$ must satisfy. With a satisfactory $\phi$, we can proceed as follows. Consider an arbitrary
    function $\nu \in \sset^\phi\times\aset\to\mathbb{R}$ that
    satisfies $\nu(s^\phi, a) = x(s^\phi, a) + \mathcal{B}^{\pi^\phi_e}\nu(s^\phi, a)$
    where  $\mathcal{B}^{\pi^\phi_e}\nu(s^\phi, a) = \gamma\E_{s'^\phi\sim {P^\phi}(s^\phi, a), a'\sim\pi^\phi_e(s'^\phi)}[\nu(s'^\phi, a')]$. Since $x(s^\phi, a) \in [0, C]$ and $\gamma < 1$, $\nu$ is well-defined. We can
    the replace  $x(s^\phi,a)$
    with $(\nu -\mathcal{B}^{\pi^\phi_e}\nu)(s^\phi,a)$. For the second
    expectation, we define $\beta^\phi_t(s^\phi) := \PR(s^\phi=s^\phi_t | s^\phi_0\sim\beta^\phi, a_k\sim\pi_e^\phi(s^\phi_k), s^\phi_{k+1}\sim P^\phi(s^\phi_k, a_k)), 0\leq k < t$
    as the abstract-state visitation probability at time-step $t$ by the 
    $\pi^\phi_e$ in the abstract MDP, and then have:
    \begin{align*}
        \E_{(s^\phi, a) \sim d_{\pi^\phi_e}}[x(s^\phi, a)] &= (1 - \gamma) \sum_{t=0}^\infty\gamma^t\E_{s^\phi\sim\beta^\phi_t, a\sim\pi^\phi_e}[
        x(s^\phi, a)]\\
        &=(1 - \gamma) \sum_{t=0}^\infty\gamma^t\E_{s^\phi\sim\beta_t, a\sim\pi^\phi_e}[
        \nu(s^\phi, a)
        - \gamma\E_{s'^{\phi}\sim P^\phi(s^\phi,a), 
        a'\sim\pi^\phi(s^{\phi'})}[\nu(s'{^\phi}, a')]\\
        &=(1 - \gamma) \sum_{t=0}^\infty\gamma^t\E_{s^\phi\sim\beta^\phi_t, a\sim\pi^\phi_e}[
        \nu(s^\phi, a)] 
        - (1 - \gamma) \sum_{t=0}^\infty\gamma^{t+1}\E_{s^\phi\sim\beta^\phi_{t+1}, a\sim\pi^\phi_e}[\nu(s{^\phi}, a)]\\
        \E_{(s^\phi, a) \sim d_{\pi^\phi_e}}[x(s^\phi, a)] &= (1 - \gamma) \E_{s^\phi\sim\beta^\phi_0, a\sim\pi^\phi_e}[
        \nu(s^\phi, a)]
    \end{align*}
    where $\beta^\phi_0 = d_{0^\phi}$ is the initial abstract state distribution. After the two changes, we then have:
    \begin{equation}
        \begin{split}
        \min_{\nu:\sset^\phi\times\aset\to \mathbb{R}} J(\nu) &:= \frac{1}{2}\E_{\mathcal{D}^\phi}[(\nu - \mathcal{B}^{\pi^\phi_e}\nu)(s^\phi,a)^2] 
        - (1 - \gamma) \E_{s^\phi_0 \sim d_{0^\phi}, a_0\sim \pi^\phi_e}[\nu(s^\phi_0, a_0)]
        \end{split}
        \label{eq:pre_main_opt}
    \end{equation}
    where we have the optimal solution $x^*(s^\phi,a) = (\nu^* - \mathcal{B}^{\pi^\phi_e}\nu^*)(s^\phi,a) = d_{\pi^\phi_e}(s^\phi, a)/d_{\mathcal{D}^\phi}(s^\phi, a)$. 
    
    \subsubsection{Optimization techniques} The optimization problem in Equation (\ref{eq:pre_main_opt})
    presents a couple of optimization challenges. These challenges
    are the same as in DualDICE \cite{nachum2019dualdice} (page 5). Namely,
    \begin{itemize}
        \item The quantity $(\nu - \nu\mathcal{B}^{\pi^\phi_e})(s^\phi,a)$
        involves a conditional exepctation inside the square. When
        the environment dynamics are stochastic and/or the action
        space is large or continuous, this quantity may not be
        readily optimized using stochastic techniques
        \item Even once $\nu^*$ is determined, $(\nu^* - \nu^*\mathcal{B}^{\pi^\phi_e})(s^\phi,a)$ may not be easily
        computable due to the same reasons as above.
    \end{itemize}
    
    To overcome these challenges, Fenchal duality is invoked
    \cite{rockafellar1970opt} where for a convex function
    $f$, $f(x) = \max_{\zeta}x\cdot\zeta - f^*(\zeta)$,
    where $f^*$ is the Fenchal conjugate of $f$. When $f(x) = \frac{1}{2}x^2$, $f^*(\zeta) = \frac{1}{2}\zeta^2$. Thus,
    the optimization given in Equation \ref{eq:pre_main_opt}
    becomes:
    
    \begin{equation}
        \min_{\nu:\sset^\phi\times\aset\to \mathbb{R}} J(\nu) := \frac{1}{2}\E_{\mathcal{D}^\phi}\left[\max_{\zeta}\left((\nu - \mathcal{B}^{\pi^\phi_e}\nu)(s^\phi,a)\zeta - \frac{1}{2}\zeta^2\right)\right] 
        - (1 - \gamma) \E_{s^\phi_0 \sim d_{0^\phi}, a_0\sim \pi^\phi_e}[\nu(s^\phi_0, a_0)]
        \label{eq:main_opt_app}
    \end{equation}
    
    Then by the interchangability principle \cite{tyrrell1998interchange, dai2016interchange}, we replace
    the inner maximization over scalars to a maximization
    over $\zeta:\sset^\phi\times\aset\to\mathbb{R}$,
    resulting in the optimization given by Equation
    (\ref{eq:main_opt_app}):

    \begin{equation*}
        \begin{split}
        \min_{\nu:\sset^\phi\times\aset\to \mathbb{R}} \max_{\zeta:\sset^\phi\times\aset\to \mathbb{R}} J(\nu, \zeta) &:=\E_{\mathcal{D}^\phi}[\zeta(s^\phi,a) (\nu(s^\phi,a)
        - \gamma\E_{ a'\sim\pi^\phi_e}\nu(s'^\phi, a')) - \frac{1}{2} \zeta^2(s^\phi,a)] \\
        &- (1 - \gamma) \E_{s^\phi_0 \sim d_{0^\phi}, a_0\sim \pi^\phi_e}[\nu(s^\phi_0, a_0)]
        \end{split}
    \end{equation*}
    On applying the KKT conditions to the
    inner optimization, which is convex and quadratic, for any $\nu$
     the optimal $\zeta(s^\phi,a) = (\nu - \mathcal{B}^{\pi^\phi_e}\nu)(s^\phi,a)$. Therefore, for the optimal $\nu^*$,
    we have $\zeta^*(s^\phi,a) = (\nu^* - \mathcal{B}^{\pi^\phi_e}\nu^*)(s^\phi,a) = d_{\pi^\phi_e}(s^\phi, a)/d_{\mathcal{D}^\phi}(s^\phi, a)$. This derivation with
    DualDICE naturally extends to BestDICE since BestDICE
    solves the same optimization with the added constraints
    that $\E[\zeta] = 1$ and $\zeta\geq0$, which are
    properties we know the true ratios would satisfy.
\subsection{Additional Tabular Experiments
and Details}
\label{app:tabular_exp}
In this section, we include the remaining tabular experiments and information.
\begin{itemize}
    \item The true value of $\rho(\pi_e)$ was determined by executing  $\pi_e$ for $200$ episodes and averaging the results.
    \item In all experiments, we use $\gamma = 0.999$
    \item In the $\hat{d}_{\pi^\phi_e}$ vs. $d_{\pi^\phi_e}$ plots, we
    use batch size of $300$. The trajectory length was $100$ time-steps.
\end{itemize}
\begin{figure}[H]
    \centering
        \subfigure[Assumption \ref{assumption:trans_bisim} (transition dynamics) violated ]{\includegraphics[scale=0.4]{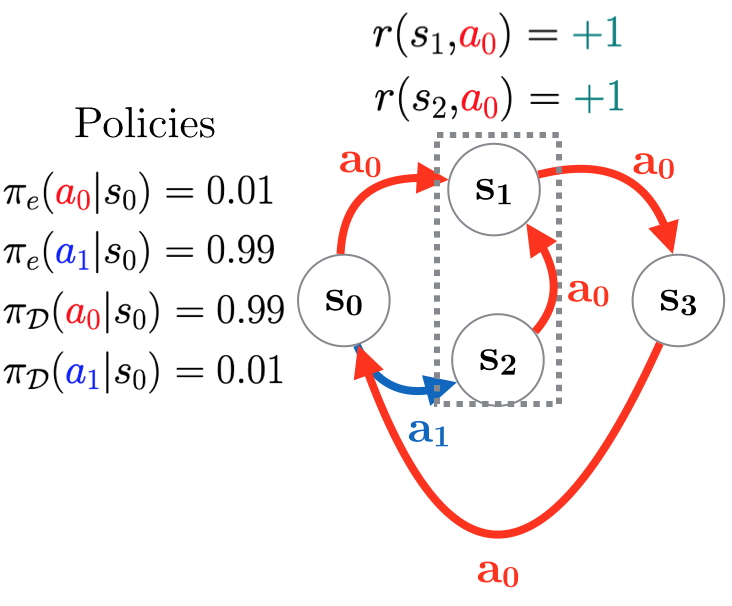}}
        \subfigure[Assumption \ref{assumption:pi_act_equality} ($\pi_e$
        action equality) violated ]{\includegraphics[scale=0.4]{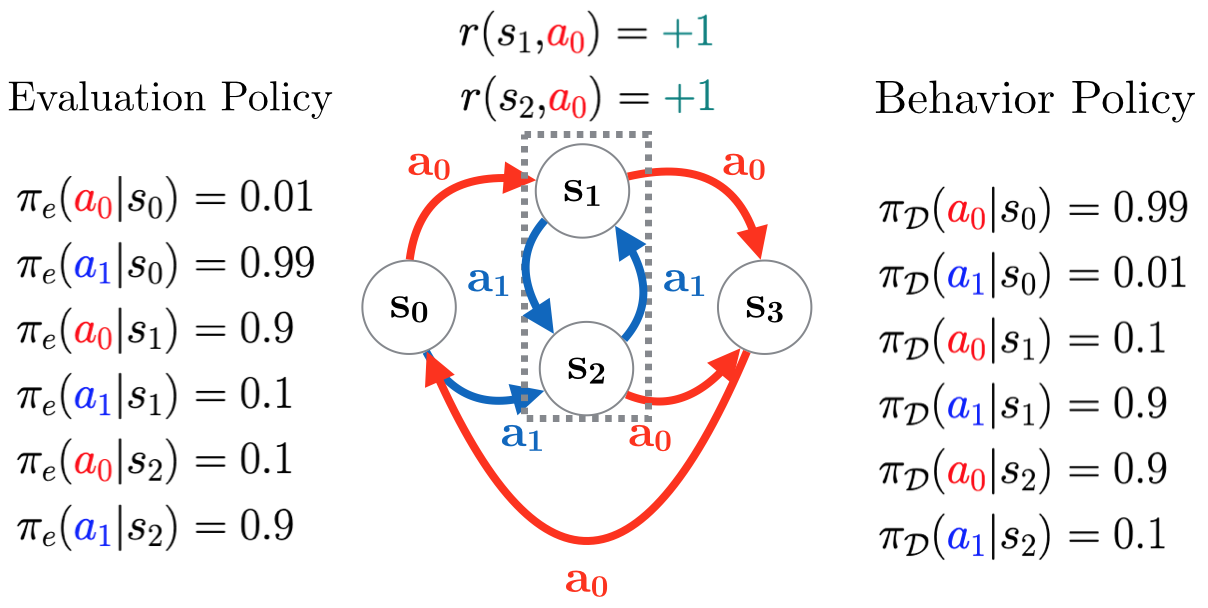}}\\
        \subfigure[Assumptions \ref{assumption:trans_bisim} and \ref{assumption:pi_act_equality} violated. ($a_0$ and $a_1$ in $s_2$ switched) ]{\includegraphics[scale=0.4]{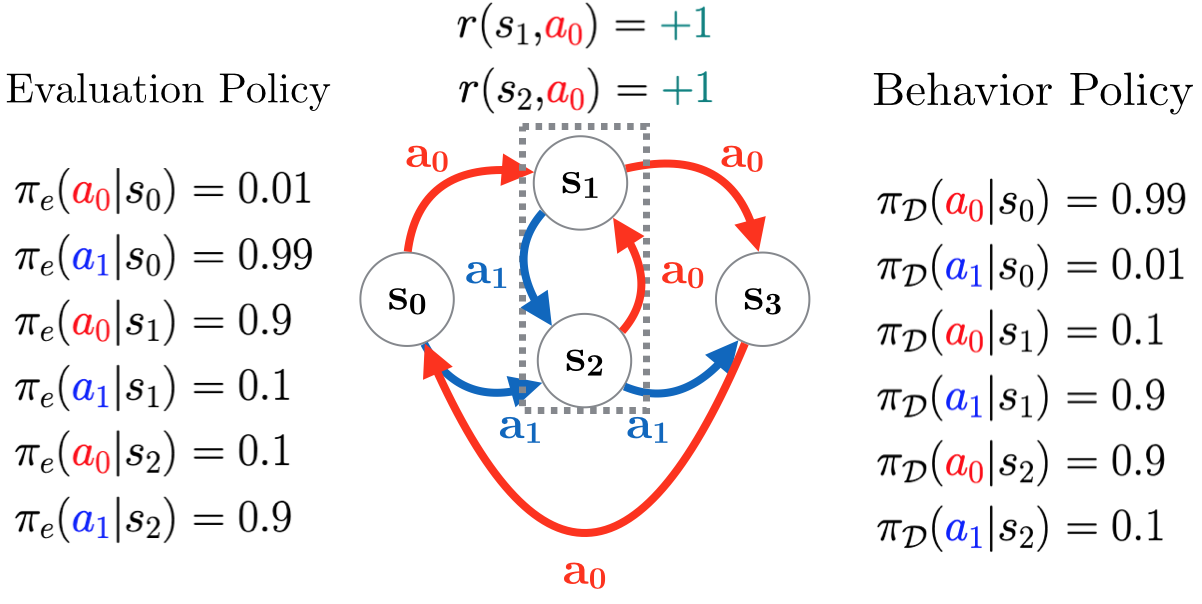}}
    \caption{\footnotesize Variations to the TwoPath MDP from Figure \ref{fig:hard_eg}
    to violate assumptions.}
    \label{fig:mdp_variations_violation}
\end{figure}

In practice, 
Assumptions \ref{assumption:trans_bisim} and \ref{assumption:pi_act_equality} may be hard to validate.
To study the impact of violation of these assumptions, we consider the following variations (Figure \ref{fig:mdp_variations_violation}) to
the original TwoPath MDP presented in Figure \ref{fig:hard_eg}. In all these MDPS, we only
aggregate $s_1$ and $s_2$. The policies in the last two MDPs are the same. Figure \ref{fig:true_ratio_exp2} illustrates the accuracy of AbstractBestDICE's
true ratio estimation. In general, we do see AbstractBestDICE fails to compute the true
abstract ratios when these assumptions are violated.

\begin{figure}[H]
    \centering
        \subfigure[Assumption \ref{assumption:trans_bisim} (transition dynamics) violated]{\label{fig:true_ratio_est_violate_assum2}\includegraphics[scale=0.22]{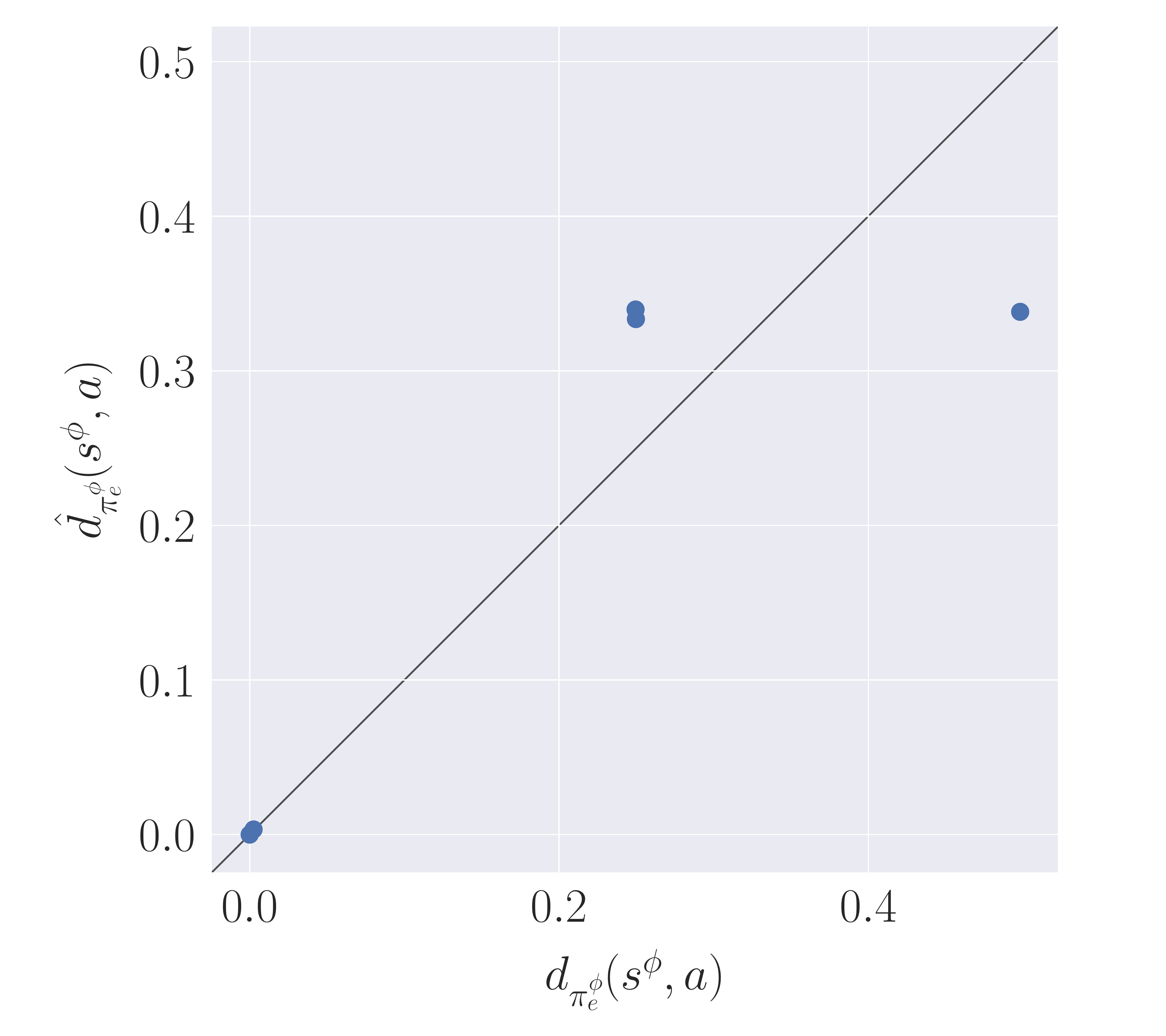}}
        \subfigure[Assumption \ref{assumption:pi_act_equality} ($\pi_e$
        action equality) violated ]{\label{fig:true_ratio_est_violate_assum3}\includegraphics[scale=0.22]{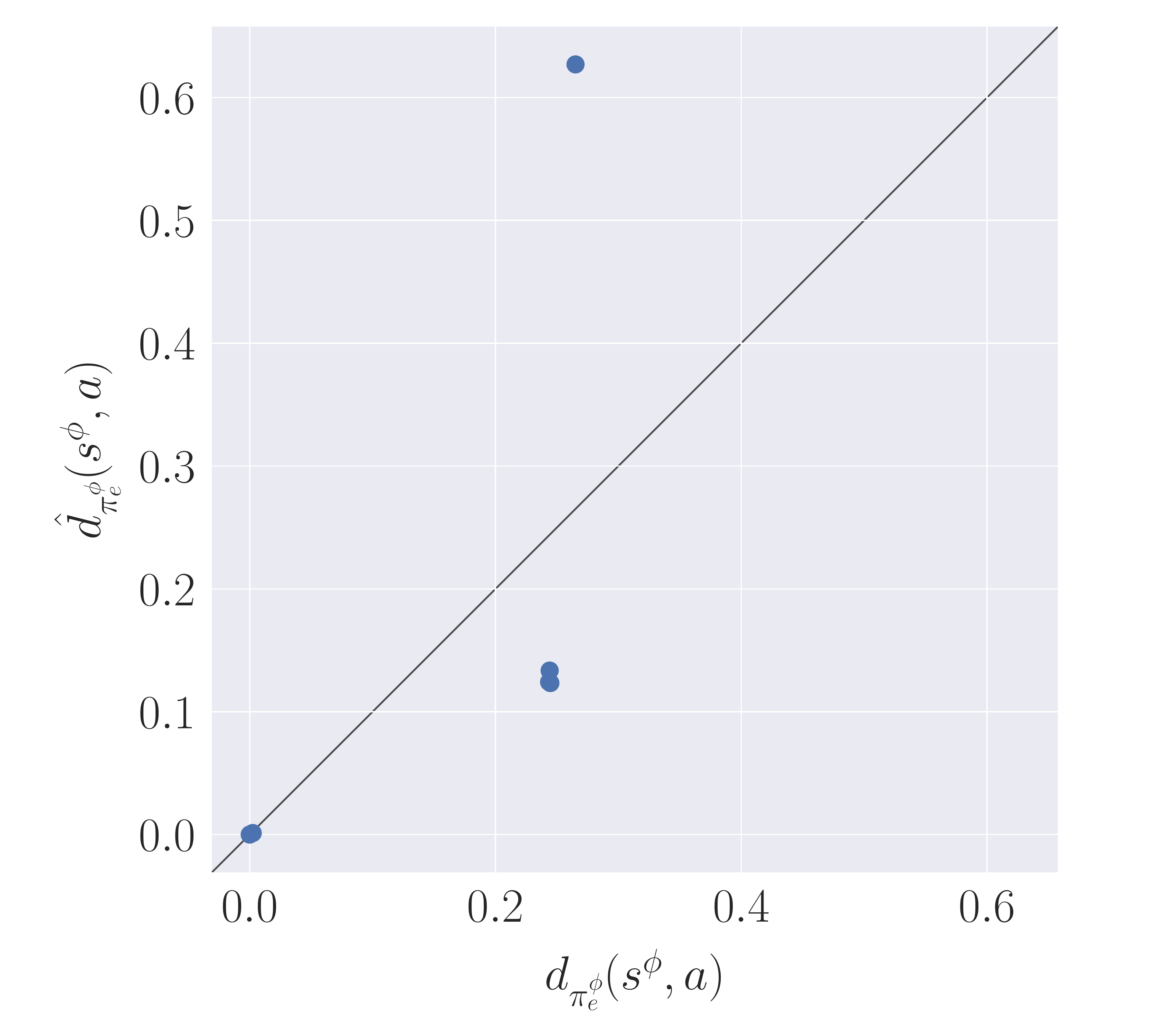}}
        \subfigure[Assumptions \ref{assumption:trans_bisim} and \ref{assumption:pi_act_equality} violated]{\label{fig:true_ratio_est_violate_assum2_assum3}\includegraphics[scale=0.22]{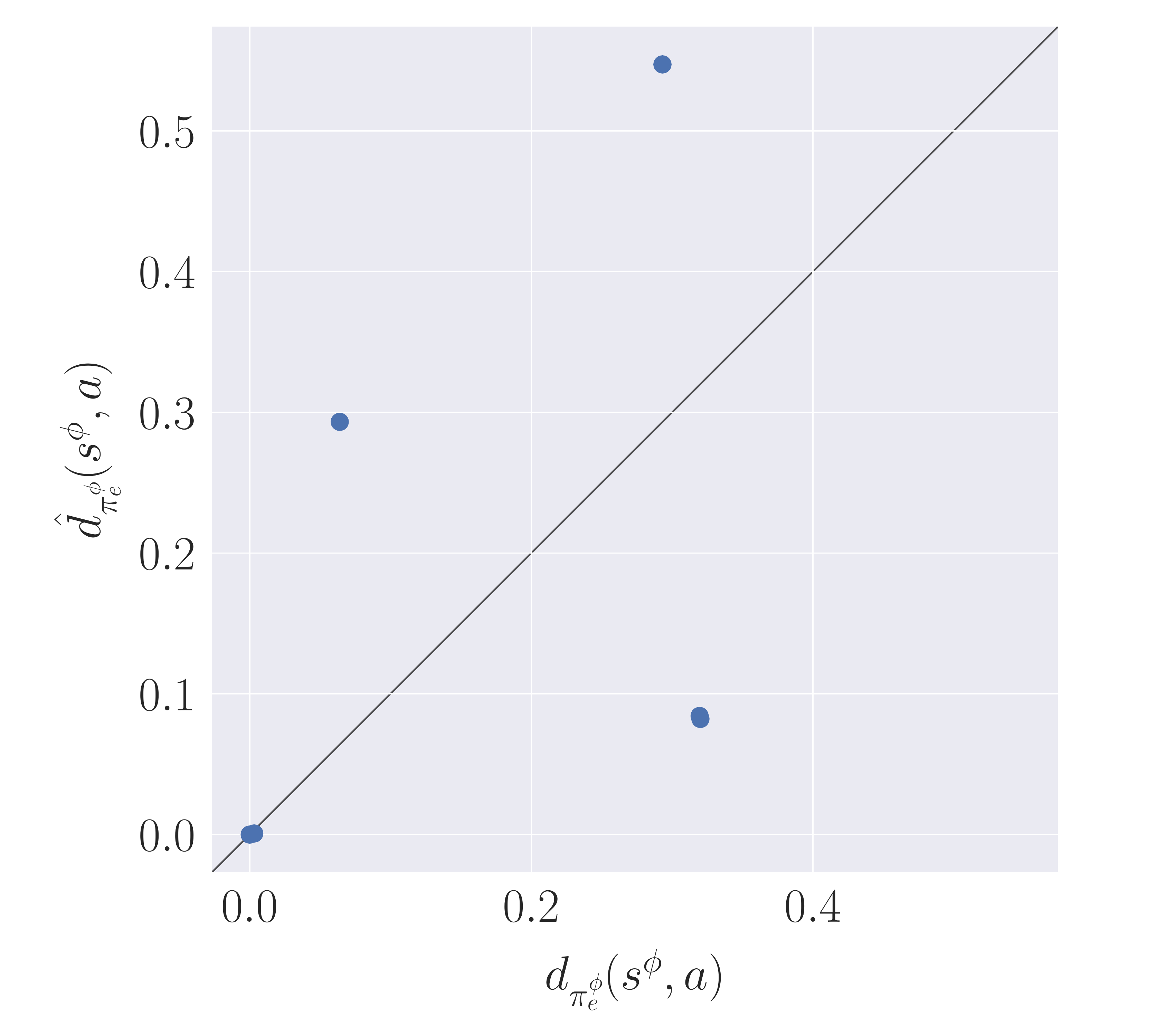}}
    \caption{\footnotesize True ratio estimations with assumption violations. Closer to the linear
    line is better.}
    \label{fig:true_ratio_exp2}
\end{figure}

\subsection{Additional Function Approximation
Experiments and Details}
\label{app:func_approx_exp}

\subsubsection{Oracle $\rho(\pi_e)$ Values} On each domain, we executed $\pi_e$
for $200$ episodes and averaged the results.

\subsubsection{Policies} For each of the domains,
we used the following policies:
\begin{itemize}
    \item Reacher: We trained a policy using
    PPO \cite{schulman2017ppo}. $\pi_e$
    was the trained policy after $100$k time-steps with
    a standard deviation of $0.1$ on the action
    dimensions while $\pi_b$ used $0.5$ as
    the standard deviation.
    \item Walker2D: We trained a policy using
    PPO \cite{schulman2017ppo}. $\pi_e$
    was the trained policy after $100$k time-steps with
    a standard deviation of $0.1$ on the action
    dimensions while $\pi_b$ used $0.5$ as
    the standard deviation.
    \item Pusher: We trained a policy using
    PPO \cite{schulman2017ppo}. $\pi_e$
    was the trained policy after $100$k time-steps with
    a standard deviation of $0.1$ on the action
    dimensions while $\pi_b$ used $0.5$ as
    the standard deviation.
    \item AntUMaze: We used the policies
    made available \cite{fu2021opebench}.
    $\pi_e$ was the final 10th snapshot
    saved and $\pi_b$ was the 5th snapshot.
    Each also had $0.1$ standard deviation
    on the action dimensions.
\end{itemize}

\subsubsection{Trajectory Length} For each of the domains, the trajectory length is: $200$ for Reacher, $500$ for Walker2D, $300$ for Pusher, and $500$ for AntUMaze.

\subsubsection{Hyperparameters} For BestDICE and 
AbstractBestDICE, we fixed the following hyperparameters:

\begin{itemize}
    \item $\gamma = 0.995$ in all experiments.
    \item Neural net architecture: All neural networks
    are $2$ layers with $64$ hidden units using tanh
    activation.
    \item Unit mean constraint learning rate \cite{zhang2020gradientdice, nachum2020bestdice}: $\lambda = 1e^{-3}$.
    \item Optimizer: Adam optimizer with default
    parameters in Pytorch.
    \item Positivity constraint: squaring function on
    the last layer of the neural network.
\end{itemize}
We conducted a search for the learning rate of $\nu$ ($\alpha_\nu)$
and learning rate of $\zeta$ ($\alpha_\zeta$), . The learning rate search for ($\alpha_\nu,
\alpha_\zeta$) was over
$\{(5e^{-5}, 5e^{-5}), (1e^{-4}, 1e^{-4}), (3e^{-4}, 3e^{-4}), (7e^{-4}, 7e^{-4}), (1e^{-3}, 1e^{-3}) \}$. The optimal hyperparameters ($\alpha_\nu = \alpha_\zeta$) for each environment and batch size were:

\begin{table}[H]
\centering

\begin{tabular}{l|l|l|l|l|l|l|l|l}
 &
5 &
10  &
50 &
75  &
100 &
300  &
500 &
1000 \\
Reacher &
$5e^{-5}$ &
$1e^{-4}$  &
$1e^{-3}$ &
$1e^{-4}$   &
$1e^{-4}$  &
$7e^{-4}$   &
$1e^{-3}$  &
$7e^{-4}$  \\
Walker2D &
$1e^{-3}$ &
$5e^{-5}$  &
$5e^{-5}$ &
$5e^{-5}$   &
$5e^{-5}$  &
$5e^{-5}$   &
$5e^{-5}$  &
$5e^{-5}$  \\
Pusher &
$5e^{-5}$ & 
$5e^{-5}$  &
$5e^{-5}$ &
$5e^{-5}$  &
$1e^{-4}$ &
$5e^{-5}$ &
$5e^{-5}$ &
$5e^{-5}$  \\
AntUMaze &
$3e^{-4}$ &
$5e^{-5}$  &
$5e^{-5}$ &
$5e^{-5}$  &
$5e^{-5}$ &
$5e^{-5}$  &
$5e^{-5}$ &
$5e^{-5}$ \\
\end{tabular}
\caption{Optimal hyparameters for AbstractBestDICE on each batch size and environment.}
\label{table1}
\end{table}

\begin{table}[H]
\centering

\begin{tabular}{l|l|l|l|l|l|l|l|l}
 &
5 &
10  &
50 &
75  &
100 &
300  &
500 &
1000 \\
Reacher &
$5e^{-5}$  &
$1e^{-4}$  &
$5e^{-5}$ &
$1e^{-4}$  &
$5e^{-5}$ &
$1e^{-4}$  &
$3e^{-4}$ &
$1e^{-3}$ \\
Walker2D &
$5e^{-5}$ &
$3e^{-4}$  &
$7e^{-4}$ &
$7e^{-4}$   &
$7e^{-4}$  &
$1e^{-4}$   &
$3e^{-4}$  &
$1e^{-4}$  \\
Pusher &
$5e^{-5}$  &
$1e^{-4}$  &
$1e^{-4}$ &
$5e^{-5}$  &
$1e^{-4}$ &
$5e^{-5}$  &
$1e^{-4}$ &
$5e^{-5}$ \\
AntUMaze &
$1e^{-4}$ &
$3e^{-4}$  &
$5e^{-5}$ &
$5e^{-5}$  &
$5e^{-5}$ &
$5e^{-5}$  &
$5e^{-5}$ &
$5e^{-5}$ \\
\end{tabular}
\caption{Optimal hyperparameters for BestDICE on each batch size and environment.}
\label{table2}
\end{table}

\subsubsection{Empirical Estimator} In practice we use a weighted importance sampling \cite{nachum2020bestdice} approach for the function approximation cases to estimate
$\rho(\pi_e)$ (same for BestDICE):
\begin{align*}
    \hat{\rho}(\pi^\phi_e) &= \frac{\sum_{i=1}^N \frac{d_{\pi^\phi_e}(s_i^\phi, a_i)}{d_{\pi^\phi_\mathcal{D}}(s_i^\phi, a_i)}r^\phi(s_i^\phi,a_i)}{\sum_{i=1}^N \frac{d_{\pi^\phi_e}(s_i^\phi, a_i)}{d_{\pi^\phi_\mathcal{D}}(s_i^\phi, a_i)}}
\end{align*}

\subsubsection{Misc Abstraction Details} 
\begin{itemize}
    \item For Walker2D, we modified the default reward function from incremental distance covered at each time-step to distance from start location at each time-step to ensure Assumption \ref{assumption:reward_equality} is satisfied.
    \item For AntUMaze, the reward function is originally $r(s')$ i.e. it is based
on the \textit{next} state that the ant moves to. To ensure Assumption \ref{assumption:reward_equality} is satisfied, we changed this reward function
to be of the \textit{current} state, $r(s)$.
\end{itemize}

\subsubsection{Additional Results}
    \textbf{Baseline Performance Comparison} As also reported by \citet{nachum2020bestdice, fu2021opebench}, we found in preliminary experiments that BestDICE performed much better than other MIS methods such as DualDICE \cite{nachum2019dualdice}, Minimax-Weight Learning \cite{masatoshi2019minimax}, etc.
        \begin{figure}[H]
        \centering
            \subfigure[Reacher ]{\includegraphics[scale=0.15]{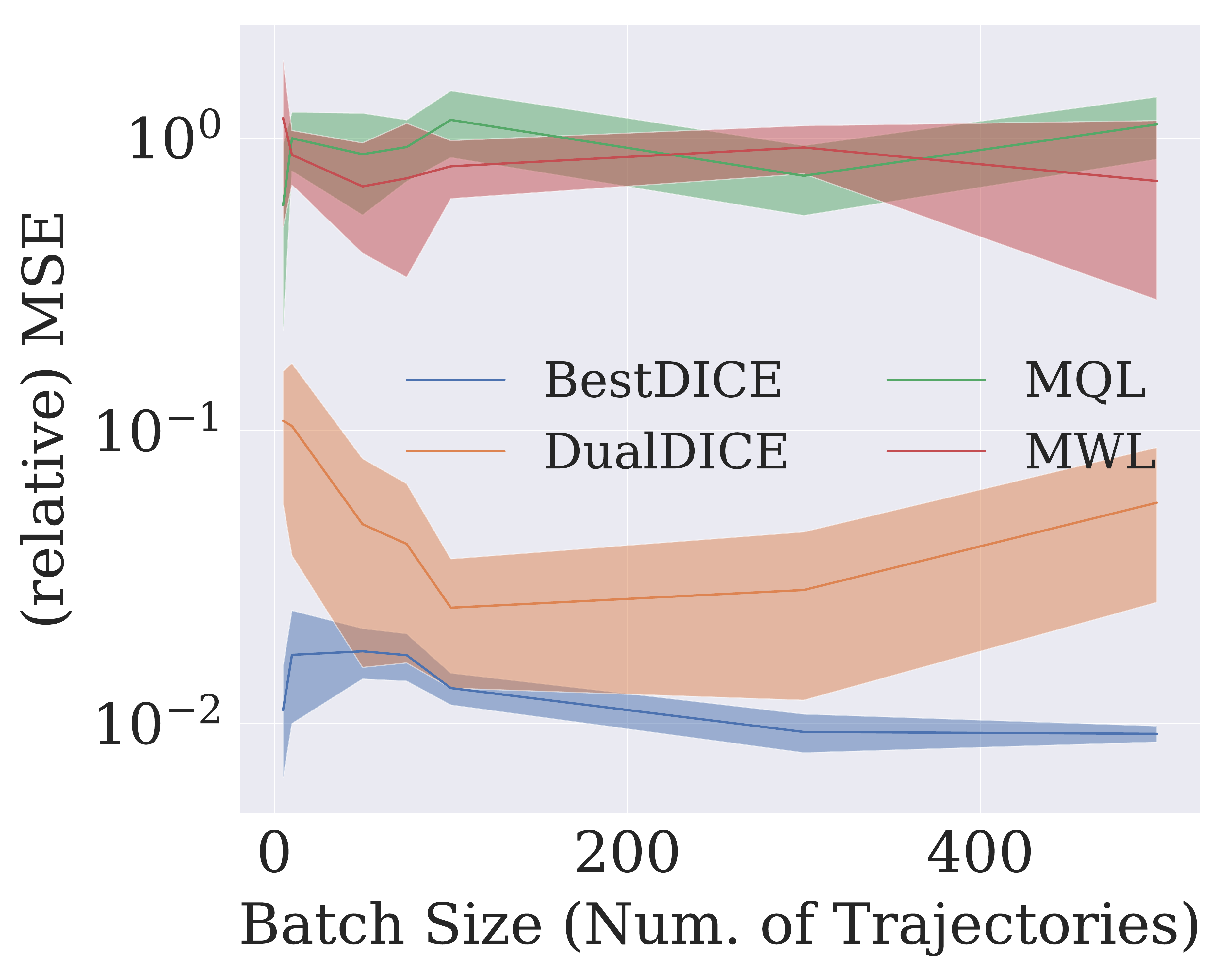}}
            \subfigure[Walker ]{\includegraphics[scale=0.15]{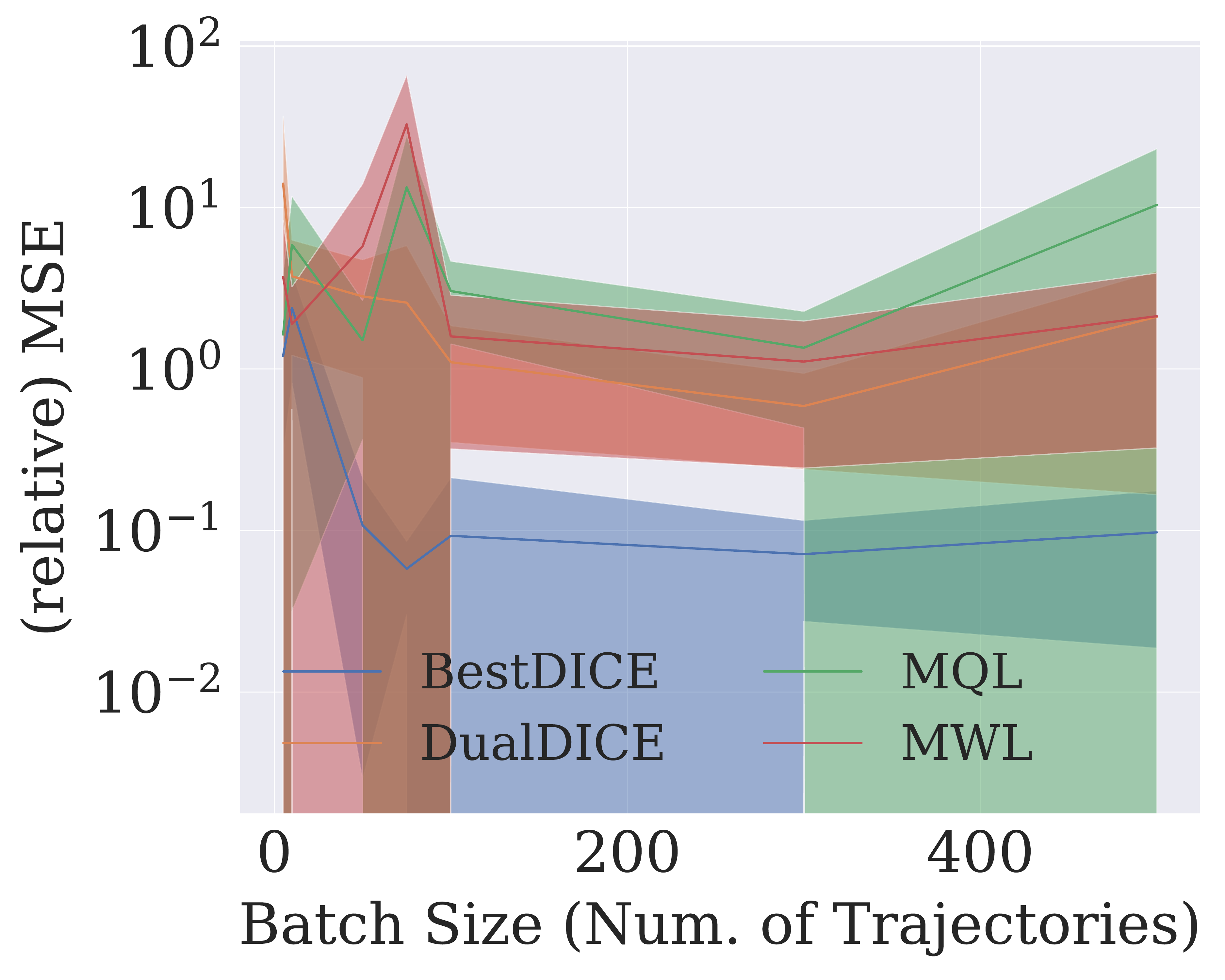}}
        \caption{\footnotesize Baseline Performance comparison among common MIS methods (after hyperparameter search). Lower is better.}
    \end{figure}
    
    \textbf{Additional Hyperparameter Robustness Results} In general, we can see
    AbstractBestDICE can be much more robust than BestDICE to hyperparameter
    tuning.
    \begin{figure}[H]
        \centering
            \subfigure[Reacher (Batch size: 5)]{\includegraphics[scale=0.15]{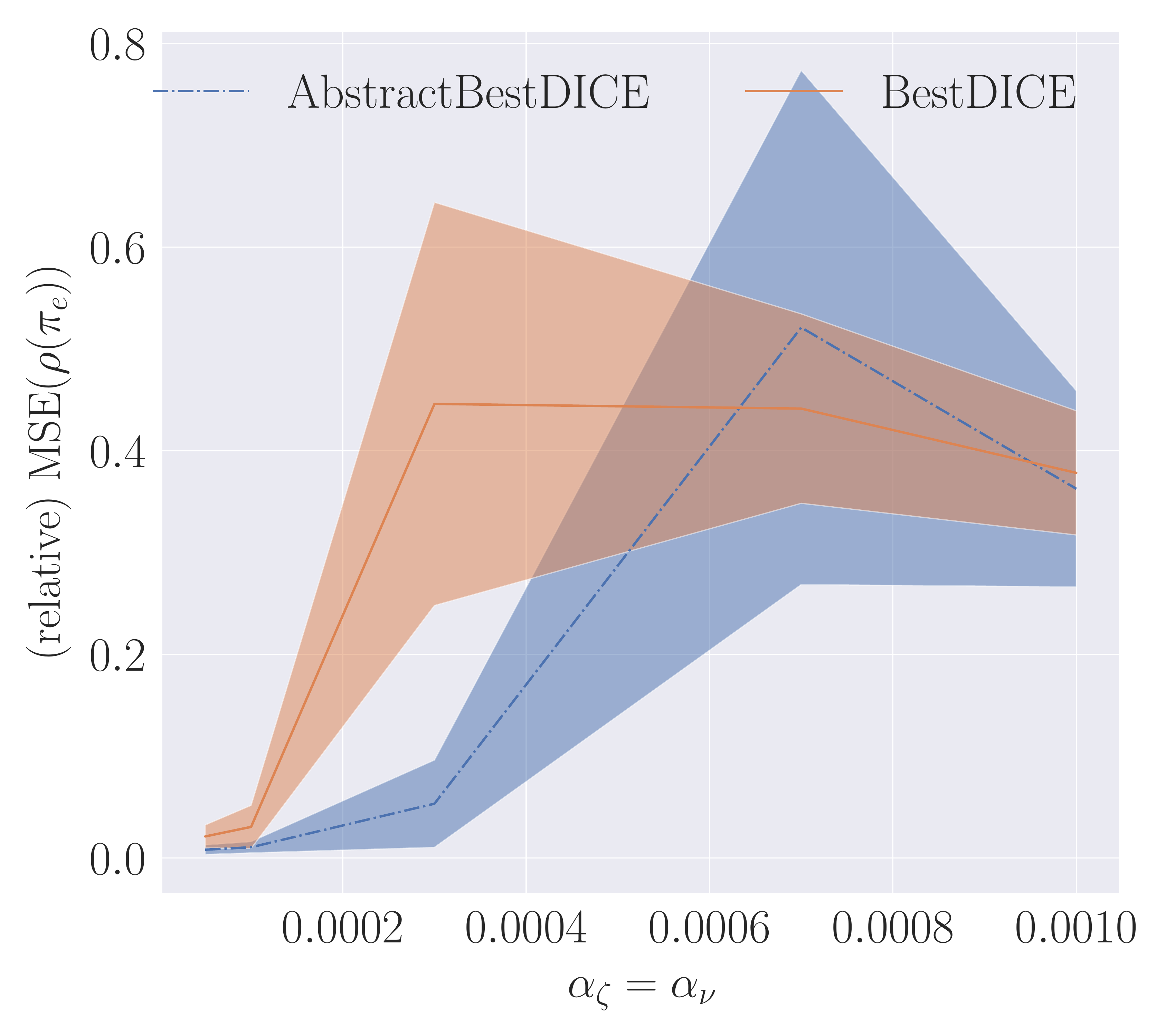}}
            \subfigure[Reacher (Batch size: 10)]{\includegraphics[scale=0.15]{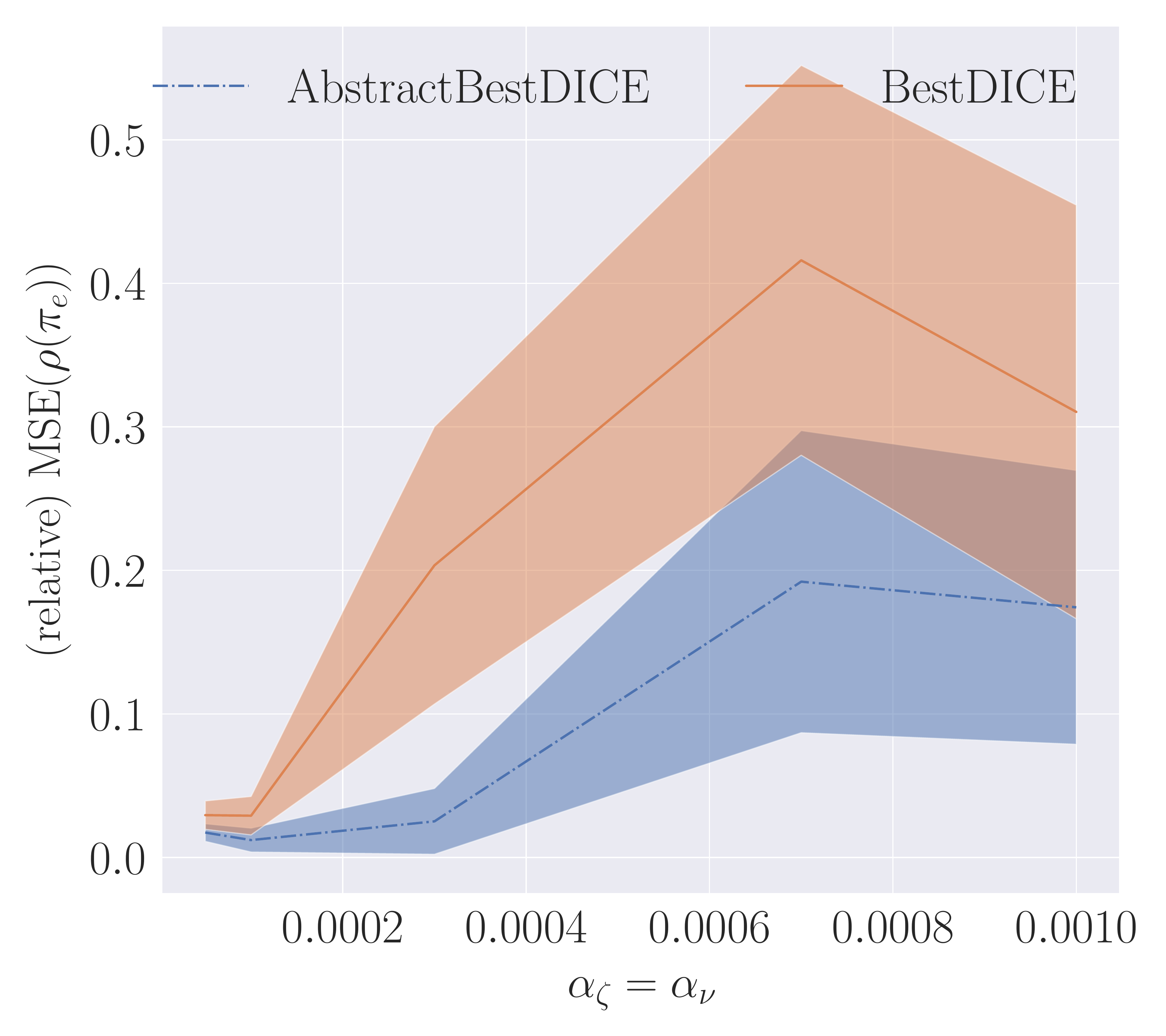}}
            \subfigure[Reacher (Batch size: 50)]{\includegraphics[scale=0.15]{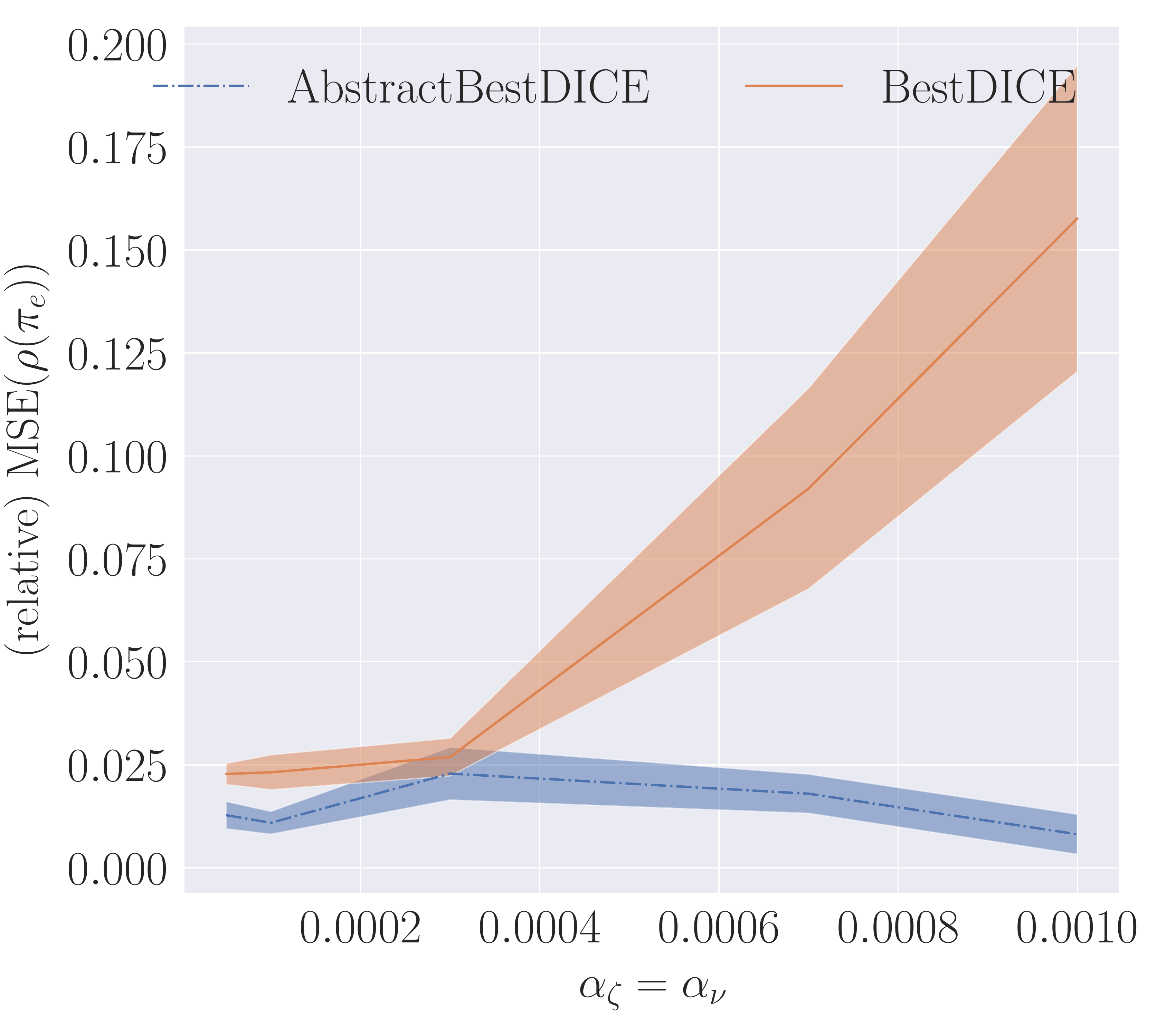}}
            \subfigure[Walker2D (Batch size: 5)]{\includegraphics[scale=0.15]{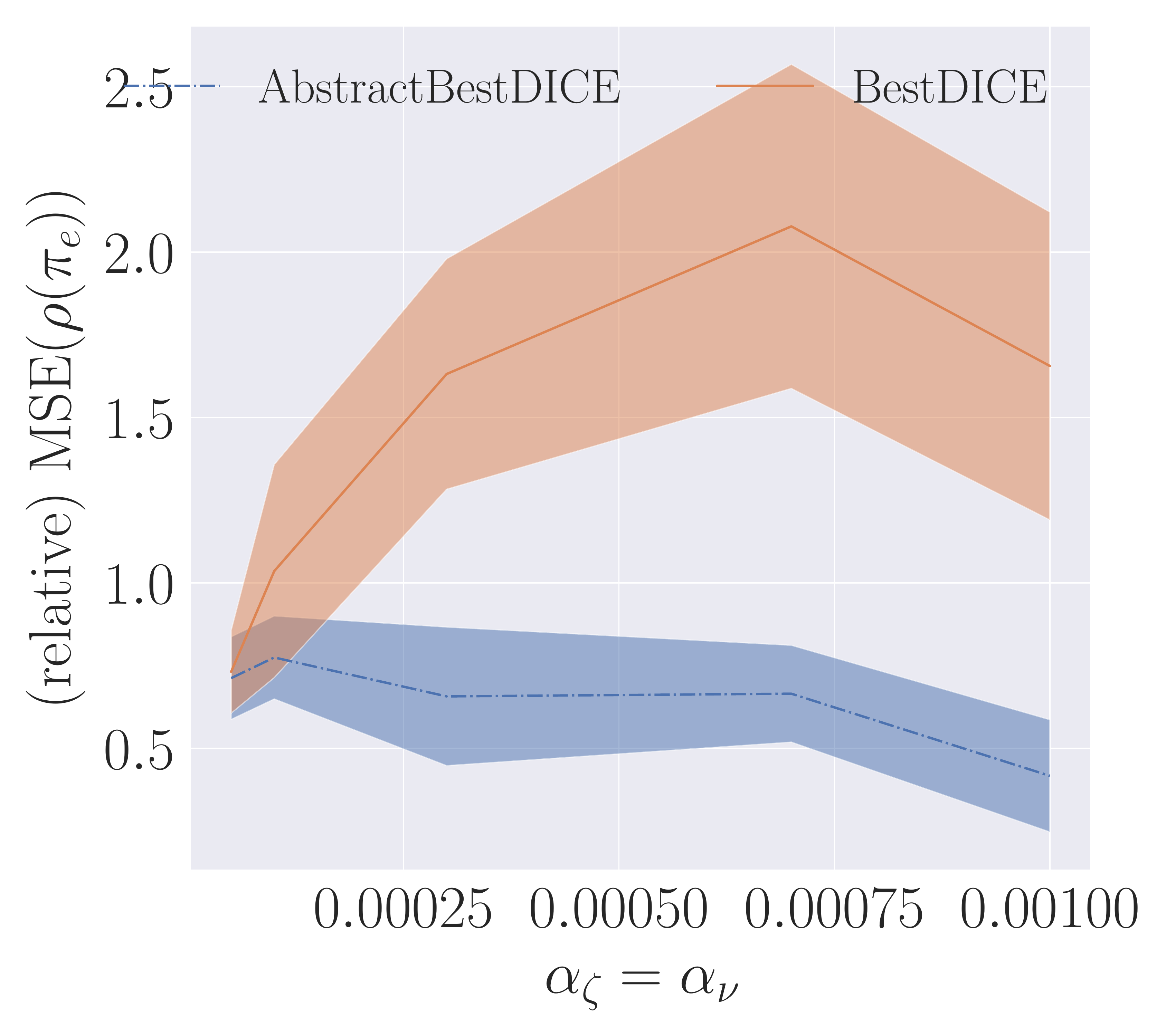}}
            \subfigure[Walker2D (Batch size: 10)]{\includegraphics[scale=0.15]{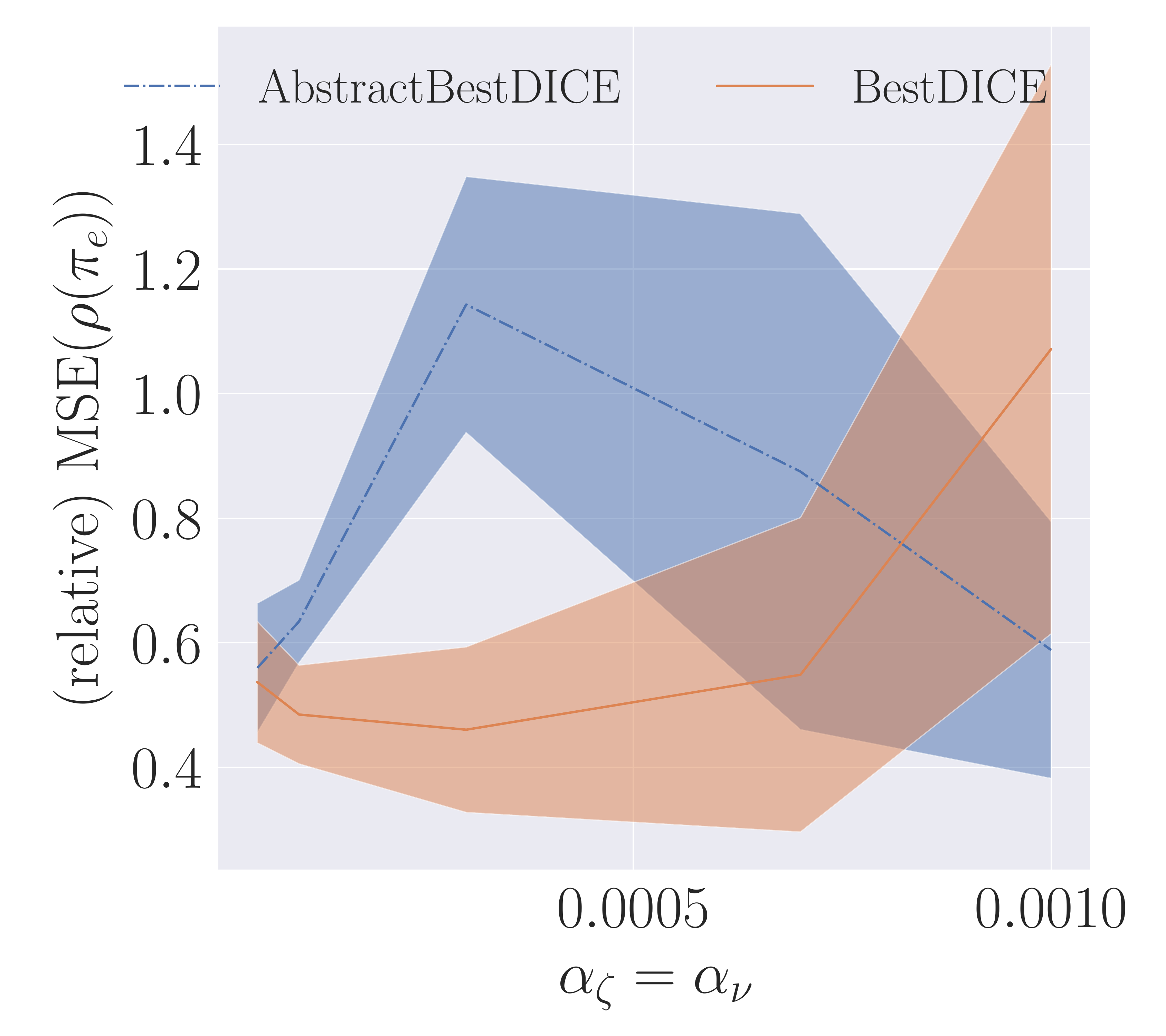}}
            \subfigure[Walker2D (Batch size: 50)]{\includegraphics[scale=0.15]{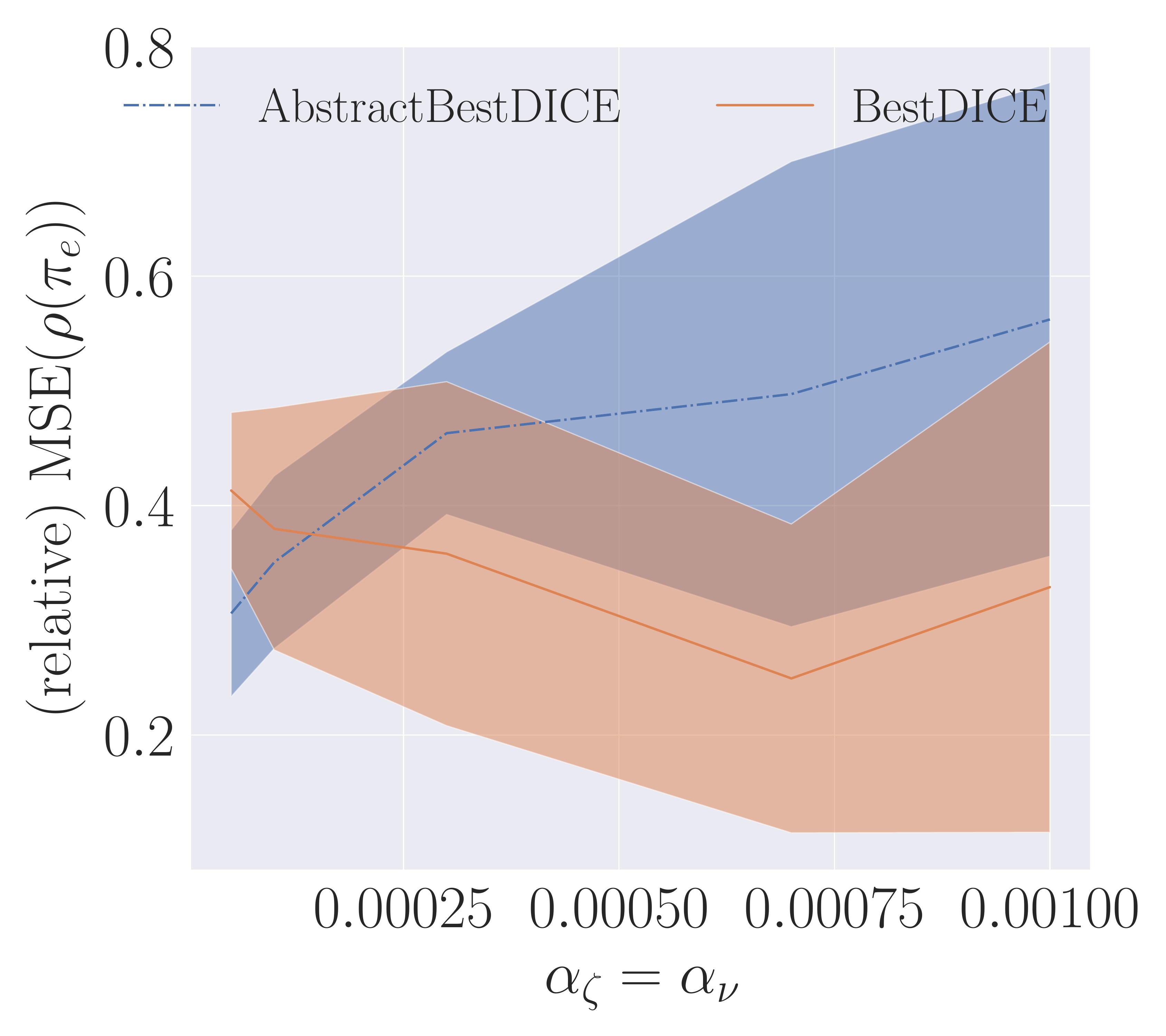}}
            \subfigure[Pusher (Batch size: 5)]{\includegraphics[scale=0.15]{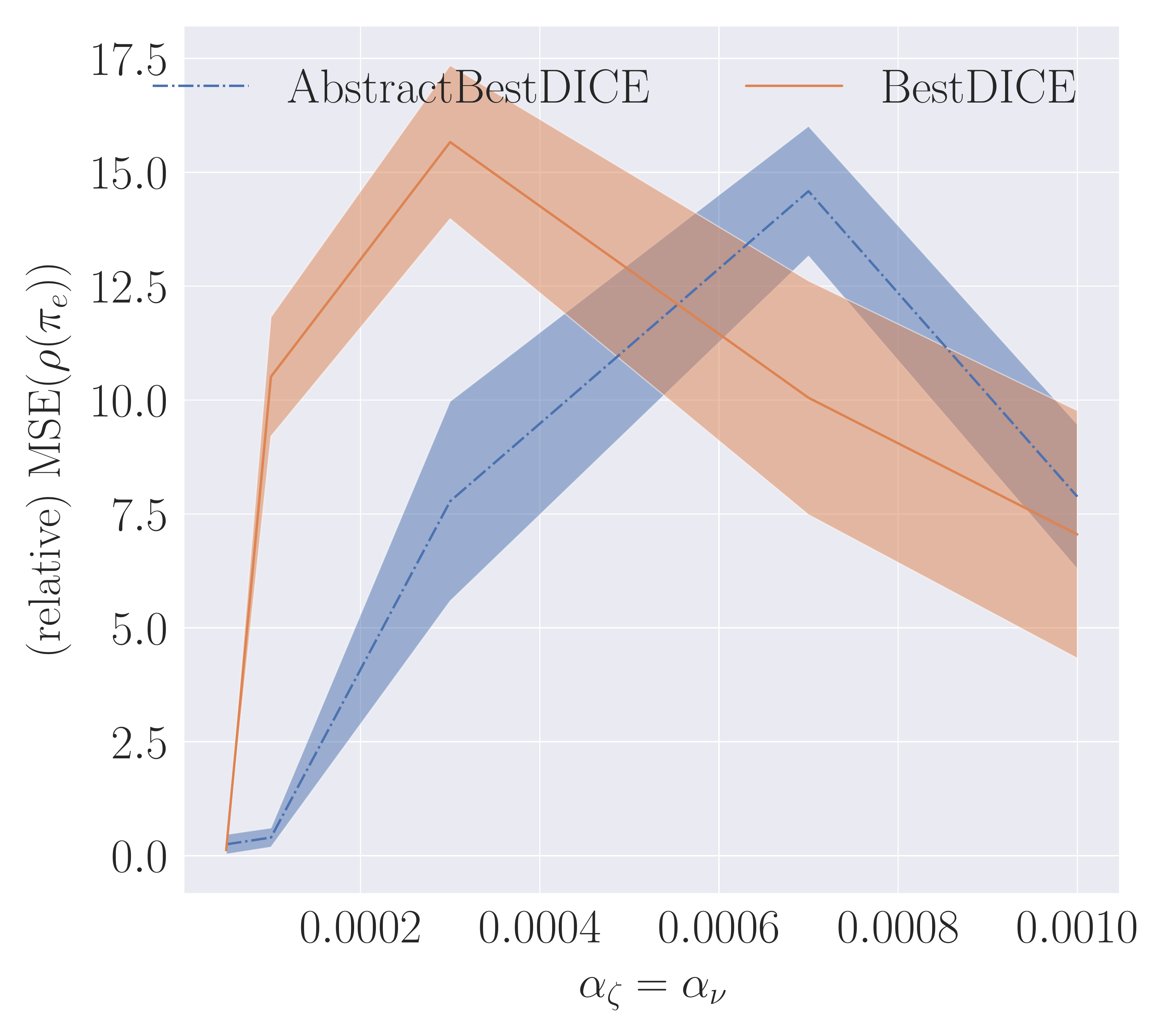}}
            \subfigure[Pusher (Batch size: 10)]{\includegraphics[scale=0.15]{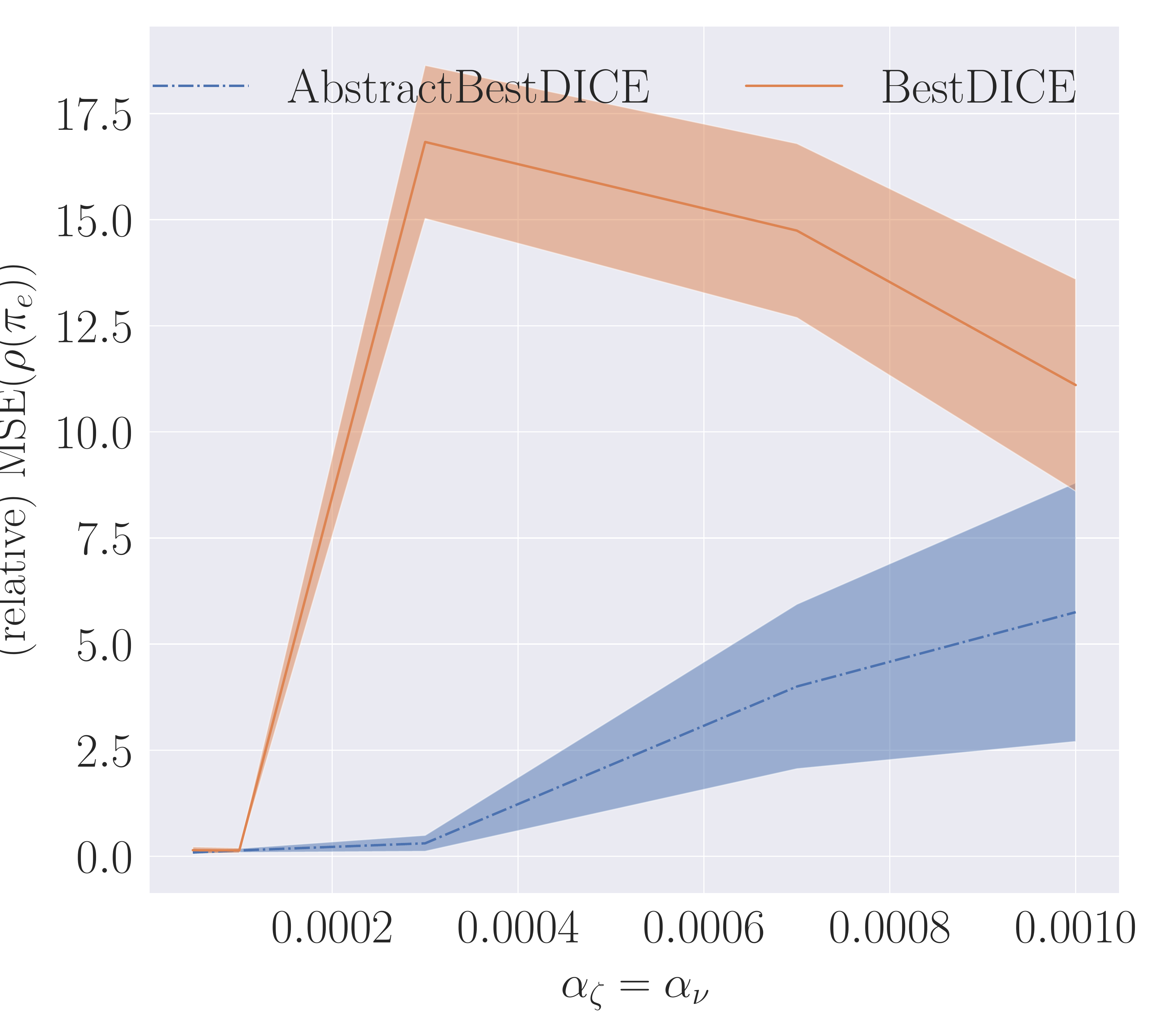}}
            \subfigure[Pusher (Batch size: 50)]{\includegraphics[scale=0.15]{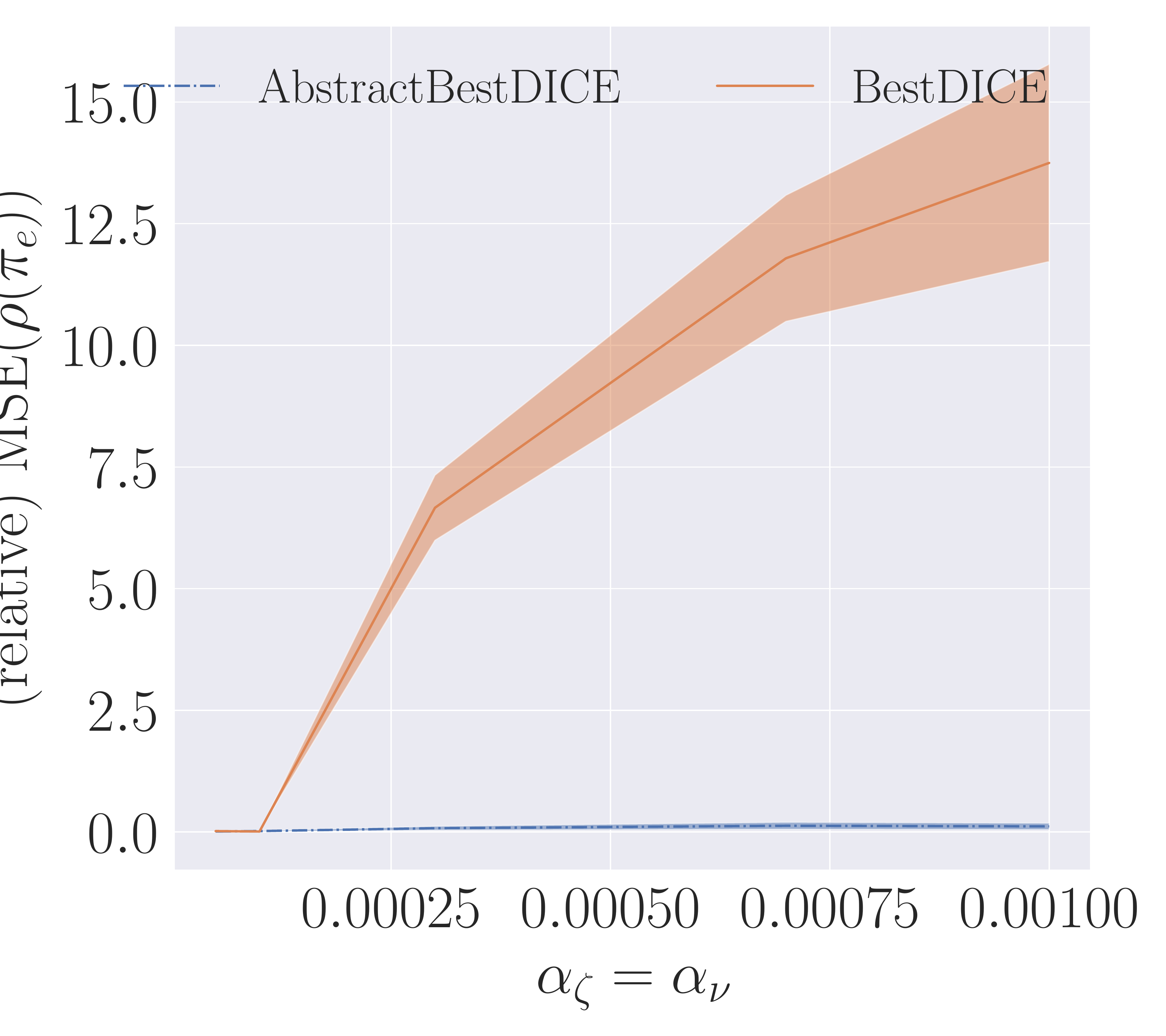}}\\
            \subfigure[AntUMaze (Batch size: 10)]{\includegraphics[scale=0.15]{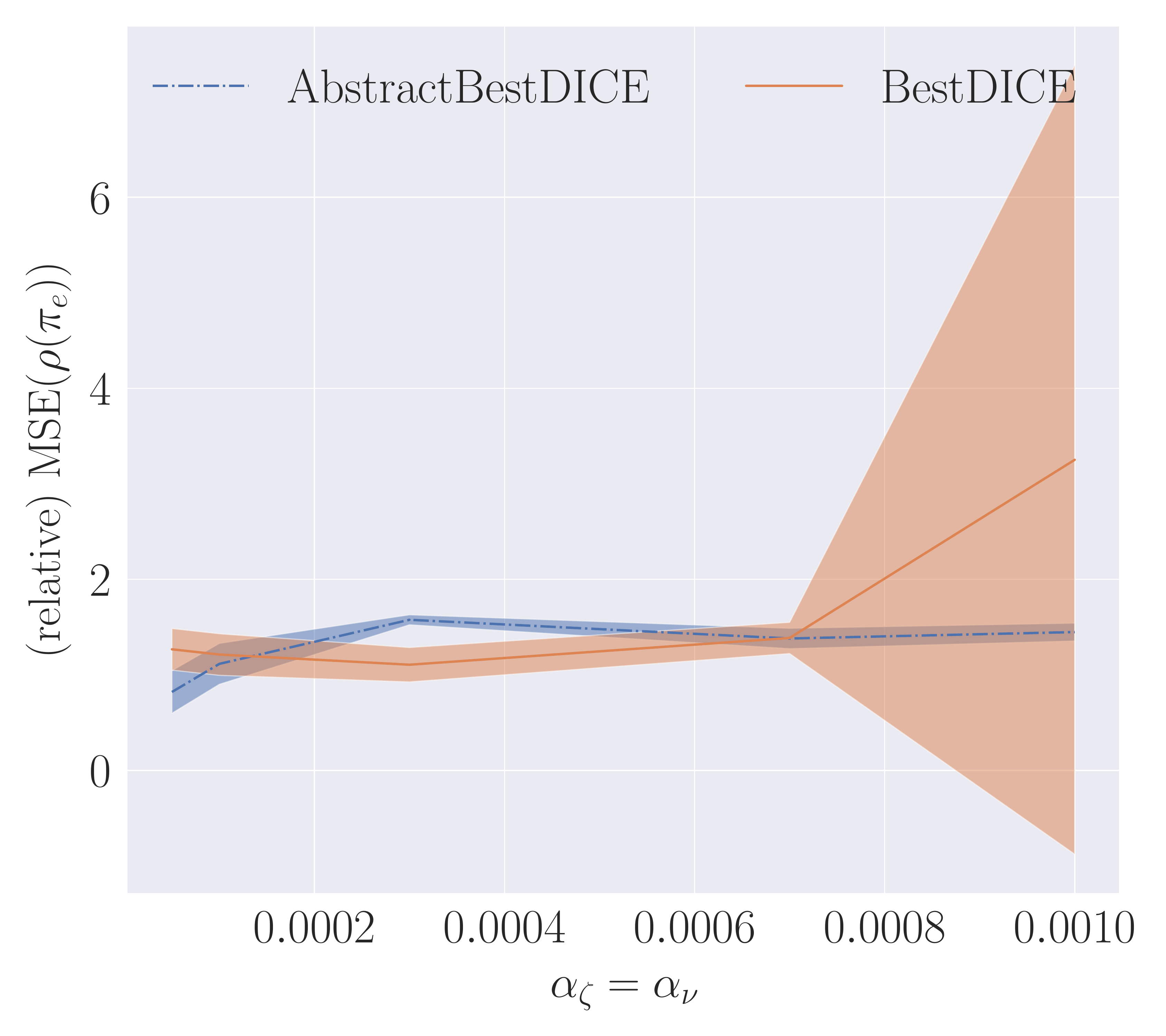}}
            \subfigure[AntUMaze (Batch size: 50)]{\includegraphics[scale=0.15]{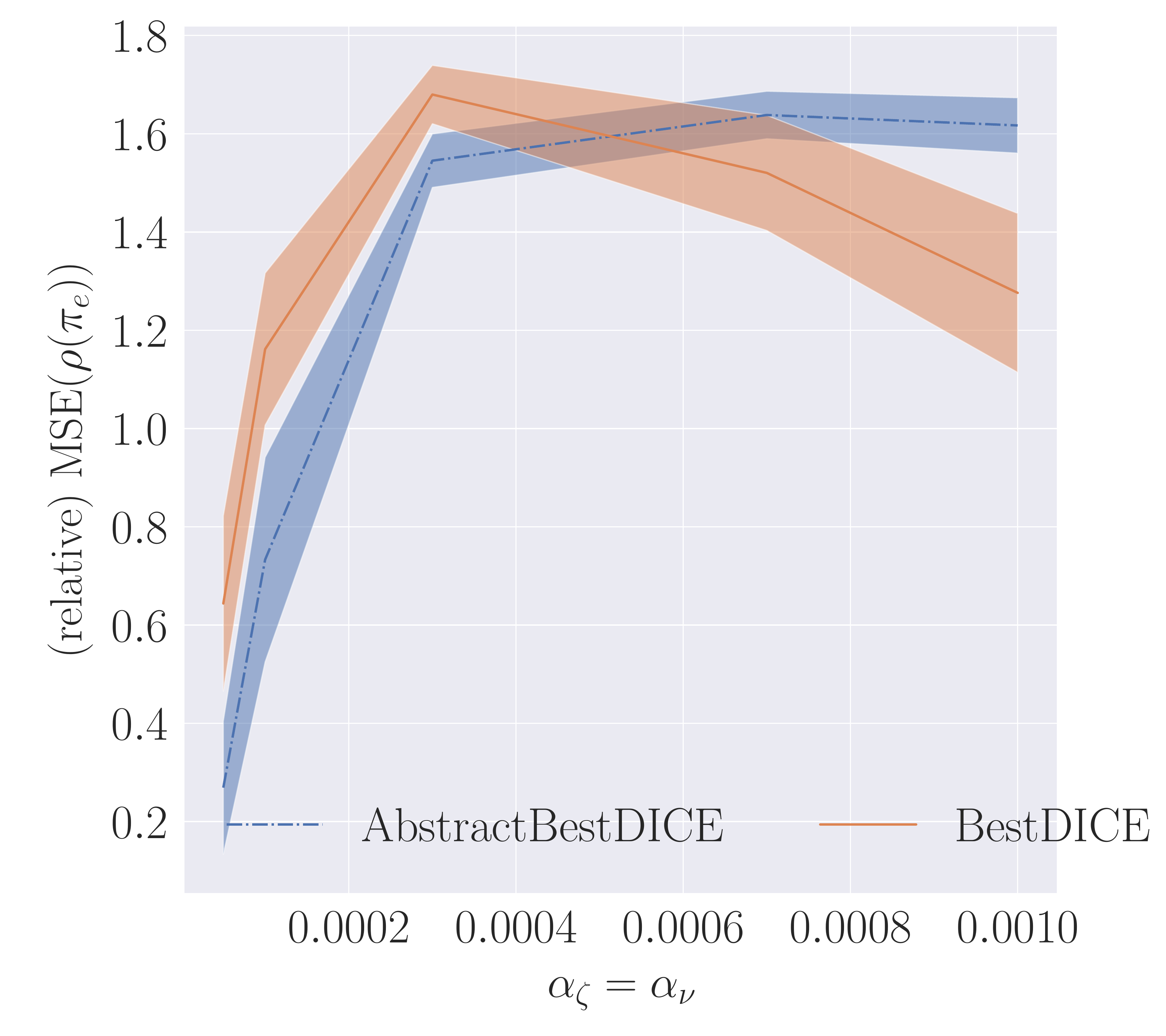}}
        \caption{\footnotesize Hyperparameter sensitivity graph for BestDICE and AbstractBestDICE. $\alpha_\zeta = \alpha_\nu$ Errors are computed over $15$ trials with $95\%$ confidence intervals. Lower is better. Pusher for
        batch size of $50$ is shown in the main paper.}
        \label{fig:hp_sens_app}
    \end{figure}

    \textbf{Training Stability} In Figure \ref{fig:rew_tr} we show that
    $\phi$ can improve training stability.
    \begin{figure}[H]
        \centering
            \subfigure[Reacher (Batch size: 10)]{\includegraphics[scale=0.15]{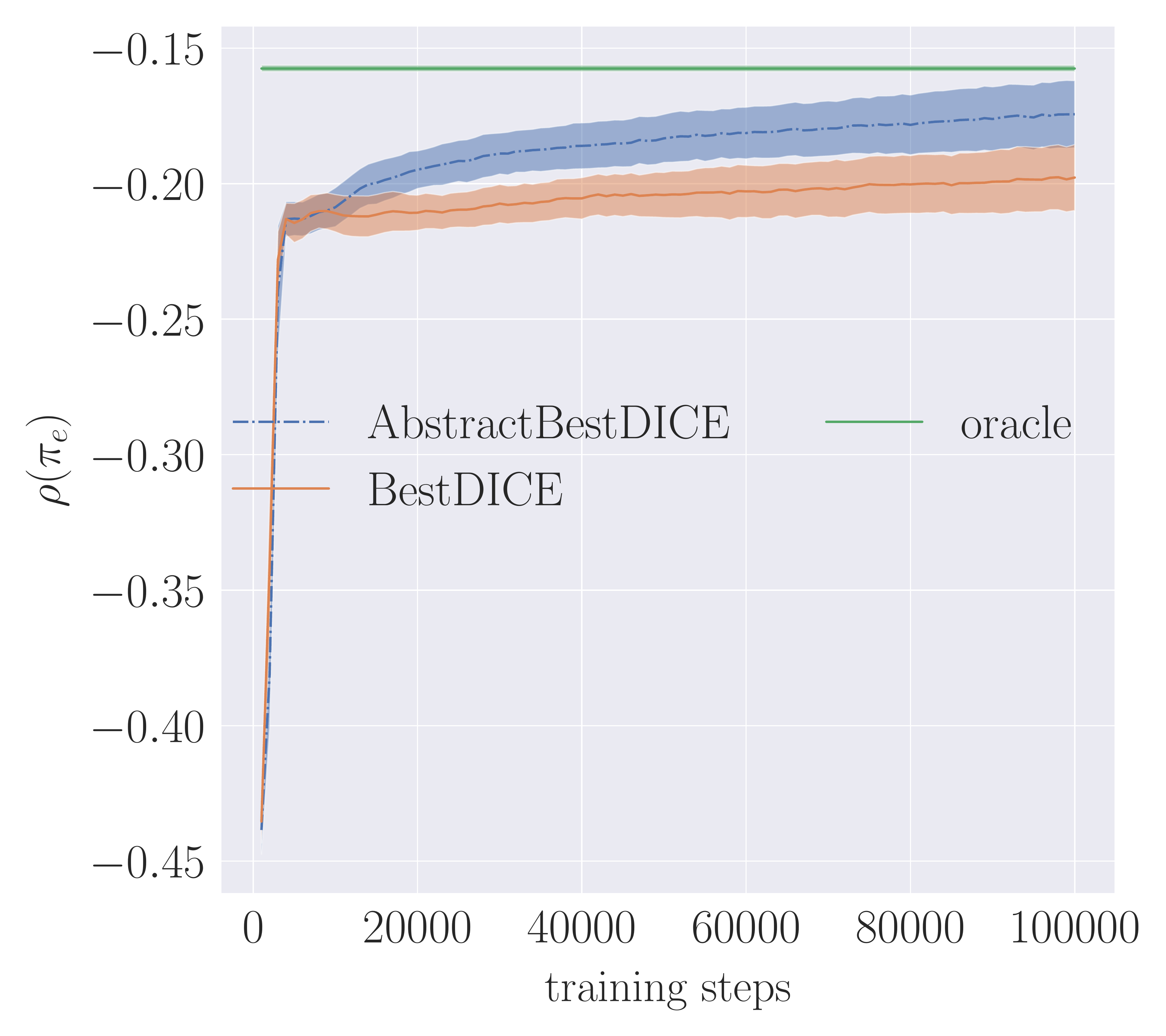}}
            \subfigure[Reacher (Batch size: 500)]{\includegraphics[scale=0.15]{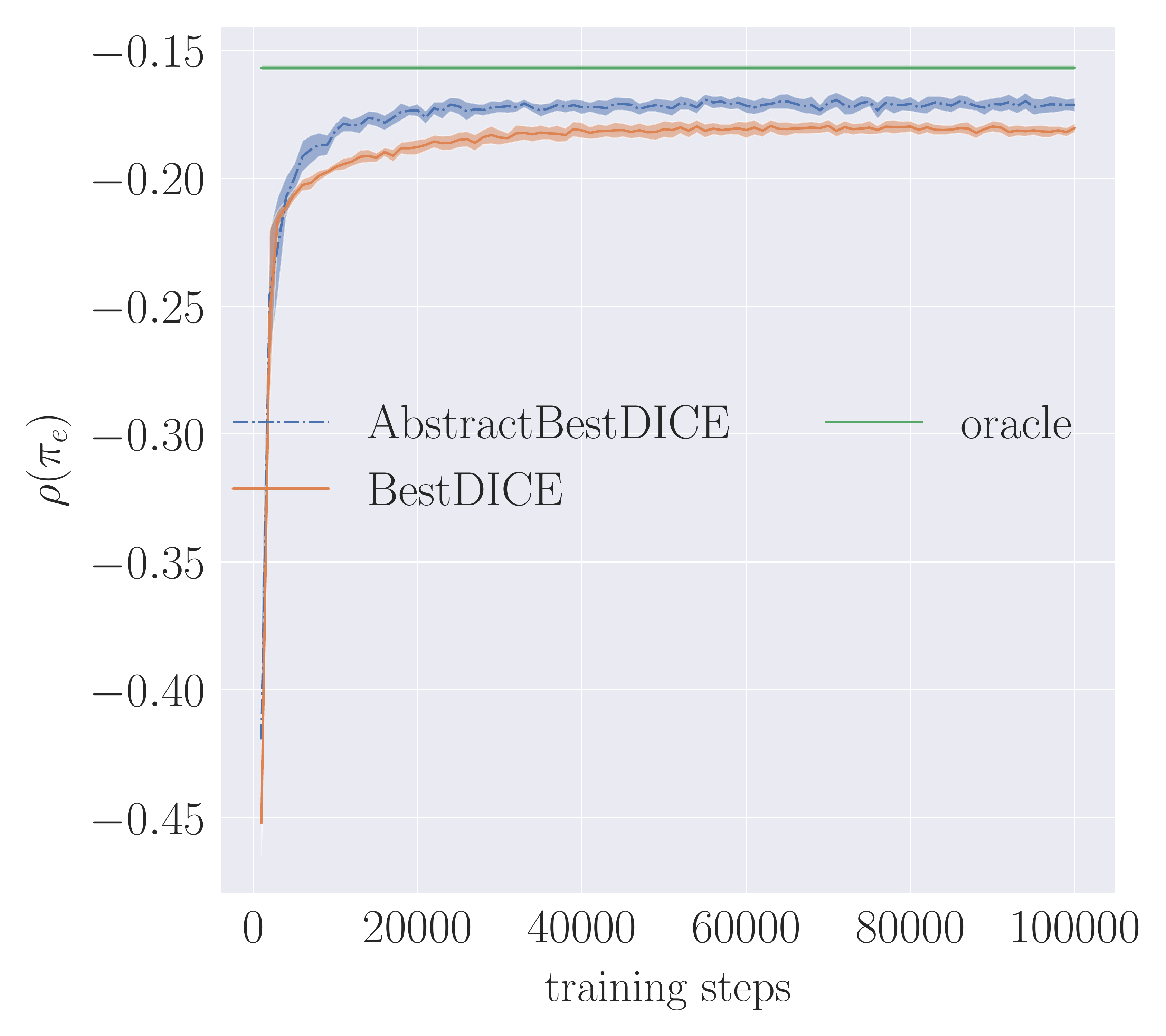}}
            \subfigure[Walker2D (Batch size: 10)]{\includegraphics[scale=0.15]{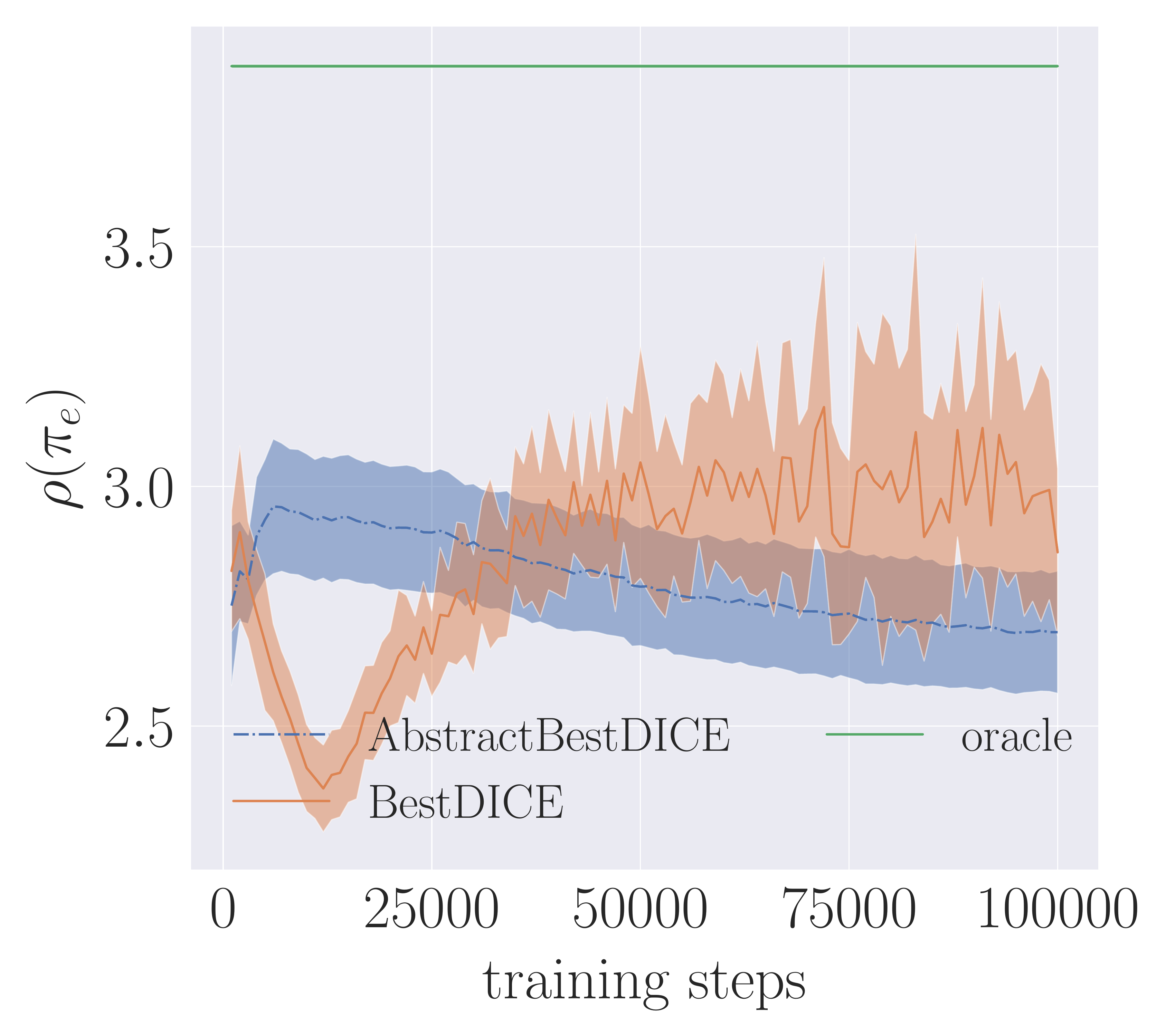}}
            \subfigure[Walker2D (Batch size: 500)]{\includegraphics[scale=0.15]{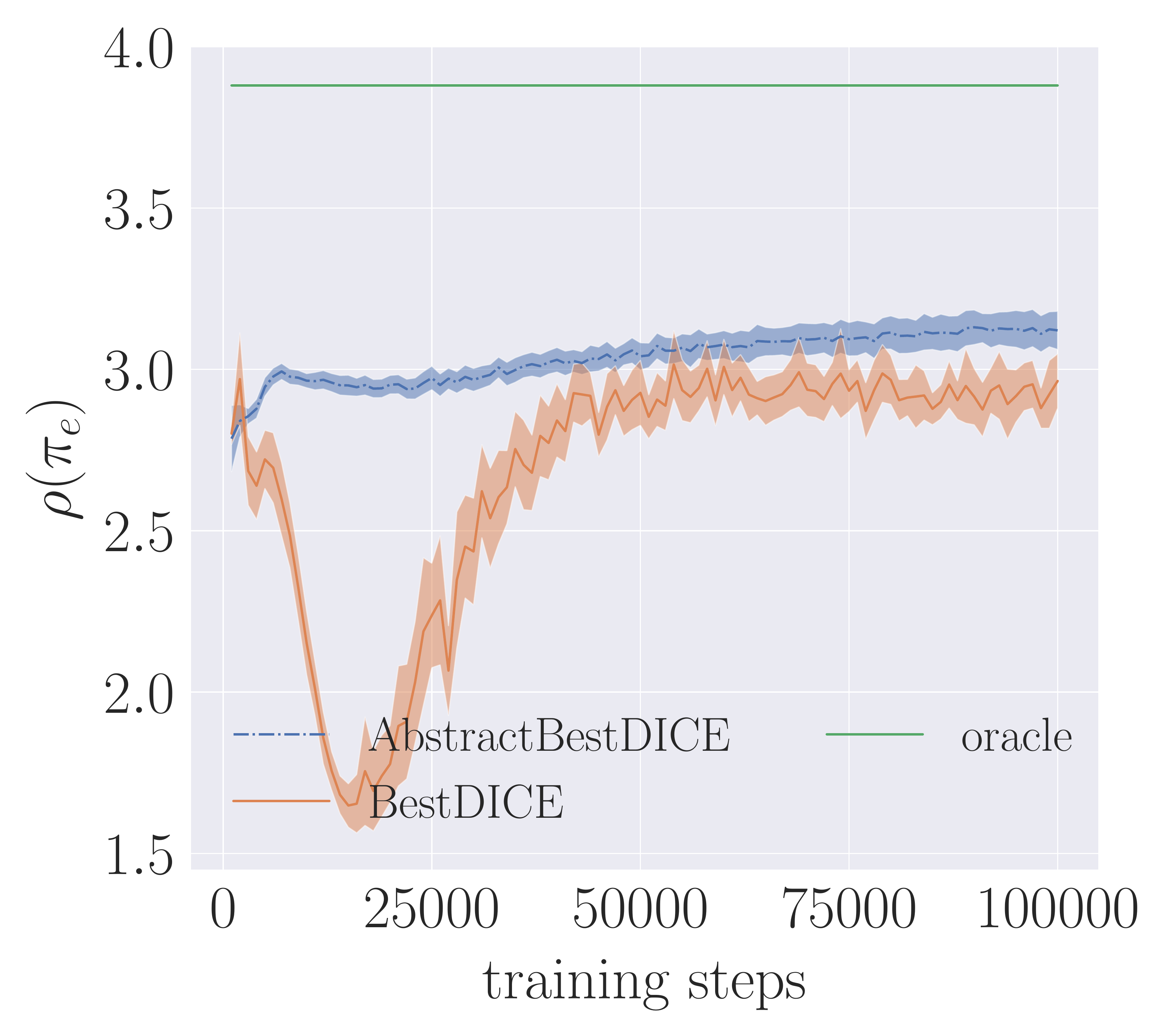}}
            \subfigure[Pusher (Batch size: 10)]{\includegraphics[scale=0.15]{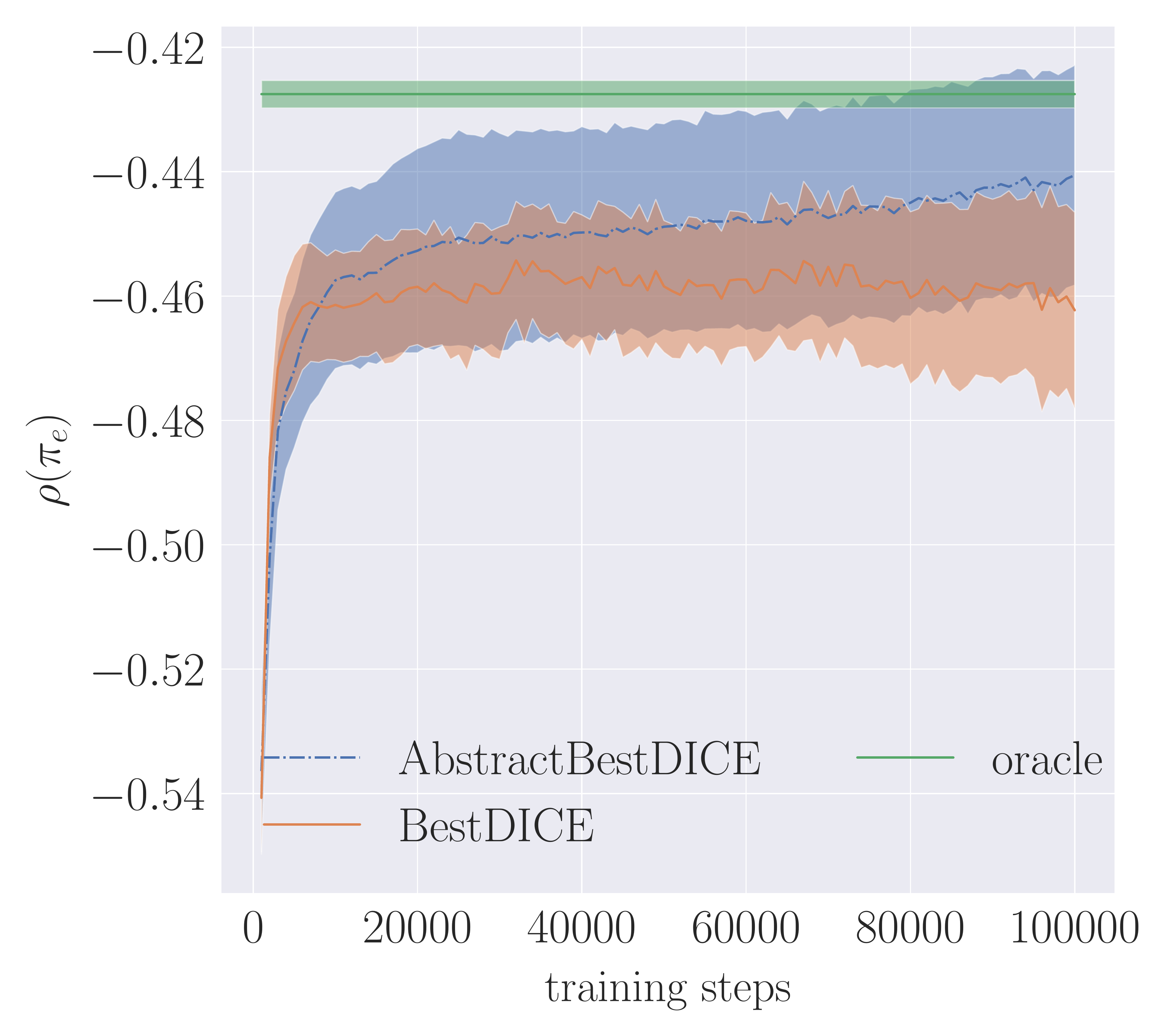}}
            \subfigure[Pusher (Batch size: 500)]{\includegraphics[scale=0.15]{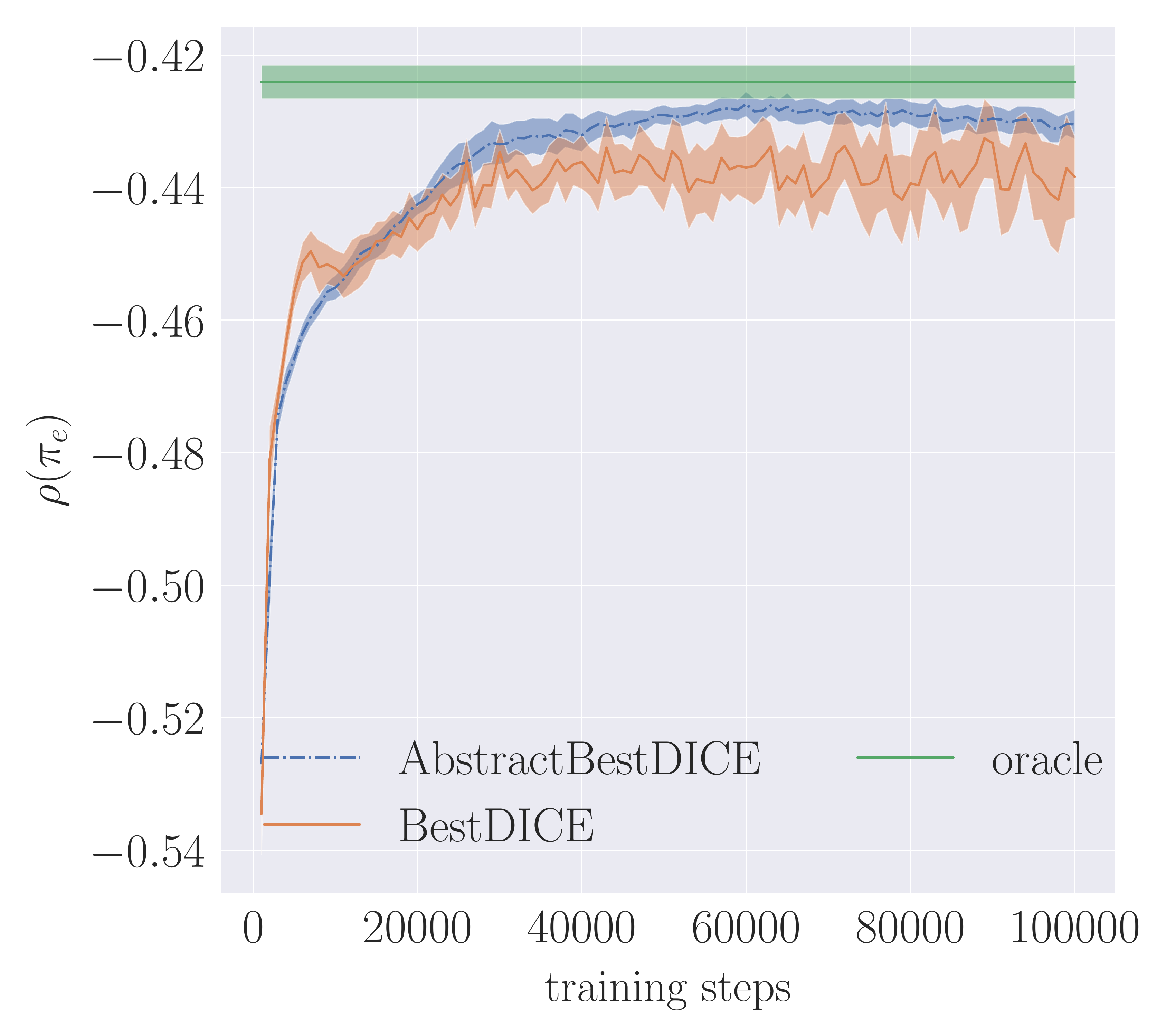}}
            \subfigure[AntUMaze (Batch size: 10)]{\includegraphics[scale=0.15]{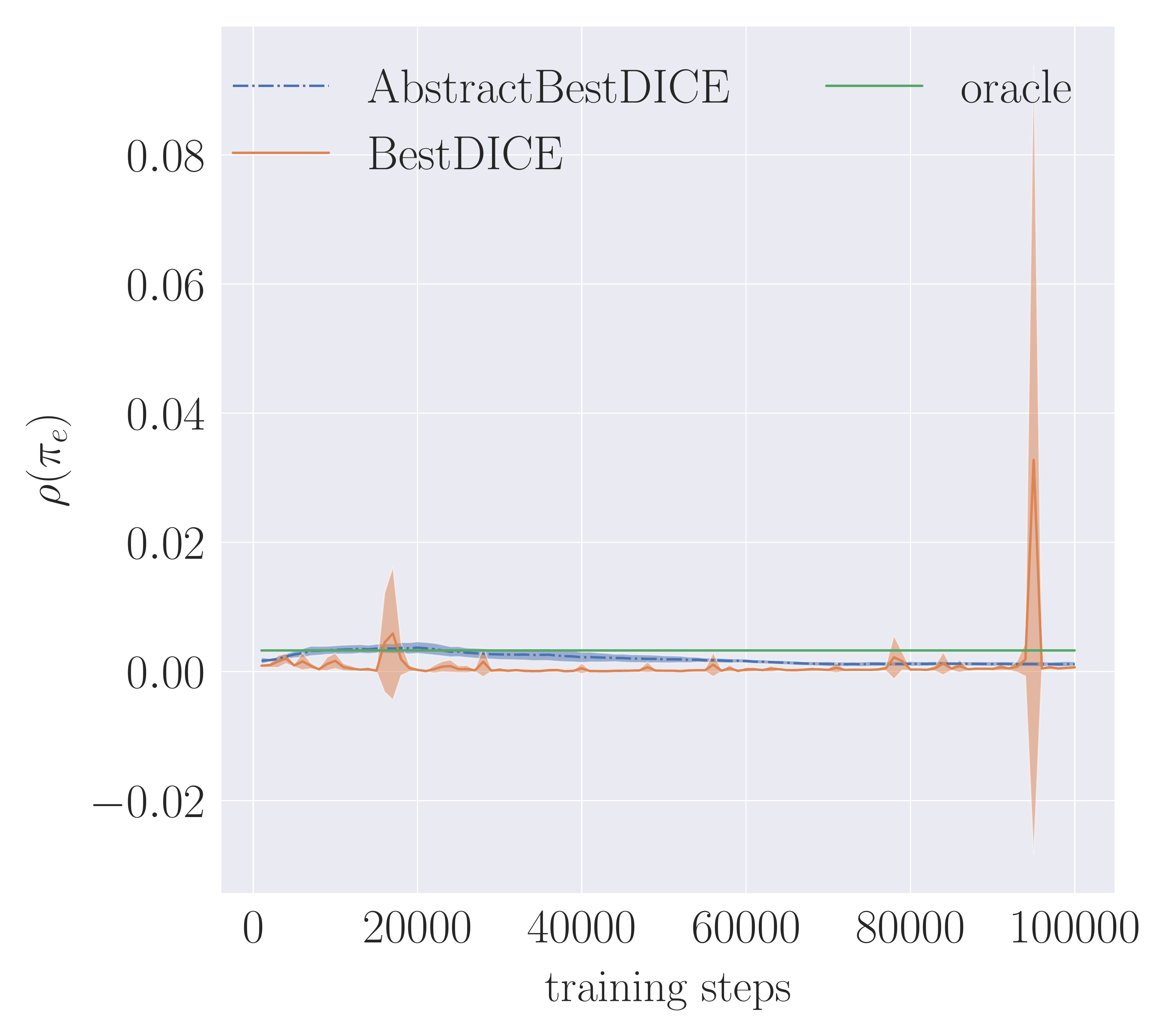}}
            \subfigure[AntUMaze (Batch size: 500)]{\includegraphics[scale=0.15]{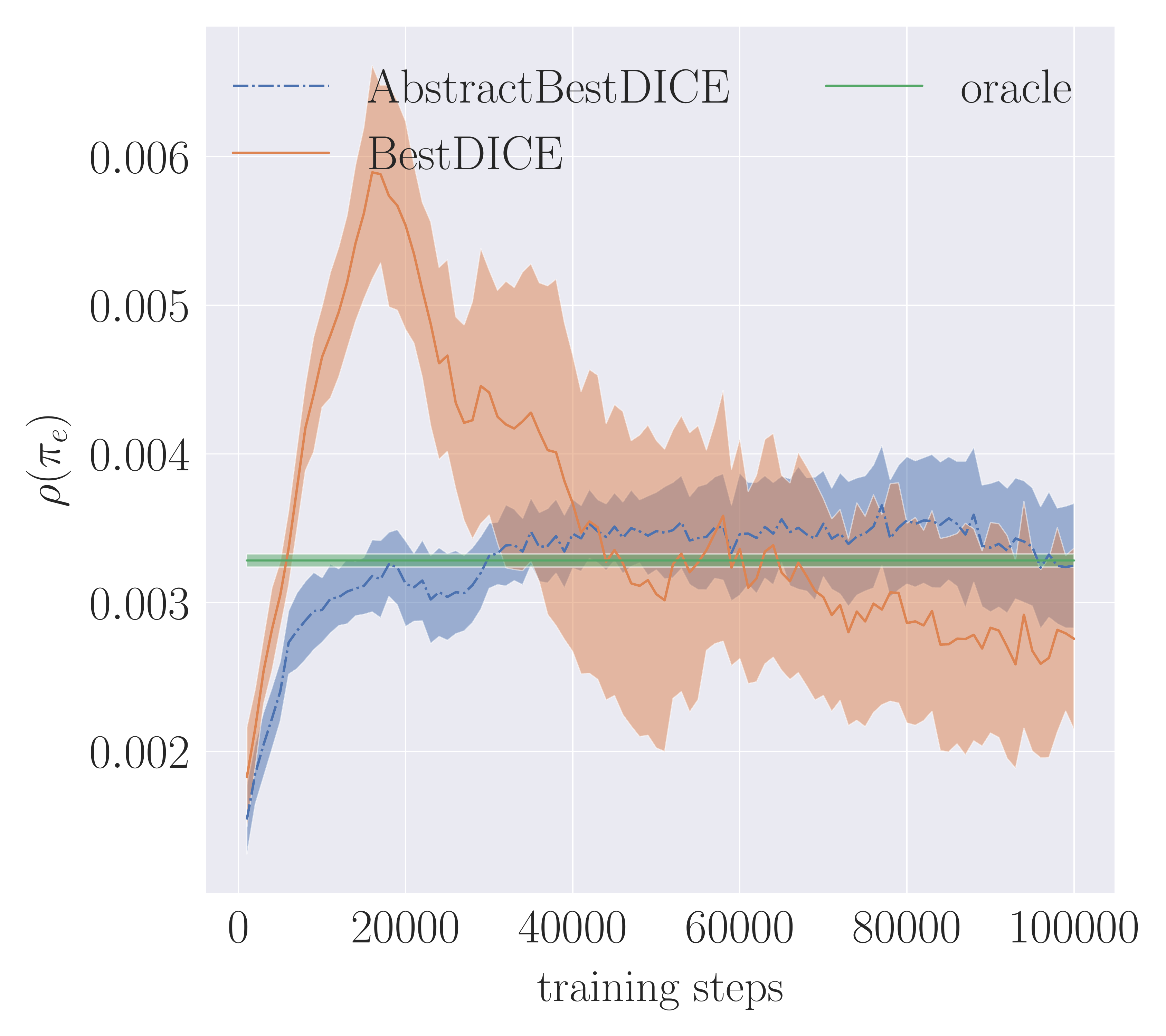}}
        \caption{\footnotesize Reward vs. Training
        Steps.  Errors are computed over $15$ trials with $95\%$ confidence intervals. 
        These figures illustrate the training
        stability of AbstractBestDICE over BestDICE. Lower is better.}
        \label{fig:rew_tr}
    \end{figure}
    
    \textbf{Abstract Quality and Data-Efficiency}. 
    We find that not all abstractions that satisfy
    Assumption \ref{assumption:reward_equality} lead to better performance. For 
    example, the following are valid abstractions on the Reacher task: 1) the
    Euclidean distance between the arm and goal, $s^\phi\in\mathbb{R}$ and
    the 3D vector between the arm and goal, $s^\phi\in\mathbb{R}^3$ (Figure \ref{fig:vs_batch_fail_abs}). However,
    in practice we found that these were unreliable. One possible reason for this
    unreliability is that these abstractions are incredibly extreme and the algorithm
    may be unable to 
    differentiate between abstract state, resulting
    in outputting similar $\zeta^\phi(s^\phi, a)\forall s^\phi$.
    \begin{figure}[H]
        \centering
            \subfigure[Reacher]{\label{fig:reacher_vs_batch_fail_abs}\includegraphics[scale=0.2]{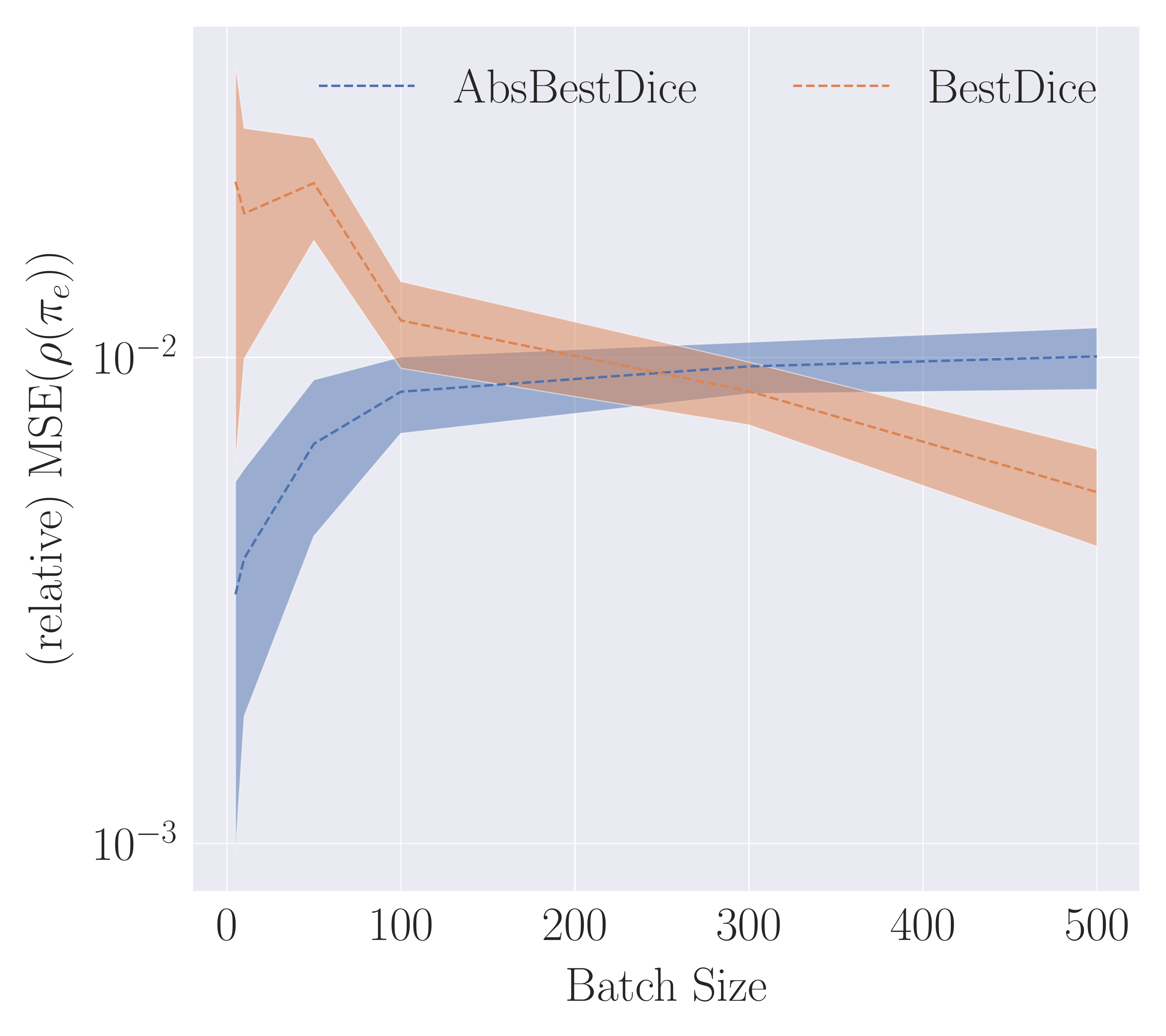}}
        \caption{\footnotesize Relative MSE vs. Batch Size. $y$ axis is log-scaled. Errors are computed over $15$ trials with $95\%$ confidence intervals. 
        This figures illustrate valid abstractions can be more
        data-inefficient than the ground equivalents. Lower is better.}
        \label{fig:vs_batch_fail_abs}
    \end{figure}
    
\subsection{Hardware For Experiments}
\begin{itemize}
    \item Distributed cluster on HTCondor framework
    \item Intel(R) Xeon(R) CPU E5-2470 0 @ 2.30GHz
    \item RAM: 5GB
    \item Disk space: 4GB 
\end{itemize}


\end{document}